\documentclass[leqno,onefignum,onetabnum]{siamltex1213}

\usepackage{amssymb,amsmath}

\usepackage{subfig}
\usepackage{tikz}

\usepackage{algorithm}
\usepackage{algorithmic}

\newcommand{\yingyu}[1]{{#1}}
\newcommand{\mycomment}[1]{}

\newcommand{\defeq}{\ensuremath{:=}}
\newcommand{\mycup}{\cup}
\newcommand{\mycap}{\cap}
\newcommand{\dcup}{\cup} 

\newcommand{\data}{S}
\newcommand{\OPT}{\mathcal{OPT}}
\newcommand{\cost}{\mathrm{cost}}

\newcommand{\ds}{\yingyu{d_\mathrm{s}}}
\newcommand{\dps}{d'_\mathrm{s}}
\newcommand{\da}{d_\mathrm{a}}
\newcommand{\drs}{d_\mathrm{rs}}
\newcommand{\dra}{d_\mathrm{ra}}

\newcommand{\NP}{NP}
\newcommand{\RP}{RP}

\newtheorem{claim}{Claim}[section]
\newtheorem{fact}{Fact}[section]

\newenvironment{proofof}[1]{\par {\it Proof of {#1}.}}{\endproof}

\newtheorem{note}{Note}[section]

\title{Clustering under Perturbation Resilience\thanks{Part of the results in this article appeared under the title \emph{Clustering under Perturbation Resilience} in the
Proceedings of the Thirty-Ninth International Colloquium on Automata, Languages and Programming, 2012.} 
}

\author{
Maria Florina Balcan\thanks{Carnegie Mellon University, Pittsburgh, PA  15213 (ninamf@cs.cmu.edu).}
\and
Yingyu Liang\thanks{Princeton University, Princeton, NJ 08540 (yingyul@cs.princeton.edu).}
}

\begin{document}
\maketitle
\slugger{sicomp}{xxxx}{xx}{x}{x--x}

\begin{abstract}
Motivated by the fact that distances between data points in many real-world
clustering instances are often based on  heuristic measures, Bilu and
Linial~\cite{BL} proposed analyzing objective based clustering problems
under the assumption  that the optimum clustering to the objective is
preserved under small multiplicative perturbations to distances between
points. The hope is that by exploiting the structure in such instances,
one can overcome worst case hardness results.

In this paper, we provide several results within this framework.
For center-based objectives, we present an algorithm that can optimally
cluster instances resilient to perturbations of factor $(1 + \sqrt{2})$,
solving an open problem of Awasthi et al.~\cite{ABS10}. For $k$-median, a
center-based objective of special interest, we additionally give algorithms
for a more relaxed assumption in which we allow the optimal solution to
change in a small $\epsilon$ fraction of the points after perturbation. We
give the first bounds known for $k$-median under this more realistic and
more general assumption. We also provide positive results for min-sum
clustering which is typically a harder objective than center-based objectives from approximability standpoint.  Our algorithms are based on new linkage criteria that may be
of independent interest.

Additionally, we  give sublinear-time algorithms, showing algorithms that
can return an implicit clustering from only access  to a small random
sample.
\end{abstract}



\pagestyle{myheadings}
\thispagestyle{plain}
\markboth{M. F. BALCAN, AND Y. LIANG}{CLUSTERING UNDER PERTURBATION RESILIENCE}


\section{Introduction}

Problems of clustering data from pairwise distance
information are ubiquitous in science. A common approach
for solving such problems is to view the data points
as nodes in a weighted graph (with the weights based on the
given pairwise information), and then to design algorithms
to optimize various objective functions such as $k$-median or
min-sum. For example, in the $k$-median clustering problem the goal is to partition the data into
$k$ clusters $C_i$, giving each a center $c_i$, in order to minimize the sum of the distances of all data points to the centers of their cluster.
In the min-sum clustering approach
the goal is to find $k$ clusters $C_i$ that minimize the sum of all intra-cluster pairwise distances.
Yet unfortunately, for most natural clustering objectives, finding the optimal solution to the
objective function is NP-hard. As a consequence, there has been substantial work on approximation algorithms~\cite{jain_new_2002,CharikarGTS02,BartalCR01,VegaKKR03,AryaGKMMP04}  with both upper and lower bounds on the approximability of these objective functions on worst case instances.

Recently, Bilu and Linial~\cite{BL} suggested an exciting, alternative  approach aimed at understanding the complexity of clustering instances which arise in practice.
Motivated by the fact that distances between data points in clustering instances are often based on a heuristic measure, they
   argue that interesting instances should be resilient to small
   perturbations in these distances.  In particular, if small
   perturbations can cause the optimum clustering for a given
   objective to change drastically, then that probably is not a
   meaningful objective to be optimizing.  Bilu and Linial~\cite{BL} specifically define an instance to be
   $\alpha$-perturbation resilient\footnote{Bilu and Linial~\cite{BL} refer to such instances as perturbation stable instances.} for an objective $\Phi$ if perturbing pairwise
   distances by multiplicative factors in the range $[1,\alpha]$ does
   not change the optimum clustering under $\Phi$. 
	They consider in detail the case of Max-Cut clustering and give an efficient algorithm to recover the optimum when the instance is  resilient to perturbations on the order of $\alpha>  \min\{n/2, \sqrt{n\Delta}\}$ where $\Delta$ is the maximal degree of the graph. They also give an efficient algorithm for unweighted Max-Cut instances that are resilient to perturbations on the order of $\alpha \geq 4n/\delta$ where $\delta$ is the minimal degree of the graph.

   Two important questions raised by the work of Bilu and Linial~\cite{BL} are:
   (1) the degree of resilience needed for their algorithm to succeed
   is quite high: can one develop algorithms for important clustering
   objectives that require much less resilience? (2) the resilience
   definition requires the optimum solution to remain {\em exactly}
   the same after perturbation: can one succeed under weaker conditions?  In the context
   of {\em center-based} clustering objectives such as $k$-median and
   $k$-center, \yingyu{Awasthi et al.~\cite{ABS10}} partially address the first of these
   questions and show that an algorithm based on the single-linkage
   heuristic can be used find the optimal clustering for
   $\alpha$-perturbation-resilient instances for $\alpha=3$.  They
   also conjecture it to be NP-hard to beat $3$ and prove beating $3$ is
   NP-hard for a related but weaker notion (see the $\alpha$-center proximity property in Definition~\ref{def:centerstability}).

     In this work, we address both questions raised by~\cite{BL} and additionally
   improve over~\cite{ABS10}. First, for the center-based objectives we design a polynomial time algorithm for finding
   the optimum solution for instances resilient to perturbations of value
   $\alpha = 1 + \sqrt{2}$, thus beating the previously best known factor of $3$ of Awasthi et al~\cite{ABS10}.
  Second, for $k$-median (which is a specific center-based objective), we consider a weaker, relaxed, and more realistic
   notion of perturbation-resilience where we allow the optimal
   clustering of the perturbed instance to differ from the optimal of
   the original in a small $\epsilon$ fraction of the points.
   Compared to the original perturbation resilience assumption, this is
   arguably a more natural though also more difficult condition to
   deal with.  We give positive results for this case as well, showing
   for somewhat larger values of $\alpha$ that we can still achieve a
   near-optimal clustering on the given instance (see Section 1.1
   below for precise results).  We additionally give positive results
   for min-sum clustering which is typically a harder objective than center-based objectives from approximability standpoint.
	For example, the best known
 guarantee for min-sum clustering on worst-case
   instances is an $O(\upsilon^{-1} \log^{1+\upsilon} n)$-approximation algorithm that runs in time
    $n^{O(1/\upsilon)}$ for any $\upsilon>0$ due to Bartal et al.~\cite{BartalCR01}; by contrast, the best guarantee known for $k$-median is
    factor $1+\sqrt{3}+\epsilon$~\cite{shi2013app} for any $\epsilon > 0$.

   Our results are achieved by carefully deriving structural properties of perturbation resilience. At a high level, all the algorithms we introduce work by first
   running appropriate linkage procedures to produce a hierarchical clustering, and then running dynamic programming to retrieve the  best $k$-clustering present in the tree. To ensure that (under perturbation resilient instances) the hierarchy output in the first step has a pruning of low cost, we derive new linkage procedures (closure linkage and robust average linkage) which are of independent interest.
   While the overall analysis is quite involved, the clustering algorithms we devise are simple and robust.
   This simplicity and robustness allow us to show how our algorithms can be made
   sublinear-time by returning an implicit clustering from only a
   small random sample of the input.

From a learning theory perspective, the resilience parameter, $\alpha$, can also be seen as an analog to a margin for clustering.
In supervised learning, the margin of a data point is the distance, after scaling, between the data point and the
decision boundary of its classifier, and many algorithms have stronger guarantees
when the smallest margin over the entire data set is sufficiently large~\cite{SchSmo02,Vapnik:book98}.
The $\alpha$ parameter, similarly controls the magnitude of the perturbation the data can withstand
before being clustered differently, which is, in essence, the data's distance to the decision boundary
for the given clustering objective.
Hence, perturbation resilience is also a natural and interesting assumption to study from a learning theory perspective.

\paragraph{Our Results}
In this paper,
we advance the line of work of~\cite{BL} by solving several
important problems of clustering perturbation-resilient instances
under metric center-based and min-sum objectives.

 In Section~\ref{section-kmedian} we improve on the bounds of~\cite{ABS10}
for $\alpha$-perturbation resilient instances for center-based objectives,
giving an algorithm that efficiently\footnote{For clarity, in this paper efficient means polynomial in both $n$
(the number of points) and $k$ (the number of clusters).}
finds the optimum clustering for $\alpha = 1+\sqrt{2}$. Most of the frequently used center-based objectives, such as $k$-median, are NP-hard to even
approximate, yet we can recover the exact solution for perturbation resilient instances.
Our algorithm is based on a new linkage procedure using a new notion of distance (closure distance) between sets that may be of independent interest.

 In Section~\ref{alphaEpsilonForMedian} we consider the more challenging and more general notion
of $(\alpha,\epsilon)$-perturbation resilience for $k$-median,
where we allow the optimal solution after perturbation to be
$\epsilon$-close to the original. We provide an efficient algorithm which for $\alpha >  2 + \sqrt{3}$ produces $(1+O(\epsilon / \rho))$-approximation to the optimum, where $\rho$ is the fraction of the points in the smallest cluster.
The key structural property we derive and exploit is that, except for $\epsilon n$ bad points, most points are
$\alpha$ times closer to their own center than to any other center.
To eliminate the noise introduced by the bad points,
we carefully partition the points into a list of sufficiently large blobs, each of which contains only good points from one optimal cluster.
This then allows us to construct a tree on the blobs with a low-cost pruning that is a good approximation to the optimum.

In Section~\ref{section-minsum} we provide the first efficient algorithm for optimally clustering
$\alpha$-perturbation resilient min-sum instances. \yingyu{We show that when $\alpha$ in the order of the ratio between the sizes of the largest and smallest clusters,  there exists an algorithm that can output the optimal clustering in  polynomial time.} Our algorithm is based on an appropriate modification of average linkage that exploits the structure of min-sum perturbation resilient instances.

In Section~\ref{sec:alpha_espilon_minsum}, we show that for $(\alpha,\epsilon)$-perturbation resilient min-sum instances with $\alpha$ in the order of the ratio between the sizes of the largest and smallest clusters and $\epsilon = \tilde{O}(\rho)$, there exists a polynomial time algorithm that outputs a clustering that is both a $(1+\tilde{O}(\epsilon/\rho))$-approximation and $\tilde{O}(\epsilon)$-close to the optimal clustering. The key structural property is that except for $\tilde{O}(\epsilon n)$ bad points, most points are $O(\alpha)$ times closer to their own optimal cluster than to any other optimal cluster.  Similar to the case of $k$-median, we can partition the points into a list of sufficiently large blobs, each of which contains only good points from one optimal cluster. However, the properties of the good points are significantly weaker than those in the $k$-median case, and thus the linkage there does not guarantee a tree with a low-cost pruning.  
To utilize these properties, we introduce the notion of potentially good points which can act as a proxy of the actual good points. We then design a robust average linkage algorithm based on the cost computed only on the potentially good points, which constructs a tree with a pruning that assigns all good points correctly. The pruning can be found out efficiently,  and after some processing it leads to a clustering that is both a good approximation and close to the optimal clustering.

We also provide sublinear-time algorithms both for the $k$-median and min-sum objectives (Sections~\ref{sec:inductive} and~\ref{sec:msinductive}), showing algorithms that can return an implicit clustering from only access to a small random sample.

\paragraph{Related Work}
A subsequent work~\cite{bilu_et_al:LIPIcs:2013} of~\cite{BL} by Bilu, Daniely, Linial and Saks studied the Max-Cut problem under perturbation resilience,
and showed how to solve in polynomial time $(1+\epsilon)$-perturbation resilient instances of metric and dense Max-Cut, and $\Omega(\sqrt{n})$-perturbation resilient instances of general Max-Cut. The later bound is further improved by Makarychev, Makarychev and Vijayaraghavan~\cite{Makarychev2013}.
They proposed a polynomial time exact algorithm for $\Omega(\sqrt{\log n}\log \log n)$-perturbation resilient Max-Cut instances based on semidefinite programming.
They also proved that for Max $k$-Cut with $k \geq 3$, there is no polynomial-time algorithm that solves $\infty$-perturbation resilient instances of Max $k$-Cut unless \NP = \RP. Here an instance is $\infty$-perturbation resilient if it is $\alpha$-perturbation resilient for every $\alpha$.
Finally, they also studied a notion called $(\gamma, N)$-weakly stability for Max-Cut, which means that after perturbing the weights by a factor of at most $\gamma$, the optimal solution must be from the set $N$. When $N$ is the set of solutions that differ from the optimal solution on at most $\delta$ fraction of nodes, the notion is the same as the $(\gamma, \delta)$-perturbation resilience studied in our work.  They showed that when $\gamma = \Omega(\sqrt{\log n}\log \log n)$, there exists an efficient algorithm that can find a cut from $N$.
In a recent work~\cite{mak2014constant}, the same authors further proposed a beyond worst-case analysis model for Balanced-Cut, which is a planted model with random edges from permutation-invariant distributions. They achieved a constant factor approximation with respect to the cost of the planted cut when the number of random edges is $\Omega(n \mathrm{polylog}(n))$.


Several recent papers have showed how to exploit the structure of perturbation resilient instances in order to obtain better approximation guarantees (than those possible on worst case instances) for other difficult optimization problems. These include the game theoretic problem of finding Nash equilibria~\cite{BalcanB10,LiptonMM06} and the classic traveling salesman problem~\cite{MihalakSSW11}.

In the context of objective based clustering, several recent papers have showed how to exploit other notions of stability for overcoming the existing hardness results on worst case instances.
The ORSS stability notion of Ostrovsky, Rabani, Schulman and Swamy~\cite{Lloyd06,ABS10} assumes that
the cost of the optimal $k$-means solution is small compared to the cost of the optimal
$(k-1)$-means solution.
The BBG $(c,\epsilon)$-approximation stability condition of Balcan, Blum and Gupta~\cite{BBG} assumes that every $c$-approximation solution is close to the target clustering.  We note that when the target clustering is the optimal clustering for the clustering objective, $(c,\epsilon)$-approximation stability implies $(c,\epsilon)$-perturbation resilience. 

Awasthi, Sheffet and Blum~\cite{Awasthi:2010:SYP} proposed a stability condition called weak-deletion stability,
and showed that it is implied by both the ORSS stability and the BBG stability.
Kumar and Kannan~\cite{Kumar:2010:CSN} proposed a proximity condition which assumes that in the target clustering, most data points satisfy that they are closer to their center than to any other center by an additive factor in the order of the maximal standard variance of their clusters in any direction.
Their results are improved by Awasthi and Sheffet~\cite{awasthi2012improved}, which proposed a weaker version of the proximity condition called center separation, and designed algorithms achieving stronger guarantees under this weaker condition. These notions are not directly comparable to the perturbation resilience property.

\section{Notation and Preliminaries}

In a clustering instance, we are given a set $S$ of $n$ points in a finite metric space, and we denote
$d: S \times S \rightarrow \mathbb{R}_{\geq 0}$ as the distance function. $\Phi$ denotes the objective
function over a partition of $S$ into $k < n$ clusters which we want to optimize over the metric, that is, $\Phi$ assigns
a score to every clustering.
The optimal clustering with respect to $\Phi$ is denoted as $\mathcal{C} = \{C_1,C_2,\dots, C_k\}$,
and its cost is denoted as $\mathcal{OPT}$.
The core concept we study in this paper is the perturbation resilience notion introduced by~\cite{BL}. Formally:

\begin{definition} \label{def:alphaPR}
A clustering instance $(S,d)$ is {$\alpha$-perturbation resilient} to a given objective
$\Phi$ if for any function $d': S \times S \rightarrow \mathbb{R}_{\geq 0}$
such that $\forall p,q \in S, d(p,q) \leq d'(p,q) \leq \alpha d(p,q)$, there is a unique optimal clustering
$\mathcal{C'}$ for $\Phi$ under $d'$ and this clustering is equal to the optimal clustering
$\mathcal{C}$ for $\Phi$ under $d$.
\end{definition}

Note that in the definition, $d'$ need not be a metric. Also note that the definition depends on the objective.
In this paper, we focus on the center-based and min-sum objectives.
For the \emph{center-based objectives}, 
we consider separable center-based objectives defined by~\cite{ABS10}.
\begin{definition} A clustering objective is center-based if
the optimal solution can be defined by $k$ points $c_
1 , \cdots , c_k$ in
the metric space called centers such that every data point is
assigned to its nearest center.
Such a clustering objective
is separable if it furthermore satisfies the following two
conditions:
\begin{itemize}
\item[(1)] The objective function value of a given clustering is
either a (weighted) sum or the maximum of the individual
cluster scores.
\item[(2)] Given a proposed single cluster, its score can be computed
in polynomial time.
\end{itemize}
\end{definition}
One particular center-based objective is the $k$-median objective.
We partition $S$ into $k$ disjoint subsets $\mathcal{P} = \{P_1,P_2,\dots, P_k\}$ and
assign a set of centers $\mathbf{p} = \{p_1,p_2,\dots, p_k\} \subseteq S$ for the subsets.
The objective is  $\Phi(\mathcal{P}, \mathbf{p}) = \sum_{i=1}^k\sum_{p\in P_i}d(p,p_i)$. The centers in the optimal clustering are denoted as $\mathbf{c} = \{c_1,\dots,c_k\}$.
Clearly, in an optimal solution, each point is assigned to its nearest center.
In such cases, the objective is denoted as $\Phi(\mathbf{c})$.

For the \emph{min-sum objective}, we partition $S$ into $k$ disjoint subsets denoted as $\mathcal{P} = \{P_1,P_2, \dots, P_k\}$, and the goal is to minimize $\Phi(\mathcal{P}) = \sum_{i=1}^k\sum_{p\in P_i} \sum_{q \in P_i} d(p,q)$.
Note that we sometimes denote $\Phi$ as $\Phi_S$ in the case where the distinction is necessary, such as in Section~\ref{sec:inductive}.

In Section~\ref{alphaEpsilonForMedian}  we consider a generalization of perturbation resilience where
we allow
a small difference between the original optimum and the new optimum
after perturbation. Formally:

\begin{definition}
Let $\mathcal{C}$ be the optimal $k$-clustering and $\mathcal{C'}$ be another $k$-clustering
of a set of $n$ points. We say $\mathcal{C'}$ is $\epsilon$-{close} to $\mathcal{C}$ if
$\min_{\sigma \in \mathcal{S}_k} \sum_{i=1}^k |C_i \setminus C'_{\sigma(i)}| \le \epsilon n$, where
$\sigma$ is a matching between indices of clusters of
$\mathcal{C'}$ and those of $\mathcal{C}$.
\end{definition}
\begin{definition}
\label{def:alphaPRr}
 A clustering instance $(S,d)$ is
{$(\alpha, \epsilon)$-perturbation resilient} to a given objective
$\Phi$ if for any function $d': S \times S \rightarrow \mathbb{R}_{\geq 0}$
s.t.\ $\forall p,q \in S, d(p,q) \leq d'(p,q) \leq \alpha d(p,q)$, the optimal clustering
$\mathcal{C'}$ for $\Phi$ under $d'$ is $\epsilon$-close to the optimal clustering
$\mathcal{C}$ for $\Phi$ under $d$.
\end{definition}

For simplicity, we assume $\epsilon n$ is an integer and assume that $\min_i |C_i|$ is known (otherwise, we can simply search over the $n$ possible different values). 

For $A, B \subseteq S$ and a distance function $d$, we define $\ds(A,B) \defeq \sum_{p\in A, q\in B}d(p,q)$, $\ds(p, B) \defeq \ds(\{p\}, B)$, and $\ds(p, q) \defeq \ds(\{p\}, \{q\})$. Also, we define $\da(A,B) \defeq \ds(A, B) / (|A||B|)$ and $\da(p, B) \defeq \da(\{p\}, B)$ for nonempty $A$ and $B$.

\section{$\alpha$-Perturbation Resilience for Center-based Objectives}\label{section-kmedian}

In this section we show that,  for $\alpha \geq 1+\sqrt{2}$, if the clustering instance is $\alpha$-perturbation resilient for center-based objectives, then we can in polynomial time find the optimal clustering.
This improves on the $\alpha \ge 3$ bound of~\cite{ABS10} and
stands in sharp contrast to the NP-Hardness
results on worst-case instances.  Our algorithm succeeds for an even weaker property, the $\alpha$-center proximity, introduced
in~\cite{ABS10}.

\begin{definition} \label{def:centerstability}
A clustering instance $(S,d)$ satisfies the {$\alpha$-center proximity} property
if for any optimal cluster $C_i\in\mathcal{C}$ with center $c_i$, $C_j\in \mathcal{C}(j\neq i)$
with center $c_j$,  any point $p\in C_i$ satisfies $\alpha d(p,c_i) < d(p,c_j)$.
\end{definition}

\begin{lemma}\label{kmedian-lemma-basic}
Any clustering instance that is $\alpha$-perturbation resilient to center-based objectives also
satisfies the $\alpha$-center proximity.
\end{lemma}

The proof follows easily by constructing a specific perturbation that blows up all the pairwise distances within cluster $C_i$ by a factor of $\alpha$.
By $\alpha$-perturbation resilience, the optimal clustering remains the same after this perturbation.
This then implies the desired result. The full proof appears in~\cite{ABS10}.
In the remainder of this section, we prove our results for $\alpha$-center proximity, but because
it is a weaker condition, our upper bounds also hold for $\alpha$-perturbation resilience.

We begin with some key properties of $\alpha$-center proximity instances.

\begin{lemma}\label{kmedian-lemma-between}\label{kmedian-lemma-inout}\label{kmedian-lemma-inout2}
For any points $p \in C_i$ and $q \in C_j (j \neq i)$ in the optimal clustering of an $\alpha$-center proximity instance,
we have
\begin{itemize}
\item[(1)] $d(c_i, q) > \frac{\alpha(\alpha-1)}{\alpha+1}d(c_i,p)$,
\item[(2)] $d(p,q) > (\alpha -1) \max\{d(p,c_i), d(q,c_j)\}$.
\end{itemize}
Consequently, when $\alpha \geq 1 + \sqrt{2}$, we have
\begin{itemize}
\item[(1)] $d(c_i, q) > d(c_i,p)$,
\item[(2)] $d(p,q) > d(p,c_i)$.
\end{itemize}
\end{lemma}

\begin{proof}
(1) Lemma~\ref{kmedian-lemma-basic} gives us that $d(q, c_i) > \alpha d(q, c_j)$.
By the triangle inequality, we have
$d(c_i, c_j) \leq d(q, c_j) + d(q, c_i) < (1 + 1/\alpha) d(q, c_i)$.
On the other hand, $d(p,c_j) > \alpha d(p,c_i)$ and therefore
$d(c_i, c_j) \geq d(p, c_j) - d(p, c_i) > (\alpha - 1) d(p, c_i)$.
Combining these inequalities, we get (1).

(2) The proof first appears in~\cite{ABS10}, and we include it for completeness.
Without loss of generality, we can assume that $d(p, c_i) \geq d(q, c_j)$. By the triangle inequality we have $d(p, q) \geq
d(p, c_j) - d(q, c_j)$. From Lemma~\ref{kmedian-lemma-basic}  we have $d(p, c_j) > \alpha d(p, c_i)$.
Hence $d(p, q) > \alpha d(p, c_i) - d(q, c_j) \geq (\alpha - 1)d(p, c_i) \geq (\alpha - 1)d(q, c_j)$.
\end{proof}

Lemma~\ref{kmedian-lemma-inout2} implies for any optimal cluster $C_i$,
 the ball of radius $\max_{p \in C_i}d(c_i,p)$ around the center $c_i$ contains {only} points from $C_i$, and moreover, points inside the ball are each closer to the center than to any point outside the ball.
Inspired by this structural property, we define
the notion of closure distance between two sets as
the radius of the minimum ball that covers the sets and
has some margin from points outside the ball.
We show that any (strict) subset of an optimal cluster
has smaller closure distance to another subset in the same cluster
than to any subset of other clusters or to unions of other clusters.
Using this, we will be able to define an appropriate linkage
procedure that, when applied to the data, produces a tree on subsets
that will all be laminar with respect to the clusters in the optimal
solution.  This will then allow us to extract the optimal solution using
dynamic programming applied to the tree.

We now define the notion of closure distance and then present our algorithm for $\alpha$-center proximity instances (Algorithm~\ref{ClosureLinkage}). Let $\mathbb{B}(p, r) \defeq \{q: d(q,p)\leq r\}$ denote the ball around $p$ with radius $r$.

\begin{definition}
The {closure distance} $d_S (A, A')$ between two disjoint nonempty
subsets $A$ and $A'$
of point set $S$ is the minimum $d\geq 0$ such that there is a point $c\in A\cup A'$ satisfying
the following requirements:
\begin{itemize} 
\item[(1)] coverage: the ball $\mathbb{B}(c,d)$ covers $A$ and $A'$, that is, $A\cup A'\subseteq \mathbb{B}(c,d)$;
\item[(2)] margin: points inside $\mathbb{B}(c,d)$ are closer to the center $c$ than to points outside, that is, $\forall p\in \mathbb{B}(c,d), q\not\in \mathbb{B}(c,d)$, we have $d(c,p) < d(p,q)$.
\end{itemize}
\end{definition}

\begin{figure}[!tbhp]
\centering

\begin{tikzpicture}

\def\maxx{5}
\def\maxy{2}
\def\bbox{(-\maxx,-\maxy) rectangle (\maxx,\maxy)}
\clip \bbox;

\draw[blue]  (0,0) ellipse (1.5); \node [left] {$c$};
\draw[fill=blue!25]  (-0.4,0) ellipse (0.7 and 0.6);
\draw[fill=blue!25]  (1,0) ellipse (0.4 and 0.7);
\draw[fill=violet!25]  (-3.6,0) ellipse (0.5 and 1);
\draw[fill=green!25]  (3.8,0) ellipse (1 and 0.5);
\node at (-0.4,0.9) {$A$};
\node at (0.8,0.9) {$A'$};

\coordinate (c) at (0,0);
\coordinate (p) at (1,-0.2);
\coordinate (q) at (3.2,0);
\coordinate (d) at (0.6,-1.38);

\draw[blue,thick,->] (c) -- (d) node[midway,left] {$d$};
\draw[blue,thick] (c) -- (p);
\draw[red,thick] (q) -- (p);
\fill[blue] (0,0) circle (0.1) node[above] {$c$};
\fill[blue] (p) circle (0.1) node[below] {$p$};
\fill[red] (q) circle (0.1) node[below] {$q$};

\end{tikzpicture}
\caption{Illustration for the closure distance.}\label{fig:closureDistance}
\end{figure}
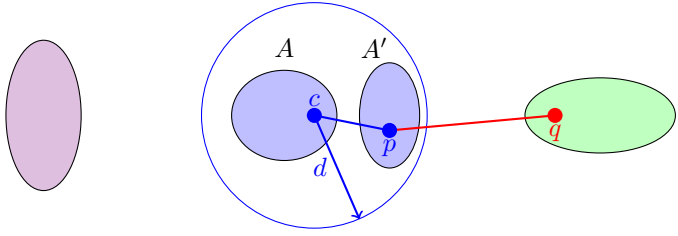

Note that $d_\data (A, A')=d_S (A', A)  \leq \max_{p,q\in \data}d(p,q)$ for any  $A$ and $A'$. Furthermore, it can be computed in polynomial time.

\begin{algorithm}[tbhp]
\caption{Center-based objectives, $\alpha$-perturbation resilience}
\label{ClosureLinkage}
\begin{algorithmic}[1]
\REQUIRE{Data set $\data$, distance function $d(\cdot, \cdot)$ on $\data$.}
\STATE{Begin with $n$ singleton clusters.}
\STATE{Repeat till only one cluster remains:\\ merge clusters $C,C'$ which minimize
$d_\data(C,C')$. }
\STATE{Let $\mathcal{T}$ be the tree with single points as leaves and internal nodes corresponding to the merges performed.}
\STATE{Run dynamic programming on $\mathcal{T}$ to get the minimum cost pruning ${\cal \tilde{C}}$.}
\ENSURE{Clustering ${\cal \tilde{C}}$.}
\end{algorithmic}
\end{algorithm}

\begin{theorem}\label{thm:alphamain}
For $(1+\sqrt{2})$-center proximity instances, Algorithm~\ref{ClosureLinkage} outputs the optimal clustering in polynomial time.
\end{theorem}

The proof follows immediately from the following key property of the Phase 1 of  Algorithm~\ref{ClosureLinkage}.
The details of dynamic programming are presented in Appendix~\ref{subsec:dp},
and an efficient implementation of the algorithm is presented in Appendix~\ref{subsec:running_time}.

\begin{theorem} \label{thm:k-median-alg}  For $(1+\sqrt{2})$-center proximity instances,
Algorithm~\ref{ClosureLinkage} constructs a binary tree $\mathcal{T}$ such that the optimal clustering is a pruning of this tree.
\end{theorem}

\begin{proof}
We prove correctness by induction.
In particular, assume that our current clustering is {laminar} with respect to the
optimal clustering. That is, for each cluster $A$ in our current clustering
and each $C$ in the optimal clustering,
we have either $A\subseteq C$, or $C\subseteq A$, or $A \mycap C = \varnothing$.
This is clearly true at the start.
To prove that the merge steps keep the laminarity,
we need to show the following:
if $A$ is a strict subset of an optimal cluster $C_i$, $A'$ is a subset of another optimal cluster
or the union of one or more other clusters,
then there exists $B$ from $C_i \setminus A$,  such that $d_\data(A,B) < d_\data(A,A') $.

We first prove that there is a cluster $B \subseteq C_i \setminus A$ in the current cluster list such that $d_\data(A, B) \leq \tilde d \defeq \max_{p\in C_i}d(c_i, p)$. 
There are two cases.
First, if $c_i \not\in A$, then define $B$ to be the cluster
in the current cluster list that contains $c_i$.
By induction, $B\subseteq C_i$ and thus $B\subseteq C_i \setminus A$.
Then we have $d_\data(B,A) \leq \tilde d$ since there is $c_i \in B$, and (1) for any
$p\in A\mycup B$, $d(c_i, p)\leq \tilde d$,
(2) for any $p\in \data$ satisfying $d(c_i, p)\leq \tilde d$, and any $q\in \data$ satisfying $d(c_i,q) > \tilde d$,
by Lemma~\ref{kmedian-lemma-inout2} we know $p\in C_i$ and $q\not\in C_i$, and thus $d(c_i,p) < d(p, q)$.
In the second case when $c_i \in A$, we pick any $B \subseteq C_i \setminus A$ and a
similar argument gives $d_\data(A,B) \leq \tilde d$.

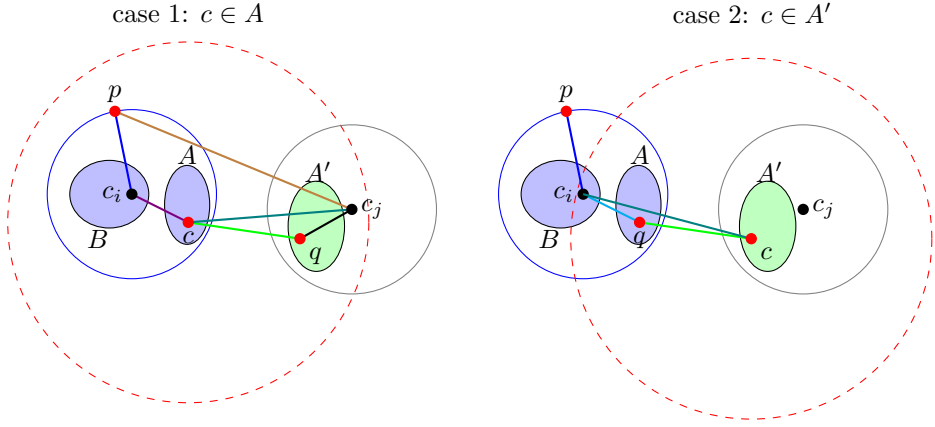
\begin{figure}[!tbhp]
\centering

\begin{tikzpicture}[scale=0.75]

\def\maxx{8.4}
\def\maxy{4.1}
\def\bbox{(-\maxx,-\maxy) rectangle (\maxx,\maxy)}
\clip \bbox;

\node at (-5,3.2) {case 1: $c \in A$};
\node at (5,3.2) {case 2: $c \in A'$};

\begin{scope}[shift={(-6cm,0cm)}]
\coordinate (ci) at (0,0) {};
\coordinate (cj) at (3.9,-0.27) {} {};
\coordinate (c) at (1,-0.5) {} {} {};
\coordinate (q) at (2.98,-0.79) {} {} {} {} {} {};
\coordinate (p) at (-0.3,1.47) {} {};

\draw[blue]  (0,0) ellipse (1.5); 
\draw[fill=blue!25]  (-0.4,0) ellipse (0.7 and 0.6);
\draw[fill=blue!25]  (0.98,-0.19) ellipse (0.4 and 0.7);
\draw[fill=green!25]  (3.27,-0.57) ellipse (0.5 and .8);
\node at (-0.6,-0.8) {$B$};
\node at (0.98,0.7) {$A$};

\draw[red,dashed] (c) circle (3.2);
\draw[blue,thick] (ci) -- (p);
\draw[blue,thick,violet] (ci) -- (c);
\fill[black] (0,0) circle (0.1) node [left] {$c_i$};

\node at (3.3,0.4) {$A'$};
\draw[red,thick,black] (cj) -- (q);
\draw[gray] (cj) circle (1.5);

\draw[brown,thick] (cj) -- (p);
\draw[green,thick] (c) -- (q);
\draw[teal,thick] (c) -- (cj);

\fill[red] (c) circle (0.1) node [below] {\textcolor{black}{$c$}};
\fill[red] (p) circle (0.1) node [above] {\textcolor{black}{$p$}};
\fill[black] (cj) circle (0.1) node [right] {$c_j$};
\fill[red] (q) circle (0.1) node [below right] {\textcolor{black}{$q$}};
\end{scope}

\begin{scope}[shift={(2cm,0cm)}]
\coordinate (ci) at (0,0) {};
\coordinate (cj) at (3.9,-0.27) {} {};
\coordinate (q) at (1,-0.5) {} {} {};
\coordinate (c) at (2.98,-0.79) {} {} {} {} {} {};
\coordinate (p) at (-0.3,1.47) {} {};

\draw[blue]  (0,0) ellipse (1.5); 
\draw[fill=blue!25]  (-0.4,0) ellipse (0.7 and 0.6);
\draw[fill=blue!25]  (0.98,-0.19) ellipse (0.4 and 0.7);
\draw[fill=green!25]  (3.27,-0.57) ellipse (0.5 and .8);
\node at (-0.6,-0.8) {$B$};
\node at (1,0.7) {$A$};

\draw[red,dashed] (c) circle (3.2);
\draw[blue,thick] (ci) -- (p);
\draw[cyan,thick] (ci) -- (q);
\fill[black] (0,0) circle (0.1) node [left] {$c_i$};

\node at (3.3,0.4) {$A'$};
\draw[gray] (cj) circle (1.5);

\draw[green,thick] (c) -- (q);
\draw[teal,thick] (ci) -- (c);

\fill[red] (c) circle (0.1) node [below right] {\textcolor{black}{$c$}};
\fill[red] (p) circle (0.1) node [above] {\textcolor{black}{$p$}};
\fill[black] (cj) circle (0.1) node [right] {$c_j$};
\fill[red] (q) circle (0.1) node [below] {\textcolor{black}{$q$}};
\end{scope}

\end{tikzpicture}
\caption{Comparing $\tilde d$ and $d_{\data}(A,A')$ in closure linkage.}\label{fig:closureLinkage}
\end{figure}

As a second step, we need to show that $\tilde d < \hat{d} \defeq d_\data(A,A')$.
There are two cases: the center for $d_\data(A,A')$ is in $A$ or in $A'$.
See Figure~\ref{fig:closureLinkage} for an illustration.
In the first case, there is a point $c\in A$ such that $c$ and $\hat{d}$ satisfy the requirements of the closure distance.
Pick a
point $q \in A'$, and define $C_j$ to be the cluster in the optimal clustering that contains $q$.
As $d(c,q)\leq \hat{d}$, and by
Lemma~\ref{kmedian-lemma-inout2} we have $d(c_j, q) < d(c,q)$,
then $d(c_j, c) \leq \hat{d}$ (otherwise it violates the
second requirement of closure distance).
Suppose $p = \arg\max_{p'\in C_i} d(c_i, p')$.
Then we have
$\tilde d = d(p, c_i) < d(p,c_j)/\alpha \leq (\tilde d + d(c_i,c)+d(c,c_j))/ \alpha$
where the first inequality comes from Lemma~\ref{kmedian-lemma-basic} and the second from the triangle inequality.
Since $ d(c_i,c) < d(c,c_j) / \alpha$, we can combine the above inequalities and compare $\tilde d$ and $d(c, c_j)$,
and when $\alpha \geq 1 + \sqrt{2}$ we have $\tilde d < d(c, c_j) \leq \hat{d}$.

Now consider the second case, when
there is a point $c \in A'$ such that $c$ and $\hat{d}$ satisfy the requirements
in the definition of the closure distance.
Select an arbitrary point $q \in A$.
We have $\hat{d} \geq d(c, q)$ from the first requirement, and $d(c,q) > d(c_i, q)$ by
Lemma~\ref{kmedian-lemma-between}.
Then from the second requirement of closure distance $d(c_i, c) \leq \hat{d}$.
And by Lemma~\ref{kmedian-lemma-inout2}, $\tilde d = d(c_i, p) < d(c_i, c)$,
we have $\tilde d < d(c_i, c) \leq \hat{d}$.
\end{proof}



{\note}
Our factor of $\alpha =1+\sqrt{2}$ beats the NP-hardness {lower bound} of
$\alpha = 3$ of~\cite{ABS10} for center-proximity instances. The reason is that the lower bound of~\cite{ABS10} requires the
addition of Steiner points that can act as centers but are not part of
the data to be clustered (though the upper bound of~\cite{ABS10} does not allow
such Steiner points).  One can also show a lower bound for
center-proximity instances without Steiner points.  In particular for any $\epsilon > 0$, the problem of solving $(2 - \epsilon)$-center proximity $k$-median
instances is NP-hard~\cite{cReyzin12}. 
There is also a low bound for perturbation resilience. 
Balcan, Haghtalab and White~\cite{balcan2015symmetric} recently showed that there is no polynomial time algorithm for $k$-center instances under $(2 - \epsilon)$-perturbation resilience, unless \NP = \RP. 
They also showed that closure linkage solves $k$-center instances under 2-perturbation resilience in polynomial time.

{\note}
The first condition in our definition of closure distance is similar to
the minimax linkage criteria~\cite{bien2011hierarchical}. More precisely, our closure distance
definition has two conditions: coverage condition and margin condition. If
the margin condition is removed from the definition, then the closure
distance reduces to the minimax linkage distance. For our purposes
however, the margin condition is crucial --- in particular, we can
provably argue that when the center promixity condition is satisfied 
Algorithm~\ref{ClosureLinkage} produces a tree such that the optimal clustering is a pruning
of the tree (Theorem~\ref{thm:alphamain}).

\section{$(\alpha, \epsilon)$-Perturbation Resilience for the $k$-Median Objective}\label{alphaEpsilonForMedian}

In this section we consider a natural  relaxation of the $\alpha$-perturbation resilience,
 the $(\alpha,\epsilon)$-perturbation resilience property,
that requires the optimum after perturbation of up to
a multiplicative factor $\alpha$ to be $\epsilon$-close to the original (one should think of $\epsilon$ as sub-constant).
We show that if the instance is $(\alpha,\epsilon)$-perturbation resilient with
$\alpha > 2+\sqrt{3}$, then we can in polynomial time output a clustering
that provides a $(1 + 5\epsilon/\rho)$-approximation to the optimum, 
where $\rho$ is the fraction of the points in the smallest cluster. Thus this improves over the best worst-case approximation guarantees known~\cite{shi2013app} when $\epsilon \leq \sqrt{3}\rho/5$ and also
beats the lower bound of $(1 + 1/e)$ on the best approximation achievable
on worst case instances for the metric $k$-median objective \cite{GuhaK99,jain_new_2002} when $\epsilon \leq \rho/(5e)$.

The key idea is to understand and leverage the structure implied by
$(\alpha,\epsilon)$-perturbation resilience. We show that perturbation resilience
implies that there exists only a small fraction of points that are bad in the sense that their distance to their own center is not $\alpha$ times smaller than their distance to any other centers in the optimal solution. We then use this
 bounded number of bad points in our clustering algorithm.

\subsection{Structure of $(\alpha,\epsilon)$-Perturbation Resilience}\label{structure}
\yingyu{
Throughout this section we will assume that $|C_i|$ is sufficiently large compared to $\epsilon n$, since for interesting practical clustering instances, one would expect that a large fraction of a optimal cluster will remain the same after small perturbation. The exact bound will be stated explicitly in our main theorems. For now we can simply assume $|C_i|>2\epsilon n$ for all $i$.
}

To understand the structure of $(\alpha, \epsilon)$-perturbation resilience,
we need to consider the difference
between the optimal clustering $\mathcal{C}$ under $d$
and the optimal clustering $\mathcal{C'}$ under a perturbation $d'$,
defined as $\min_{\sigma\in \mathcal{S}_k}\sum_{i=1}^k|C_i \setminus C'_{\sigma(i)}|$. 
Since $\sum_{i=1}^k |C_i \setminus C'_{\sigma(i)}| \leq \epsilon n$ by
assumption, we
clearly have separately
for each $i$ that $|C_i \setminus C'_{\sigma(i)}| \leq \epsilon n$.
Since $|C_i|>2\epsilon n$ this implies that $C'_{\sigma(i)}$ is
the unique cluster in ${\cal C'}$ such that $|C_i \cap C'_{\sigma(i)}| > \frac{1}{2}|C_i|$.
Without loss of generality, let us index $\cal C'$ so that $\sigma$ is
the identity.
We denote by $c'_i$ the center of $C'_i$.

In the following we introduce the notions of bad points and good points, 
and then show that under perturbation resilience we do not have too many bad points.
\begin{definition}
Define bad points for $k$-median to be those that are not $\alpha$ times closer to its own center than to any other center in the optimal clustering.
That is,
$$
  B:= \cup_i B_i,~B_i := \{p\in C_i: \exists j\neq i, \alpha d(c_i,p) \geq d(c_j,p)\}.
$$
The other points $G := \data\setminus B$ are called good points.
Let $G_i := G \cap C_i$ denote the good points in cluster $C_i$.
\end{definition}

\begin{theorem} \label{badPointsBoundTheorem}
Suppose the clustering instance is $(\alpha, \epsilon)$-perturbation resilient
and $\min_i |C_i| > 6 \left(\frac{\alpha+1}{\alpha - 1} \right)(\epsilon n + \alpha + 1)$.
Then $|B| \leq \epsilon n$.
\end{theorem}
\paragraph{Intuition}

Assume for contradiction that $|B| > \epsilon n$.
The main idea is to select a subset of $(\epsilon n + 1)$ bad points and then construct a specific perturbation so that in the new optimal clustering these (and only these) selected bad points move to new clusters, leading to a clustering that is $\epsilon$ far from the original optimal clustering. This is contradictory to the $(\alpha, \epsilon)$-perturbation resilience property, and thus there are at most $\epsilon n$ bad points.

The selected bad points and the perturbation are defined as follows.
Select an arbitrary subset $\hat{B}$ of $(\epsilon n + 1)$ bad points from $B$,
and let $\hat{B}_i = \hat{B} \cap C_i$ denote the selected bad points in $C_i$. 
Let $c(p)$ denote the second nearest center for $p\in \hat{B}_i$ and the nearest center for $p \in C_i \setminus\hat{B}_i$.
That is, for any $1 \leq i \leq k$ and any $p \in C_i$, let
\begin{eqnarray*}
c(p) = 
\begin{cases}
c_j \text{ where } j=\arg\min_{j'\neq i} d(p,c_{j'}) & \text{if } p \in \hat{B}_i\\
c_i & \text{if } p \in C_i \setminus \hat{B}_i.
\end{cases}
\end{eqnarray*}
The perturbation blows up all distances by a factor of $\alpha$ except for those distances between $p$ and $c(p)$.
Formally,
\begin{eqnarray*}
d'(p,q) = \left\{ \begin{array}{ll}
d(p,q) & \textrm{if $p =c(q)$, or $q = c(p)$,}\\
\alpha d(p,q) & \textrm{otherwise}.
\end{array} \right.
\end{eqnarray*}

The key challenge in showing the contradiction is to show that $c'_i = c_i$ for all $i$, that is, the optimal centers do not change after the perturbation. Once this is shown
it is then immediate that in the optimum clustering under $d'$ each point $p$ is
assigned to the center $c(p)$, and thus the selected bad points $\hat{B}$ will move from their original optimal clusters
and all others will not.
\yingyu{So the distance between the new clustering
and the original clustering is $|\hat{B}| > \epsilon n$, which} is contradictory to the $(\alpha, \epsilon)$-perturbation resilience property.  

\begin{figure}[th]
\centering
\subfloat[Notations $A_i$ and $M_i$]{%

\begin{tikzpicture}
\usetikzlibrary{shapes,backgrounds}

\def\maxx{2.5}
\def\maxy{1.7}
\def\bbox{(-\maxx,-\maxy) rectangle (\maxx,\maxy)}
\clip \bbox;

\coordinate (centerC) at (-0.25,0);
\def\C{(centerC) circle (1.2)}

\coordinate (centerCt) at (0.25,0);
\def\Ct{(centerCt) circle (1.2)}

\fill[red!25] \C;
\fill[green!25] \Ct;

\begin{scope}
	\clip \C;
	\fill[red!50!green!50] \Ct;
\end{scope}

\draw \C;
\draw \Ct;

%


\node at (-2,-1.1) {$C'_i$};
\node at (2,-1.1) {$C_i$};
\node at (-1.2,0) {$A_i$};
\node at (1.2,0) {$M_i$};

\end{tikzpicture}
}
\quad\quad\quad
\subfloat[Notations $W_i$ and $V_i$]{%

\begin{tikzpicture}
\usetikzlibrary{shapes,backgrounds}

\def\maxx{2.5}
\def\maxy{1.7}
\def\bbox{(-\maxx,-\maxy) rectangle (\maxx,\maxy)}
\clip \bbox;

\coordinate (centerC) at (-0.25,0);
\def\C{(centerC) circle (1.2)}

\coordinate (centerCt) at (0.25,0);
\def\Ct{(centerCt) circle (1.2)}


\begin{scope}
	\clip \C;
	\fill[yellow!25] \Ct;
\end{scope}

\draw \C;
\draw \Ct;

\coordinate (a) at (-2,0.7) {} {} {};
\coordinate (b) at (2,0.7) {};
\def\FPup{(a) -- (b) -- ++(0,2) -- ++(-2.7,0) -- cycle}

\begin{scope}
	\clip \C;
	\clip \FPup;
	\fill[magenta] \Ct;
\end{scope}

\draw (a) -- (b);

\node at (-2,-1.1) {$C'_i$};
\node at (2,-1.1) {$C_i$};
\node at (0,0) {$W_i$};
\node at (0,0.94) {$V_i$};
\node at (1.8,0.94) {$\hat{B}_i$};

\end{tikzpicture}
}
\caption{Different types of points. (a) $A_i = C'_i \setminus C_i, M_i = C_i \setminus C'_i$. (b) $W_i = (C_i \cap C'_i) \setminus \hat{B}_i, V_i = (C_i \cap C'_i) \setminus \hat{B}_i$. As a result, $C_i = W_i \dcup V_i \dcup M_i$ and $C'_i = W_i \dcup V_i \dcup A_i$. }\label{fig:alphaEpsilonNotation}
\end{figure}
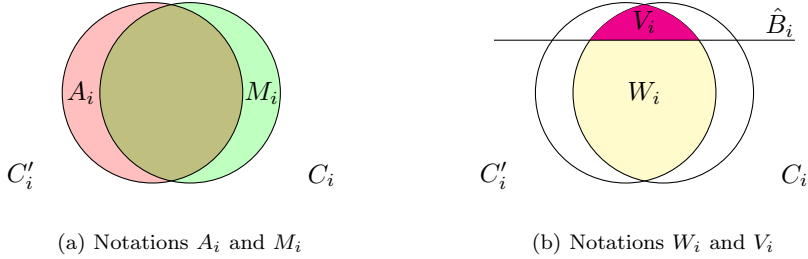

It will now be convenient to define a few quantities.  Let $A_i = C_i' \setminus C_i$ (the points added when switching from $C_i$ to $C_i'$), $M_i = C_i \setminus C_i'$ (the points removed), $W_i = (C_i \cap C_i') \setminus \hat{B}_i$ (the common points excluding selected bad points), and $V_i = (C_i \cap C_i') \cap \hat{B}_i$ (the selected bad points in common).  So, $C_i = W_i  \dcup V_i \dcup M_i$ and $C_i' = W_i \dcup V_i \dcup A_i$.  See Figure~\ref{fig:alphaEpsilonNotation}.  Note that $|A_i| \leq \epsilon n$, $|M_i| \leq \epsilon n$, and $|V_i| \leq \epsilon n+1$, with the bulk of the points in $W_i$.

The intuition for the proof that $c_i' = c_i$ is the following.  Assume for contradiction that $c_i' \neq c_i$. First, $d(c_i,c_i')$ cannot be too large compared to the average distance between $c_i$ and $W_i$, else by the triangle inequality $\ds(c_i',W_i)$ would also be large, violating the fact that $\dps(c'_i, C'_i) \leq \dps(c_i, C'_i)$; see Claim~\ref{newold}.  On the other hand, if $d(c_i,c_i')$ is small then by the triangle inequality $\ds(c_i',W_i) \approx \ds(c_i,W_i)$.  Since distances between $c_i'$ and $W_i$ are blown up by a factor of $\alpha$ in moving from $d$ to $d'$ but distances between $c_i$ and $W_i$ are not, $\dps(c_i',W_i)$ will be significantly larger than $\dps(c_i,W_i)$, which will also violate the fact that $\dps(c'_i, C'_i) \leq \dps(c_i, C'_i)$; see Claim~\ref{otherdirection}.

\begin{proofof}{Theorem~\ref{badPointsBoundTheorem}}
We now present the formal proof. Before proving the two key claims mentioned in the intuition, we begin with two convenient claims. The first convenient claim shows that $c_i' \neq c_j$ for $j \neq i$. 
The second convenient claim shows the relation of $\dps$ and $\ds$ on $A_i$. 

\begin{claim}\label{newneqold}
If $\min_i |C_i| > (\frac{2}{\alpha - 1} + 3) \epsilon n + 1$, then $c'_i \neq c_j (\forall j\neq i)$.
\end{claim}
\mycomment{The proof is moved from appendix to here. We can also put it at the end of the section but putting it here seems fine, since its not too long.}
\begin{proof}
Assume for contradiction that $c'_i = c_j$.
We first need to show $c'_j \neq c_l (\forall l)$.
Clearly, $c'_j\neq c_j$, since otherwise,
moving all the points in $C'_j$ to $C'_i$ will not increase the cost,
which violates $(\alpha, \epsilon)$-perturbation resilience.
We also know that $c'_j \neq c_l (l\neq j)$ since otherwise,
there is $p\in W_j$, $d(c_l, p) = d(c'_j, p) \leq d'(c'_j,p) < d'(c'_i, p) = d(c_j, p)$,
which contradicts the fact that $p\in C_j$.

Now we can apply the intuition described above to show that $c'_i = c_j$
and $c'_j \neq c_l (\forall l)$ lead to an contradiction.
Note that points in $ W_j \cup V_j = C_j \cap C'_j$ are closer to $c'_j$ than to $c'_i = c_j$ under $d'$.
Then back to $d$, for any $p\in W_j$, since $c'_j \neq c_l (\forall l)$,
$\alpha d(c'_j, p) = d'(c'_j,p )\leq d'(c'_i,p) = d(c_j, p)$,
resulting in $\ds(c'_j, W_j) \leq  \ds(c_j, W_j)/\alpha$.
Similarly, for any $p \in V_j$, $\alpha d(c'_j, p) = d'(c'_j,p )\leq d'(c'_i,p) = \alpha d(c_j, p)$,
resulting in $\ds(c'_j, V_j) \leq  \ds(c_j, V_j)$.
These facts have two consequences.

First, since points in $W_j$ are $\alpha$ time closer to $c'_j$ than to $c_j$,
the distance between $c'_j$ and $c_j$ is small:
\begin{eqnarray}
d(c'_j, c_j) \leq \frac{d(c'_j, W_j)}{|W_j|} + \frac{\ds(c_j, W_j)}{|W_j|} \leq (1+\frac{1}{\alpha}) \ds(c_j,W_j).\label{eqn:newOldLowerBound}
\end{eqnarray}

Second, since $c_j$ is the optimal center
for $C_j = W_j \cup V_j \cup M_j$,
it should save a lot of cost on $M_j$ compared to $c'_j$,
which suggests that $c_j$ and $c'_j$ would be far apart.
Formally,
\begin{eqnarray*}
\ds(c'_j,C_j) = \ds(c'_j, W_j \cup V_j \cup M_j)
\geq  \ds(c_j, C_j) = \ds(c_j, W_j \cup V_j \cup  M_j).
\end{eqnarray*}
Since $\ds(c'_j, W_j) \leq  \ds(c_j, W_j)/\alpha$ and  $\ds(c'_j, V_j) \leq  \ds(c_j, V_j)$,
we have
\begin{eqnarray}
\ds(c'_j, M_j) - \ds(c_j, M_j) & \geq & \ds(c_j, W_j) - \frac{1}{\alpha} \ds(c_j, W_j),\nonumber\\
|M_j| \ds(c'_j, c_j) & \geq & (1-\frac{1}{\alpha}) \ds(c_j, W_j)\label{eqn:newOldUpperBound}.
\end{eqnarray}

When $|C_j| >  (\frac{2}{\alpha - 1} + 3) \epsilon n + 1$, we have
$(1 - 1 / \alpha) |W_j| >  (1 + 1 / \alpha) |M_j|$.
Then Inequalities~\ref{eqn:newOldUpperBound} and \ref{eqn:newOldLowerBound}
lead to $d(c_j,c'_j) = 0$.
This means $c_j = c'_j$ which is a contradiction to the assumptions.
\end{proof}

\begin{claim}\label{lemma:distanceTranslation}
Suppose $\min_i |C_i| > (\frac{2}{\alpha - 1} + 3) \epsilon n + 1$.
If $c_i \neq c'_i$, then we have
\begin{itemize}
\item[(1)] $\dps(c'_i, A_i) \geq  \alpha \ds(c'_i, A_i \setminus \{ c(c'_i) \})$,
\item[(2)] $\dps(c_i, A_i)  \leq  \alpha \ds(c_i, A_i \setminus \{ c(c'_i)\}) + \alpha (1 + \alpha) d(c'_i, c_i)$.
\end{itemize}
\end{claim}

\begin{proof}
These translations from $d'$ to $d$ can be verified by the definition of $d'$.
In most cases, $d'(\cdot, \cdot) = \alpha d(\cdot, \cdot)$;
the only exceptions are the distances between $p$ and $c(p)$.
The detailed verification is presented below.

(1) Since $c'_i \neq c_i$, and by Claim~\ref{newneqold}, we know $c'_i \neq c_j (\forall j)$.
So we only need to check if $c(c'_i) \in A_i$. We have
\begin{eqnarray*}
\dps(c'_i, A_i) \geq \dps(c'_i, A_i \setminus \{ c(c'_i) \}) = \alpha \ds(c'_i, A_i \setminus \{ c(c'_i) \}).
\end{eqnarray*}

(2) If $c(c'_i) \not\in A_i$, then the inequality is trivial. If $c(c'_i) \in A_i$, then
\begin{eqnarray*}
\dps(c_i, A_i) & = & \dps(c_i, A_i \setminus \{c(c'_i)\} ) + d'(c_i, c(c'_i)) \leq \alpha \ds(c_i, A_i \setminus \{ c(c'_i) \}) + \alpha d(c_i, c(c'_i)).
\end{eqnarray*}
We have $d(c_i, c(c'_i)) \leq d(c_i, c'_i) + d(c'_i, c(c'_i))$.
If $c'_i$ is a selected bad point, then $d(c'_i, c(c'_i)) \leq \alpha d(c'_i, c_i)$.
Otherwise, $c(c'_i)$ is the nearest center for $c'_i$, then $d(c'_i, c(c'_i)) \leq d(c'_i, c_i)$.
In any case, the inequality for $\dps(c_i, A_i)$ follows.
\end{proof}

We are now ready to present the complete proofs of the two key claims.

\begin{claim}\label{newold}
For each $i$,
$d(c_i,c'_i) \leq 3\left(\frac{\alpha + 1}{\alpha}\right)\frac{\ds(c_i,  W_i)}{|C_i|}.$
\end{claim}

\begin{proof}
The key idea is that since $\dps(c'_i, C'_i) \leq \dps(c_i, C'_i)$ and
$C_i' \setminus W_i$ is small,  it must be the case that $\dps(c_i',
W_i)$ is not too much larger than $\dps(c_i,W_i)$.  Now, since distances
between $c_i$ and $W_i$ remain the same in 
moving from $d$ to $d'$ but distances between $c_i'$ and $W_i$ are
blown up by a factor of $\alpha$ (except for the distance between
$c_i'$ and $c_i$ itself), this means that $\alpha \ds(c_i',W_i) -
\ds(c_i,W_i)$ must be small.  This is then used together with the triangle
inequality to get an upper bound on $d(c_i,c_i')$. We provide the formal proof below.

First, if $c'_i = c_i$ the claim is trivially true so assume $c_i' \neq c_i$.
We begin with the fact that
$\dps(c'_i, C'_i) \leq \dps(c_i, C'_i) $ and then break $C_i'$ into its three components $W_i$, $V_i$, and $A_i$.
We move the $W_i$ terms to one side and move
the rest of the terms to the other side, resulting in
\begin{eqnarray}
\dps(c'_i, W_i) - \dps(c_i, W_i)
\leq
  \dps(c_i, A_i)  - \dps(c'_i, A_i)  + \dps(c_i, V_i) - \dps(c'_i, V_i) . \label{eq:claim1a}
\end{eqnarray}

Beginning with the right-hand side of (\ref{eq:claim1a}), by the triangle inequality we have 
$\ds(c_i,V_i) \leq \ds(c_i', V_i) + |V_i| d(c_i,c_i')$.  
Thus, $\dps(c_i, V_i) \leq \dps(c_i', V_i) + \alpha|V_i|d(c_i,c_i')$.
Similarly, by Claim~\ref{lemma:distanceTranslation} we have 
$\dps(c_i, A_i) \leq \dps(c_i', A_i) + \alpha|A_i|d(c_i,c_i') + \alpha(\alpha+1)d(c_i, c'_i)$.  So, the right-hand side of (\ref{eq:claim1a}) is at most $\alpha(|V_i|+|A_i| + \alpha+1)d(c_i,c_i')$.
Now, examining the left-hand side, this quantity is at least 
$\alpha \ds(c'_i, W_i \setminus \{ c(c'_i) \}) - \ds(c_i, W_i)$.  So, we have 
\begin{eqnarray}
\alpha \ds(c'_i, W_i \setminus \{ c(c'_i) \}) - \ds(c_i, W_i) & \leq &
\alpha(|V_i|+|A_i| + \alpha+1)d(c_i,c_i').  \label{eq:claim2a}
\end{eqnarray}
Using the fact that by the triangle inequality, $\alpha (|W_i|-1)d(c_i,c_i') \leq \alpha \ds(c_i',W_i \setminus \{c(c'_i)\}) + \alpha \ds(c_i, W_i \setminus \{c(c'_i)\})$, and subtracting $(\alpha+1)\ds(c_i,W_i)$ from both sides, 
we get 
\begin{eqnarray}
\alpha \ds(c_i,c_i') (|W_i|-1) - (\alpha+1)\ds(c_i,W_i)
 \leq 
\alpha \ds(c_i', W\!\setminus\!\{c(c_i')\}) - \ds(c_i,W_i).  \label{eq:claim3a}
\end{eqnarray}
Combining (\ref{eq:claim2a}) and (\ref{eq:claim3a}) we have:
\begin{eqnarray*}
\alpha d(c_i, c'_i) (|W_i| - 1) -  (\alpha + 1)\ds(c_i, W_i)
\leq \alpha d(c_i, c'_i) \left(|V_i| + |A_i| + \alpha + 1\right)
\end{eqnarray*}
which implies the desired result when $|C_i| > 5\epsilon n + 2\alpha + 6$.
\end{proof}

\begin{claim}\label{otherdirection}
For each $i$, if $c_i' \neq c_i$ then 
$d(c_i,c'_i) \geq \left(\frac{\alpha-1}{2\alpha}\right)\frac{d(c_i,W_i)}{\epsilon n + \alpha + 1}.$
\end{claim}

\begin{proof}
Assume $c_i' \neq c_i$ and let $d_i =
d(c_i,c_i')$.  \yingyu{We will begin with the fact that $\ds(c_i,C_i) \leq \ds(c_i',C_i)$, and then proceed to compare $\ds(c_i,C'_i)$ and $\ds(c_i',C'_i)$, and finally compare $\dps(c_i,C'_i)$ and $\dps(c_i',C'_i)$, which will give the desired bound.}

First, since $\ds(c_i, C_i) \leq \ds(c'_i, C_i)$ and the difference between $C'_i$ and $C_i$ is small,  $\ds(c_i,C_i')$ cannot be much larger than $\ds(c_i',C_i')$.
Specifically, by $(\alpha,\epsilon)$-perturbation resilience, $|A_i| \leq
\epsilon n$ and $|M_i|
\leq \epsilon n$. We have by the triangle  inequality
\begin{eqnarray}
\ds(c_i,A_i) & \leq & \ds(c_i',A_i) + (\epsilon n)d_i, \label{eqn:dAa} \\
\ds(c_i,M_i) & \geq & \ds(c_i',M_i) - (\epsilon n)d_i. \label{eqn:dMa}
\end{eqnarray}
So, $\ds(c_i,C_i') \leq \ds(c_i',C_i') + 2(\epsilon n) d_i$.

\yingyu{Now, we turn to compare $\dps(c_i,C'_i)$ and $\dps(c_i',C'_i)$.}
We begin with $A_i$.  By Claim~\ref{lemma:distanceTranslation} we have 
\begin{eqnarray}
\dps(c_i, A_i) \leq \dps(c_i', A_i) + \alpha\epsilon n d_i + \alpha(\alpha+1)d_i. \label{eqn:dpAa}
\end{eqnarray}
\yingyu{On $C_i \setminus M_i$, the cost of $c_i$ is smaller than that of $c_i'$.}
Specifically, from (\ref{eqn:dMa}) we have:
\begin{eqnarray*}
\ds(c_i, C_i \setminus M_i) & \leq & \ds(c_i', C_i \setminus M_i) +
(\epsilon n) d_i.
\end{eqnarray*}
so
\begin{eqnarray}
\alpha \ds(c_i, C_i \setminus M_i) & \leq & \alpha \ds(c_i', C_i \setminus M_i) +
\alpha (\epsilon n) d_i\nonumber\\
& \leq & \dps(c_i', C_i \setminus M_i) + (\alpha-1)d_i + \alpha (\epsilon n)d_i \label{eqn:dpWVa}
\end{eqnarray}
where the second step is from the following fact: $\dps(c_i', C_i \setminus M_i) = \alpha \ds(c_i', C_i \setminus M_i)$ if $c(c'_i) \neq c_i$, else $\dps(c_i', C_i \setminus M_i)  =  \alpha \ds(c_i', C_i \setminus M_i) - (\alpha-1)d_i$.
Now, the left-hand side above equals $\dps(c_i, C_i \setminus M_i) +
(\alpha-1)\ds(c_i,W_i)$ because distances between $c_i$ and $W_i$ are {\em not} blown up by a factor of $\alpha$.

Adding up (\ref{eqn:dpAa}) and (\ref{eqn:dpWVa}) 
means that we get a contradiction if $(\alpha-1)\ds(c_i,W_i) >  (2\alpha\epsilon n  + \alpha(\alpha+1) + \alpha - 1) d_i$. In other words, if our savings in using $c_i$ as center is greater than our extra cost.
Therefore, $d_i \geq \left(\frac{\alpha-1}{2\alpha}\right)\frac{\ds(c_i,W_i)}{\epsilon n + \alpha + 1}$ as desired.
\end{proof}

Combining the upper bound of  Claim~\ref{newold} with the lower bound of Claim~\ref{otherdirection} when $c_i' \neq c_i$, we get a contradiction for sufficiently large $|C_i|$ as given in the theorem statement, yielding $c_i'=c_i$.
\end{proofof}

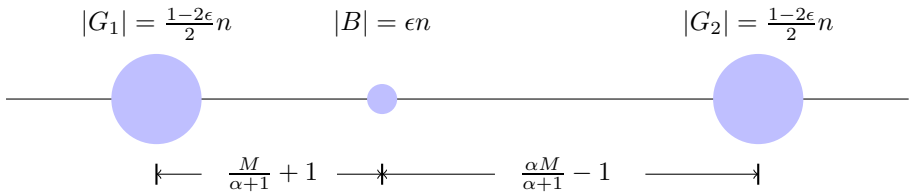
\begin{figure}[!tbhp]
\centering

\begin{tikzpicture}[scale = 2]
\usetikzlibrary{shapes,backgrounds}

\def\maxx{3}
\def\maxy{1}
\def\bbox{(-\maxx,-\maxy) rectangle (\maxx,\maxy)}
\clip \bbox;

\coordinate (G1) at (-2,0) ;
\coordinate (G2) at (2,0);
\coordinate (B) at (-0.5,0);

\draw (-\maxx,0) -- (\maxx,0);
\fill[blue!25] (G1) circle (0.3);
\fill[blue!25] (G2) circle (0.3);
\fill[blue!25] (B) circle (0.1);

\coordinate (G1n) at (-2,-.5) ;
\coordinate (G2n) at (2,-.5);
\coordinate (Bn) at (-0.5,-.5);
\draw[thick] (G1n) ++ (0,0.07)-- +(0,-0.14);
\draw[thick] (G2n) ++ (0,0.07)-- +(0,-0.14);
\draw[thick] (Bn) ++ (0,0.07)-- +(0,-0.14);

\draw[<-] (G1n) -- ++(0.3,0); 
\node at (-1.6,-.5) [right] {$\frac{M}{\alpha+1} + 1$}; 
\draw[<-] (Bn) -- ++(-0.3,0);

\draw[<-] (Bn) -- ++(0.75,0); 
\node at (0.35,-.5) [right] {$\frac{\alpha M}{\alpha+1} - 1$}; 
\draw[<-] (G2n) -- ++(-0.75,0);

\node  at (-2,.5) {$|G_1| = \frac{1-2\epsilon}{2} n$};
\node  at (2,.5) {$|G_2| = \frac{1-2\epsilon}{2} n$};
\node at (-0.5,.5) {$|B|=\epsilon n$};
\end{tikzpicture}
\caption{An example showing the optimality of the bound on the number of bad points.}\label{fig:badPointOptimal}
\end{figure}

{\note}
The bound in Theorem~\ref{badPointsBoundTheorem} is optimal in the sense that for any $\alpha > 1$ and $0< \epsilon < 1/5$,
we can easily construct an $(\alpha, \epsilon)$-perturbation resilient $2$-median instance which has $\epsilon n$ bad points.

The instance is shown in Figure~\ref{fig:badPointOptimal}.
It has $3$ groups of points: $G_1, G_2$, and $B$.
Both $G_1$ and $G_2$ have $(1-\epsilon)n/2$ points, and $B$ has $\epsilon n$ points. Let $M$ be a sufficiently large constant, say, $M > n^2/\epsilon$.
The distances within the same group are $1$, while those between the points in $G_1$ and $G_2$ are $M$,
those between the points in $B$ and $G_1$ are $\frac{M}{\alpha + 1}+1$, and those between the points in $B$ and $G_2$ are $\frac{\alpha M}{\alpha + 1}-1$.
The instance satisfies the triangle inequality, which can be verified by a case analysis.
The optimal clustering before perturbation has one center in $G_1$ and the other in $G_2$.
Then $B$ are trivially bad points, and thus we have $\epsilon n$ bad points in this instance.

Now we show that the instance is $(\alpha, \epsilon)$-perturbation resilient.
To prove that the optimal clustering after perturbation $\mathcal{C}'$ is $\epsilon$-close to the original optimal clustering,
it suffices to show that $\mathcal{C}'$ has one center from $G_1 \cup B$ and the other center from $G_2$.
Assume for contradiction that this is not true.
If both centers come from $G_2$, the cost of points in $G_1$ is $\frac{(1-\epsilon)n}{2}M$.
On the other hand, the optimal cost before perturbation is $(1-\epsilon)n - 2 + \epsilon n ( \frac{M}{\alpha + 1} + 1)$,
so the optimal cost after perturbation is no more than $\alpha ((1-\epsilon)n - 2 + \epsilon n ( \frac{M}{\alpha + 1} + 1))$.
But this is smaller than $\frac{(1-\epsilon)n}{2}M$, which is a contradiction.
Similarly, we get a contradiction if both centers come from $G_1 \cup B$.


%

\subsection{Approximation Bound}\label{subsec:app}

Now, we consider \yingyu{the problem of approximating} the cost of the optimum clustering. We can
see that after removing the bad points, the optimal clusters are far apart from each other. In order to get rid of
the influence of the bad points, we generate a list of blobs, which form a partition of the data points, and each
of which contains only good points from one optimal cluster.
Then we construct a tree on the list of blobs with a pruning that assigns all good points correctly.
We will show that this pruning has low cost, so the lowest cost pruning of the tree is a good approximation.
The details are described in Algorithm~\ref{alg:aekmedian_blobapprox}.

A key step is to generate the list of almost ``pure'' blobs, which is described in Algorithm~\ref{alg:aekmedian_blobs}.
Suppose for any $i$ and any good point $p \in G_i$, its $\gamma |G_i|$ nearest neighbors contain no good points outside $C_i$.
Also suppose the algorithm knows the value of $\gamma$.
Informally, the algorithm maintains a threshold $t$.
At each threshold, for each point $p$ that has not been added to the list,
the algorithm checks its $\gamma t$ nearest neighbors $N_{\gamma t}(p)$.
It constructs a graph $F_t$ by connecting any two points that have sufficiently many common neighbors.
It then builds another graph $H_t$ by connecting any two points that have sufficiently many common neighbors in $F_t$,
and adds sufficiently large components in $H_t$ to the list.
Finally, for each remaining point $p$, it checks if most of $p$'s neighbors are in the list and
if there are blobs containing a significant amount of $p$'s neighbors.
If so, it inserts $p$ into such a blob with the smallest median distance.
Then the threshold is increased and the above steps are repeated.

The intuition behind Algorithm~\ref{alg:aekmedian_blobs} is as follows.
As mentioned above, the algorithm works when for any $i$ and any good point $p \in G_i$,
the $\gamma |G_i|$ nearest neighbors of $p$ contain no good points outside $C_i$ ($\gamma = 1$ for the $k$-median instances considered in this section, as shown in Lemma~\ref{lem:strictSep}; $\gamma=\frac{4}{5}$ for the min-sum instances considered in Section~\ref{sec:alpha_espilon_minsum}, as shown in Claim~\ref{cla:nn}).
Without loss of generality, assume $|C_1| \leq |C_2| \leq \dots \leq |C_k|$.
When $t \leq |C_1|$, good points in different clusters do not have most neighbors in common
and thus are not connected in $F_t$. However, they may be connected by a path of bad points.
So we further build the graph $H_t$ to disconnect such paths,
which ensures that the blobs added into the list contain only good points from one optimal cluster.
The final insert step (Step~\ref{step:insert}) makes sure that when $t=|C_1|$, all remaining good points in $C_1$
will be added to the list and will not affect the construction of blobs from other optimal clusters.
We can show by induction that, at the end of the iteration $t=|C_i|$, all good points in $C_j (j\leq i)$ are added to the list.
When $t$ is large enough, any remaining bad points are inserted into the list, so the points are partitioned into a list of almost pure blobs.
The formal guarantee for Algorithm~\ref{alg:aekmedian_blobs} is stated in Lemma~\ref{lem:bloblist}.

\newcommand{\score}{\mathrm{score}}
\newcommand{\myrank}{\mathrm{rank}}
\newcommand{\median}{\mathrm{median}}

Another key step is to construct a tree on these blobs.
Since good points are closer to good points in the same optimal cluster than to those in other clusters (Lemma~\ref{lem:strictSep}),
there exist algorithms that can build a tree with a pruning that assigns all good points correctly.
In particular, we can use the robust linkage procedure in~\cite{nina_colt_2010},
which repeatedly merges the two blobs $C, C'$ with the maximum $\score(C,C')$ defined as follows.
For each $p \in C$, sort the other blobs in decreasing order of the median distance between $p$ and points in the blob,
and let $\myrank(p,C')$ denote the rank of $C'$.
Then define $\myrank(C,C') = \median_{x\in C} [\myrank(x,C')]$ and $\score(C,C') = \min[\myrank(C,C'), \myrank(C',C)]$.
Intuitively, for any blobs $A,A'$ from the same optimal cluster and $D$ from a different cluster,
good points in $A$ always rank $A'$ later than $D$ in the sorted list,
so $\myrank(A,A') > \myrank(A,D)$. Similarly, $\myrank(A',A) > \myrank(A',D)$, and thus $\score(A',A) > \score(A,D)$.
This means the algorithm will always merge blobs from the same cluster before merging them with blobs outside,
so there is a pruning that assigns all good points correctly.

\newcommand{\blobs}{\mathcal{L}}

\begin{algorithm}[!t]
\caption{$k$-median, $(\alpha,\epsilon)$ perturbation resilience}
\label{alg:aekmedian_blobapprox}
\begin{algorithmic}[1]
\REQUIRE{Data set $\data$, distance function $d(\cdot, \cdot)$ on $\data$, $\min_i |C_i|$, $\epsilon>0$}
\STATE Run Algorithm~\ref{alg:aekmedian_blobs} to generate a list $\blobs$ of blobs with parameters $u_B=\epsilon n, \gamma = 1$.
\STATE Run the robust linkage procedure in~\cite{nina_colt_2010} to get a cluster tree $\mathcal{T}$.
\STATE Run dynamic programming on $\mathcal{T}$ to get the minimum cost pruning ${\cal \tilde{C}}$ and its centers $\tilde{\mathbf{c}}$.
\ENSURE{Clustering ${\cal \tilde{C}}$ and its centers $\tilde{\mathbf{c}}$.}
\end{algorithmic}
\end{algorithm}

\begin{algorithm}[!t]
\caption{Generating interesting blobs}
\label{alg:aekmedian_blobs}
\begin{algorithmic}[1]
\REQUIRE{Data set $\data$, distance function $d(\cdot, \cdot)$ on $\data$, the size of the smallest optimal cluster $\min_i |C_i|$, the upper bound on
the number of bad points $u_B$, a parameter $\gamma\in [4/5,1]$}
\STATE Let $N_r(p)$ denote the $r$ nearest neighbors of $p$ in $\data$.
\STATE Let $\blobs=\emptyset, A_\data= \data$. Let the initial threshold $t = \min_i |C_i|$.
\STATE Construct a graph $F_t$ by connecting $p,q \in A_\data$ if\\ $|N_{\gamma t}(p) \mycap N_{\gamma t}(q)| > (2\gamma-1)t - 2 u_B$. \label{step:Ft}
\STATE Construct a graph $H_t$ by connecting points $p,q \in A_\data$ if $p, q$ share more than $u_B$ neighbors in $F_t$.
\STATE Add to $\blobs$ all the components $C$ of $H_t$ with $|C| \geq \frac{1}{2}\min_i |C_i|$ and remove them from $A_\data$. \label{step:newblob}
\STATE For each point $p \in A_\data$, check if most of $N_{\gamma t}(p)$ are in $\blobs$ and if there exists $C \in \blobs$ containing a significant number of points in $N_{\gamma t}(p)$. More precisely, check if\\
(1) $|N_{\gamma t}(p) \setminus \blobs| \leq \frac{1}{2}\min_i |C_i| + 2 u_B$;\\
(2) $\blobs_p \neq \emptyset$ where $\blobs_p = \{C \in \blobs: | C \mycap N_{\gamma t}(p)|\geq (\gamma-\frac{3}{5})|C|\}$. \\
If so, assign $p$ to the blob in $\blobs_p$ of smallest median distance,  remove $p$ from $A_\data$. \label{step:insert}
\STATE While $|A_\data| > 0$, increase $t$ by $1$ and go to Step~\ref{step:Ft}.
\ENSURE{The list $\blobs$.}
\end{algorithmic}
\end{algorithm}

In the following, we prove that Algorithm~\ref{alg:aekmedian_blobapprox} outputs a good approximation.
We first prove a key property of the good points in $(\alpha,\epsilon)$-perturbation resilience instances in Lemma~\ref{lem:strictSep} and show in Lemma~\ref{lem:bloblist} that the property ensures the success of Algorithm~\ref{alg:aekmedian_blobs}, and then prove a property of the bad points in Lemma~\ref{lem:badpoints}. Finally, we use these lemmas to prove the approximation bound in Theorem~\ref{thm:aekmedian}. 

\begin{lemma}[Theorem 8 in~\cite{cReyzin12}, Lemma 2.6 in~\cite{ABS10}]\label{lem:strictSep}
When $\alpha > 2+\sqrt{3}$, for any good points $p_1, p_2 \in G_i, q \in G_j (j\neq i)$, we have $d(p_1,p_2) < d(p_1,q)$.
Consequently, for any good point $p \in G_i$, all its $|G_i|$ nearest neighbors belong to $C_i \cup B$.
\end{lemma}
\begin{proof}
The following proof is implicit in~\cite{ABS10} and we include it for completeness. We rephrase it slightly so that it is more intuitive.
By the triangle inequality and the definition of good points, 
$$
	d(p_2, p_1) + d(p_1, q) + d(q, c_j) \geq d(p_2, c_j) > \alpha d(p_2, c_i) \geq \alpha (d(p_1,p_2) - d(p_1, c_i)).
$$
Rearranging terms leads to
\begin{eqnarray}
  d(p_1, q) + d(q, c_j) + \alpha d(p_1,c_i) > (\alpha - 1) d(p_1,p_2). \label{eqn:strictSep}
\end{eqnarray}
Now, to compare $d(p_1, q)$ and $d(p_1, p_2)$, we need to get rid of the extra terms $d(q,c_j)$ and $d(p_1, c_i)$. By the same proof in Lemma~\ref{kmedian-lemma-between}(2), 
$$d(p_1,q) > (\alpha -1) d(p_1, c_i),~~\textrm{and}~~d(p_1,q) > (\alpha -1) d(q, c_j).$$ 
Plugging these into (\ref{eqn:strictSep}), we have
\begin{eqnarray*}
   \left( 1 + \frac{1}{\alpha-1}+ \frac{\alpha }{\alpha-1}\right) d(p_1, q) > (\alpha - 1) d(p_1,p_2). 
\end{eqnarray*}
So when $\alpha > 2 + \sqrt{3}$, $d(p_1, q) > d(p_1, p_2)$.
\end{proof}

\begin{lemma}\label{lem:bloblist}
Suppose the number of bad points is bounded by $u_B$,
and for any $i$ and any good point $p \in G_i$, all its $\gamma |G_i|$ nearest neighbors in $S$ are from $C_i \cup B$.
If $\min_i |C_i| > 30 u_B$, then Algorithm~\ref{alg:aekmedian_blobs} generates a list $\blobs$ of
blobs each of size at least $\frac{1}{2}\min_i |C_i|$ such that:
\begin{itemize}
\item[(1)] The blobs in $\blobs$ form a partition of $S$.
\item[(2)] Each blob in $\blobs$ contains good points from only one optimal cluster.
\end{itemize}
\end{lemma}

\begin{proof}
Without loss of generality, assume $|C_1| \leq |C_2| \leq \dots \leq |C_k|$.
We prove the following two claims by induction on $i \leq k$:
\begin{itemize}
\item[(1)] For any $t\leq |G_i|$, any blob in the list $\blobs$ only contains good points from only one optimal cluster; all blobs
have size at least $\frac{1}{2}\min_i |C_i|$.
\item[(2)] At the beginning of the iteration $t = |G_i|+1$, any good point $p \in G_j, j\leq i$ has already been assigned to a blob in
the list that contains good points only from $C_j$.
\end{itemize}

The first two claims imply that each blob in the list contains good points from only one optimal cluster. Moreover,
at the beginning of the iteration $t =|G_k|+1$, all good points have been assigned to one of the blobs in $\blobs$, so there
are only bad points left, the number of which is smaller than $\frac{1}{2}\min_i |C_i|$.
These remaining points will eventually be assigned to the blobs before $\gamma t>n$, so the blobs form a partition of $S$.

The claims are clearly both true initially. We show now that as long as $t \leq |G_1|$, the graphs $F_t$ and $H_t$ have the following properties.

\begin{figure}[!tbhp]
\centering
\begin{tikzpicture}[scale=0.4]

\fill[blue!30]  (0,-.5) ellipse (2 and 2.5);
\fill[red!30]  (-2.2,-2.6) rectangle (2.2,-3.8);

\begin{scope}[shift={(6cm,0cm)}]
\fill[teal!30]  (0,-.5) ellipse (2 and 2.5);
\fill[red!30]  (-2.2,-2.6) rectangle (2.2,-3.8);
\end{scope}

\begin{scope}[shift={(12cm,0cm)}]
\fill[olive!30]  (0,-.5) ellipse (2 and 2.5);
\fill[red!30]  (-2.2,-2.6) rectangle (2.2,-3.8);
\end{scope}

\node (v1) at (2,-3) {};
\node (v3) at (2,-3.4) {};
\node (v2) at (4,-3) {};
\node (v4) at (4,-3.4) {};

\begin{scope}[shift={(0cm,0cm)}]
\draw[thick,red!30]  (2,-3) edge (4,-3);
\draw[thick,red!30]  (2,-3.4) edge (4,-3.4);
\end{scope}
\begin{scope}[shift={(6cm,0cm)}]
\draw[thick,red!30]  (2,-3) edge (4,-3);
\draw[thick,red!30]  (2,-3.4) edge (4,-3.4);
\end{scope}

\node at (-3.6,-3.2) {\textcolor{red}{$B$}};
\node at (0,2.6) {\textcolor{blue}{$G_1$}};
\node at (6,2.6) {\textcolor{teal}{$G_2$}};
\node at (12,2.6) {\textcolor{olive}{$G_3$}};
\end{tikzpicture}
\caption{A high level illustration of the graph $F_t$.}\label{fig:Ft}
\end{figure}
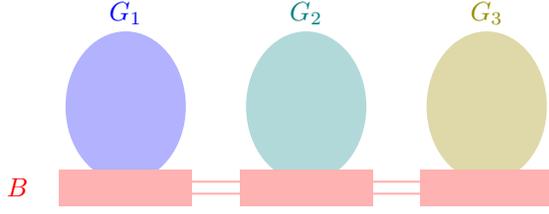

\begin{itemize}
\item No good point $p_i$ in cluster $C_i$ is connected in $F_t$ to a good point $p_j$ in a different cluster $C_j$.
By assumption, $p_i$ has no neighbors outside $C_i \cup B$
and $p_j$ has no neighbors outside $C_j \cup B$, so they share at most $u_B < (2\gamma-1)t - 2 u_B$ neighbors.
\item No point $q$ is connected in $F_t$ to both a good point $p_i$ in $C_i$ and a good point $p_j$ in a different cluster $C_j$.
If $q$ is connected to $p_i$, then $|N_{\gamma t}(p_i) \cap N_{\gamma t}(q)| > (2\gamma-1)t - 2 u_B$.
Since $p_i$ has no neighbors outside $C_i \cup B$, $N_{\gamma t}(q)$ contains more than $(2\gamma-1)t - 3u_B \geq \gamma t/2$ points from $G_i$.
Similarly, if $q$ is connected to $p_j$, then $N_{\gamma t}(q)$ contains more than $\gamma t/2$ points from $G_j$, which is contradictory.
Thus, the graph $F_t$ looks like the illustration in Figure~\ref{fig:Ft}.

\item All the components in $H_t$ of size at least $\frac{1}{2}\min_{i} |C_i|$ will only contain good points from one optimal cluster.
As there are at most $u_B$ bad points, any two points connected in $H_t$ must be connected in $F_t$ to at least one good point.
Then by the above two properties, points on a path in $H_t$ must be connected in $F_t$ to good points in the same cluster,
so there is no path connecting good points from different clusters.
\end{itemize}

We can use the three properties to argue the first claim: as long as $t \leq |G_1|$, each blob in $\blobs$ contains good
points from at most one optimal cluster. This is true at the beginning and by the third property, for any
$t \leq |G_1|$, anytime we insert a whole new blob in the list in Step~\ref{step:newblob}, that blob must contain point from at most one
optimal cluster. We now argue that this property is never violated as we assign points to blobs already in the list
in Step~\ref{step:insert}. Suppose a good point $p \in C_i$ is inserted into $C \in \blobs$.
Then $C \in \blobs_p$, which means $|N_{\gamma t}(p) \cap C| \geq |C|/5 > u_B$.
So $N_{\gamma t}(p) \cap C$ contains at least one good point,
which must be from $C_i$ since $N_{\gamma t}(p)$ contains no good points outside $C_i$.
Then by induction $C$ must contain only good points from $C_i$,
and thus adding $p$ to $C$ does not violate the first claim.

We now show the second claim: after the iteration $t = |G_1|$, all the good points in $C_1$ have already been
assigned to a blob in the list that only contains good points from $C_1$. There are two cases. First, if at the
beginning of the iteration $t = |G_1|$, there are still at least $\frac{1}{2}\min_i |C_i|$ points from the good point
set $G_1$ that do not belong to blobs in the list.
Any such good point has all $\gamma |G_1|$ neighbors in $C_1\cup B$.
Then any two such good points share at least $2\gamma |G_1| - |C_1 \cup B| \geq (2\gamma -1)|G_1| - |B| \geq  (2\gamma -1)t - 2u_B$ neighbors.
So they will connect to each other in $F_t$ and then in $H_t$, and thus we will add one
blob to $\blobs$ containing all these points.
Second, it could be that at the beginning of the iteration $t = |G_1|$, all
but less than $\frac{1}{2}\min_i |C_i|$ good points in $G_1$ have been assigned to a blob in the list.
Denote the points that have not yet been assigned as $E$. Any point $p \in E$ has no neighbors outside $C_1 \cup B$.
Then $|N_{\gamma t}(p) \setminus \blobs| \leq |E| + |B| \leq \frac{1}{2}\min_i |C_i|+ 2u_B$.
Also, there exists a blob $C$ containing good points from $C_1$ such that $C \in \blobs_p$.
Otherwise, $N_{\gamma t}(p)$ contains at most $(\gamma-\frac{3}{5}) (|C_1 \cup B|) < \gamma |C_1| - \frac{1}{2}|C_1| - 2u_B$ points in $C_1 \cap \blobs$,
while it contains at most $|E|$ good points in $C_1 \setminus \blobs$ and contains no points outside $C_1 \cup B$.
In total, $N_{\gamma t}(p)$ has less than $\gamma t$ points, which is contradictory.
So $\blobs_p \neq \emptyset$ and $p$ will be added to the list in Step~\ref{step:insert}.

We then iterate the argument on the remaining set $A_S$. The key point is that for $t \geq |G_i|, i > 1$, we have
that all the good points in $C_1, C_2, \dots, C_i$ have already been assigned to blobs in $\blobs$.
\end{proof}

Lemma~\ref{lem:strictSep} and~\ref{lem:bloblist} show that Algorithm~\ref{alg:aekmedian_blobs} with parameters $u_B=\epsilon n$
and $\gamma=1$ produces a list of sufficiently large, almost pure blobs. Then the robust linkage procedure in~\cite{nina_colt_2010}
can build a tree on these blobs with a pruning that assigns all good points correctly.
Now it suffices to show that this pruning is a good approximation,
for which we need to bound the cost increased by the bad points assigned incorrectly.
The following property of these bad points turns out to be useful.
Intuitively, Algorithm~\ref{alg:aekmedian_blobs} is designed such that
whenever a bad point is added to a blob containing good points from a different cluster,
it must be closer to a significant number of points in that cluster than to a significant number of points
in its own cluster. Then the cost increased by incorrectly assigning each such bad point is small, resulting in a good approximation.

\begin{lemma}\label{lem:badpoints}
Suppose for any good point $p \in G_i$, all its $|G_i|$ nearest neighbors in $S$ are from $C_i \cup B$,
and $\min_i |C_i| > 30 u_B$. When running Algorithm~\ref{alg:aekmedian_blobs} with $\gamma=1$, if a bad
point $q \in B_i$ is assigned to a blob $C$ containing good points from a different optimal clustering $C_j$,
then there exist $m=\frac{1}{5}\min_i |C_i|$ points $Z_i$ from $C_i$, and $m$ points $Z_j$ from $C_j$,
such that $d(q, Z_i) \geq d(q,Z_j)$.
\end{lemma}

\begin{proof}
There are two cases: $q$ is added into $C$ in (1) Step~\ref{step:newblob} or (2) Step~\ref{step:insert}.

\paragraph{Case 1}
There must be a path in $H_t$ connecting $q$ to a good point in $C_j$ at threshold $t$.
For any edge $(x,y)$ in $H_t$, since $x,y$ share at least $\epsilon n$ neighbors in $F_t$ and
there are at most $\epsilon n$ bad points, they share at least one good point as neighbor in $F_t$.
As shown in the proof of Lemma~\ref{lem:bloblist}, no point can connect to good points from different clusters,
so in $F_t$ all points on the path must connect to good points in $C_j$.
In particular, $q$ is connected in $F_t$ to a good point $p \in G_j$.
Then $|N_{t}(p) \cap N_{t}(q)| > t - 2 u_B$.
Since $p$ is still in $A_S$, $t \leq |G_j|$, and thus $N_{t}(p)$ contains no points outside $C_j \cup B$.
This means that at least $t - 3 u_B \geq m$ points in $N_{t}(q)$ are good points in $C_j$,
then we can select $m$ points $Z_j$ from $N_{t}(q) \cap G_j$.
We also have that at most $2 u_B$ points in $N_{t}(q)$ are points in $C_i$,
so we can select $m$ points $Z_i$ from $C_i \setminus N_{t}(q)$.

\paragraph{Case 2}
There are three subcases when $q$ is inserted into $C$ at threshold $t$.
\begin{itemize}
\item[(1)] There is no good points from $C_i$ in the list.
Since $|N_t(q) \setminus \blobs| \leq \frac{1}{2} \min_i |C_i| + 2u_B$,
$N_t(q)$ contains at most this number of good points in $C_i$.
This means at least $\frac{1}{2} \min_i |C_i| - 2 u_B > m$ good points in $C_i$ are outside $N_t(q)$,
from which we can select $Z_i$.
On the other hand, we can select $Z_j$ as follows.
When inserting $q$ into $C$, we have $|N_t(q) \cap C| \geq \frac{2}{5}|C| \geq m + u_B$.
Since $C$ contains only good points from $C_j$ and some bad points, $N_t(q) \cap C$ contains at least $m$
good points in $C_j$, from which we can select $Z_j$.
Since $Z_j$ are from $N_t(q)$ and $Z_i$ are outside $N_t(q)$, we have $d(q,Z_i) \geq d(q,Z_j)$.
\item[(2)] There exists $C' \in \blobs$ containing good points from $C_i$, but $C' \not\in \blobs_p$.
This means $|B(q,t) \cap C'| \leq \frac{2}{5}|C'|$, so there are at least $\frac{3}{5}|C'| \geq m + u_B$ points in $C'$ are outside $N_t(q)$.
At least $m$ of these points are good points from $C_i$, since $C'$ contains only good points from $C_i$ and at most $u_B$ bad points. So, we can select $Z_i$ from them.
On the other hand, we can select $Z_j$ as in the first subcase.
\item[(3)] There exists $C' \in \blobs_p$ containing good points from $C_i$.
Since $q$ is assigned to $C$ rather than $C'$ according to median distances,
we know that at least half of the points $Z'_j$ from $C$ are closer to $q$ than at least half of the points $Z'_i$ from $C'$.
Since there are at most $u_B$ bad points, we can select $m$ good points $Z_j$ from $Z'_j$ and select $m$ good points $Z_i$ from $Z'_i$.
Note that $Z_j$ are all from $G_j$ and $Z_i$ are all from $G_i$, so $d(q,Z_i) \geq d(q,Z_j)$.
\end{itemize}
Therefore, the statement is true in all cases.
\end{proof}

\begin{theorem}\label{thm:aekmedian}
If the clustering instance is $(\alpha,\epsilon)$-perturbation resilient for $\alpha > 2+\sqrt{3}$ and $\epsilon \leq \rho/30$ where $\rho =\frac{\min_i |C_i|}{n}$,
then Algorithm~\ref{alg:aekmedian_blobapprox} produces a clustering which is $(1+\frac{5\epsilon}{\rho})$-approximation to the optimal clustering with
respect to the $k$-median objective in polynomial time.
\end{theorem}

\begin{proof}
By Lemma~\ref{lem:strictSep} and~\ref{lem:bloblist},
Algorithm~\ref{alg:aekmedian_blobs} partitions the points into a list of blobs, each of which has size at least $\frac{1}{2}\min_i|C_i|$
and contains only good points from one optimal cluster.
Let $B'_i$ denote the bad points that are assigned to blobs containing good points in $C_i$.
By Lemma~\ref{lem:strictSep}, Theorem 9 in~\cite{nina_colt_2010} can be applied to $\blobs$,
by which we know that $\{(C_i \cap G) \cup B'_i\}$ is a pruning of the tree.
Suppose the cost of the optimum is $\OPT$.
We now show that this pruning, using the original centers $\{c_i\}$, is a $(1+\frac{5\epsilon}{\rho})$-approximation to $\OPT$.

Suppose a bad point $q \in C_i$ is assigned to a blob $C$ containing good points from a different optimal cluster $C_j$.
By Lemma~\ref{lem:badpoints}, there exist $m=\frac{1}{5}\min_i |C_i|$ points $Z_i$ from $C_i$, and $m$ points $Z_j$ from $C_j$,
such that $d(q, Z_i) \geq d(q,Z_j)$.
Then the increase in cost due to $q$ is bounded as follows:
\begin{eqnarray*}
d(q,c_j) - d(q,c_i) & \leq & \frac{d(q, Z_j) + d(c_j, Z_j)}{m} - \frac{d(q, Z_i) - d(c_i, Z_i)}{m} \\
& \leq & \frac{1}{m}[ d(c_j, Z_j) + d(c_i, Z_i)] \leq \frac{\OPT}{m}.
\end{eqnarray*}
As there are at most $\epsilon n$ bad points and $m = \frac{\min_i |C_i|}{5}$, the increase of cost is at most $\frac{\epsilon n}{m} \OPT = \frac{5\epsilon}{\rho} \OPT$.

\paragraph{Running Time} In Algorithm~\ref{alg:aekmedian_blobs}, for each $p \in S$, we first
sort all the other points in ascending order of distances in time $O(n^2 \log n)$.
At each threshold $t$, think of a directed $t$-regular graph $E_t$, where, for each point $q$ in the $t$ nearest neighbors
of a point $p$, there is a directed edge from $p$ to $q$ in $E_t$.
Let $A_{E}$ denote the adjacency matrix for $E_t$, and let $N= A_{E} A^\top_{E}$.
Then $N_{pq}$ is the number of common neighbors between $p$ and $q$, which can be used in constructing $F_t$.
Computing $N$ takes time $O(n^\omega)$, \yingyu{where $\omega$ is the matrix multiplication exponent.}
The same method can be used to compute the number of common neighbors in $F_t$ and construct $H_t$.
Since there are $O(n)$ thresholds, the total time for constructing $F_t$ and $H_t$ is $O(n^{\omega+1})$.
For the other steps, adding a blob takes time $O(n^2)$ and inserting a point takes time $O(n^2)$.
These steps can be performed at most $O(n)$ times, so they take $O(n^3)$ time.
In total, Algorithm~\ref{alg:aekmedian_blobs} takes time $O(n^{\omega+1})$.
Since the robust linkage algorithm~\cite{nina_colt_2010} takes time at most $O(n^{\omega+1})$,
and the dynamic programming takes time $O(n^3)$ (Appendix~\ref{subsec:dp}),
the running time of Algorithm~\ref{alg:aekmedian_blobapprox} is $O(n^{\omega+1})$.
\end{proof}


\subsection{Sublinear Time Algorithm for the $k$-Median Objective}\label{sec:inductive}

Consider a clustering instance $(X, d)$ that is $(\alpha,\epsilon)$-perturbation resilient to $k$-median.
For simplicity, suppose the distances are normalized such that $\max_{p,q} d(p,q) = 1$.
Let $N = |X|$. Let $\rho = \min_i |C_i| / N$ denote the fraction of the points in the smallest cluster,
$\zeta = \Phi_X(\mathbf{c})/N$ denote the average cost of the points in the optimum clustering.

\begin{algorithm}[!tbhp]
\caption{$k$-median, $(\alpha,\epsilon)$ perturbation resilience, sublinear}\label{alg:kmedian_Inductive}
\begin{algorithmic}[1]
\REQUIRE{Data set $X$, distance function $d(\cdot, \cdot)$ on $X$, $\min_i |C_i|$.}
\STATE{Draw a sample $\data$ of size $n = \Theta(\frac{k}{\epsilon^2 \zeta}\ln\frac{N}{\delta})$ i.i.d.\ from $X$.}
\STATE{Run Algorithm~\ref{alg:aekmedian_blobapprox}  on $\data$ to obtain ${\cal \tilde{C}}$ and $\tilde{\mathbf{c}}$.}
\ENSURE{The implicit clustering obtained by assigning each point in $X$ to its nearest neighbor in $\tilde{\mathbf{c}}$.}
\end{algorithmic}
\end{algorithm}

\begin{theorem}\label{inductive-median-theorem}
Suppose $(X,d)$ is $(\alpha, \epsilon)$-perturbation resilient for
$\alpha > 2+\sqrt{3}$, $\epsilon < \rho /100$.
Then with probability $\geq 1-\delta$, Algorithm~\ref{alg:kmedian_Inductive} outputs an implicit clustering
that is $2(1 + \frac{16\epsilon}{\rho} )$-approximation in time
$poly(\log\frac{N}{\delta}, k, \frac{1}{\epsilon}, \frac{1}{\zeta})$.
\end{theorem}

\begin{proof}
It suffices to show that $\Phi_S(\tilde{\mathbf{c}})$ is close to $\Phi_S(\mathbf{c})$ where $\mathbf{c}$ are the optimal centers for $X$. Note that Algorithm~\ref{alg:aekmedian_blobapprox} builds a tree with a pruning $\mathcal{P'}$ that assigns all good points correctly.
The key is to use the cost of this pruning as a bridge for $\Phi_S(\tilde{\mathbf{c}})$ and $\Phi_S(\mathbf{c})$: on one hand,
$\Phi_S(\tilde{\mathbf{c}})$ is no more than the cost of $\mathcal{P'}$ since $\tilde{\mathbf{c}}$ is the centers in the minimum cost pruning; on the other hand,
the cost of $\mathcal{P'}$ is roughly bounded by twice $\Phi_S(\mathbf{c})$ by the triangle inequality.

Formally, recall the following notations. If we partition $A$ into $\mathcal{P}$, the cost using centers $\mathbf{p}$ is denoted as $\Phi_A(\mathcal{P},\mathbf{p})$.
If we partition $A$ by assigning points to nearest centers in $\mathbf{p}$, the cost is denoted as $\Phi_A(\mathbf{p})$.
We will show that the cost of the implicit clustering $\Phi_X(\tilde{\mathbf{c}})$ approximates the optimum $\Phi_X(\mathbf{c})$.

First, we will prove that when $n$ is sufficiently large, with high probability,
$\Phi_X(\tilde{\mathbf{c}}) / N \approx \Phi_S(\tilde{\mathbf{c}}) / n$
and $\Phi_X(\mathbf{c}) / N \approx \Phi_S(\mathbf{c}) / n $.
For every set of centers $\mathbf{p}$, if $n = \Theta(\frac{k }{ \upsilon^2\zeta}\log \frac{N}{\delta})$ where $0 < \upsilon <1$, then by the Chernoff bound,
\begin{eqnarray*}
\mathrm{Pr}\biggl [\biggl|\frac{\Phi_S(\mathbf{p})}{n} - \frac{\Phi_X(\mathbf{p} )}{N} \biggr | > \upsilon \frac{\Phi_X(\mathbf{p} )}{N}\biggr]
\leq 2 \exp\left\{-\frac{\upsilon^2}{3} \frac{\Phi_X(\mathbf{p} )}{N} n\right\}
\leq 2 \exp\left\{-\frac{\upsilon^2}{3} \zeta n\right\}
 \leq   \frac{\delta}{4N^k}.
\end{eqnarray*}
By the union bound, we have with probability at least $1-\delta/4$,
$$
(1 - \upsilon) \frac{\Phi_X(\tilde{\mathbf{c}})}{N} \leq \frac{\Phi_S(\tilde{\mathbf{c}}) }{ n},~\textrm{and}~ \frac{\Phi_S(\mathbf{c}) }{n} \leq (1+\upsilon) \frac{\Phi_X(\mathbf{c}) }{N}.
$$
We can choose $\upsilon=\epsilon/20$,
then it is sufficient to show $ \Phi_S(\tilde{\mathbf{c}}) \leq  2(1 + {12\epsilon}/{ \rho})\Phi_S(\mathbf{c})$.

Next, since ${\cal \tilde{C}}$ may be different from $\mathcal{C} \cap S$,
we need to find a bridge for comparing $ \Phi_S(\tilde{\mathbf{c}})$ and $\Phi_S(\mathbf{c})$.
Now, we turn to analyze Algorithm~\ref{alg:aekmedian_blobapprox} on $S$ to find such a bridge.
First, we know that $X$ has at most $\epsilon N$ bad points.
Since $n$ is sufficiently large,
with probability at least $1-\delta/4$,
$S$ has at most $2\epsilon n$ bad points.
Similarly, with probability at least $1-\delta/4$, for any $i,|C_i \cap S| > 60\epsilon n$.
These ensure that Algorithm~\ref{alg:aekmedian_blobapprox} can successfully produce a tree with a pruning
$\mathcal{P'}$ that assigns all good points in $S$ correctly, as shown in Theorem~\ref{thm:aekmedian}.
Suppose in $S$, $\mathbf{c'}$ are the optimal centers for $\mathcal{P'}$.
Then we can use $\Phi_S(\mathcal{P'}, \mathbf{c'})$ as a bridge for comparing $ \Phi_S(\tilde{\mathbf{c}})$ and $\Phi_S(\mathbf{c})$.

On one hand, $\Phi_S(\tilde{\mathbf{c}}) \leq \Phi_S({\cal \tilde{C}}, \tilde{\mathbf{c}}) \leq \Phi_S(\mathcal{P'}, \mathbf{c'})$.
The first inequality comes from the fact that in $\Phi_S(\tilde{\mathbf{c}})$ each point is assigned to its nearest center
and the second comes from that ${\cal \tilde{C}}$ is the minimum cost pruning.

On the other hand, $\Phi_S(\mathcal{P'}, \mathbf{c'}) \leq 2\Phi_S(\mathcal{P'}, \mathbf{c}) \leq 2(1 + {12\epsilon}/{ \rho})\Phi_S(\mathbf{c})$.
The second inequality comes from an argument similar to that in Theorem~\ref{thm:aekmedian}
and the fact that $\Phi_S(\mathcal{P'}, \mathbf{c})$ is different from $\Phi_S(\mathbf{c})$ only on the bad points.
The first inequality comes from the triangle inequality. More precisely,
for any $N'_i \in \mathcal{P'}$,
\begin{eqnarray*}
2 |N'_i| \sum_{p \in N'_i} d(p, c_i) & = & \sum_{q \in N'_i}\sum_{p \in N'_i} (d(p, c_i) + d(q, c_i)) \\
& \geq & \sum_{p \in N'_i} \sum_{q \in N'_i} d(p,q)
\geq \sum_{p \in N'_i} \sum_{q \in N'_i} d(q, c'_i) = |N'_i|\sum_{q \in N'_i} d(q, c'_i).
\end{eqnarray*}
This completes the proof.
\end{proof}

{\note} 
If we have an oracle that given a set of points
$C_i'$ finds the best center {\em in $X$} for that set, then we can save
a factor of $2$ \yingyu{in the approximation factor}.

\section{$\alpha$-Perturbation Resilience for the Min-Sum Objective}\label{section-minsum}

In this section we provide an efficient algorithm for clustering $\alpha$-perturbation resilient instances for the metric min-sum $k$-clustering problem (Algorithm~\ref{algminsum}).
Recall that $\ds(A,B) = \sum_{p\in A, q \in B} d(p,q)$ and $\da(A,B) = \ds(A,B) / (|A||B|)$.

\begin{algorithm}[!tbhp]
\caption{Min-sum, $\alpha$ perturbation resilience}\label{algminsum}
\label{AverageLinkage}
\begin{algorithmic}[1]
\REQUIRE{Data set $\data$, distance function $d(\cdot, \cdot)$ on $\data$, $\min_i |C_i|$.}
\STATE{Connect each point with its $\frac{1}{2}\min_i |C_i|$ nearest neighbors.}\label{componentConstruct}
\STATE{Initialize the clustering $\mathcal{C'}$ with each connected component being a cluster.}
\STATE{Repeat until only one cluster remains in $\mathcal{C'}$: \\ merge clusters $C,C'$ in $\mathcal{C'}$ which minimize $\da(C,C')$.}\label{componentMerge}
\STATE{Let $\mathcal{T}$ be the tree with components as leaves and internal nodes corresponding to the merges performed.}
\STATE{Run dynamic programming on $\mathcal{T}$ to get the minimum min-sum cost pruning ${\cal \tilde{C}}$.}
\ENSURE{${\cal \tilde{C}}$.}
\end{algorithmic}
\end{algorithm}

\begin{theorem}\label{thm:mainminsum}
For $(3\frac{\max_i |C_i|}{\min_i |C_i|})$-perturbation resilient instances, Algorithm~\ref{algminsum} outputs the optimal min-sum $k$-clustering in polynomial time.
\end{theorem}

\paragraph{Intuition}
To prove the theorem,
first we show that the $\alpha$-perturbation resilience property implies the following (Claim~\ref{lemma:minsumBasicFacts1}):
for any two different optimal clusters $C_i$ and $C_j$ and any
$A \subseteq C_i$, we have $\alpha \ds(A, C_i \setminus A) < \ds(A, C_j)$. This follows by considering the perturbation where $d'(p,q) = \alpha d(p,q)$ if $p\in A, q\in C_i \setminus A$ and $d'(p,q) = d(p,q)$ otherwise, and using the fact that the optimum  does not change after the perturbation.
This can be used to show that
when $\alpha > 3\frac{\max_{i'} |C_{i'}|}{ \min_{i'} |C_{i'}| }$, we have the following (Claim~\ref{lemma:minsumBasicFacts23}): (1) for any optimal clusters $C_i$ and $C_j$ and any
$A_i \subseteq C_i$, $A_j \subseteq C_j$ such that $\min(|C_i \setminus A_i|, |C_j \setminus A_j|)>\min_i |C_i|/2$ we have   $\da(A_i, A_j) > \min\{\da(A_i, C_i\setminus A_i), \da(A_j, C_j \setminus A_j)\}$; (2) for any point $p$  in the optimal cluster $C_i$, twice its average distance to points in $C_i \setminus \{p\}$ is smaller than the distance to any point in other optimal cluster $C_j$. Claim (2) implies that for any point $p \in C_i$ its $|C_i|/2$ nearest neighbors are in the same optimal cluster, so the leaves of the tree $\mathcal{T}$
are laminar to the optimum clustering. Claim (1) can be used to show that the merge steps preserve the laminarity with the optimal clustering, so  the minimum cost pruning of $\mathcal{T}$ will be the optimal clustering, as desired.

\begin{proofof}{Theorem~\ref{thm:mainminsum}}
We now present the formal proof of the theorem.
We first prove the key claims mentioned in the intuition.

\begin{fact}\label{fac:minsumBasicFacts}
For any nonempty sets $A,C,D$ such that $D \subseteq C$ we have\\
$
|D|\ds(A,C) \leq |C|\ds(A,D) + |A|\ds(D, C \setminus D).
$
\end{fact}
\begin{proof}
By the triangle inequality, for any $ p\in A, q \in C\setminus D$ and $z \in D$,
$$
  \ds(p,q) \leq \ds(p,z) + \ds(z,q).
$$
Summing over all $ p\in A, q \in C\setminus D$ and $z \in D$, we have
$$
  |D| \ds(A,C\setminus D) \leq |C\setminus D| \ds(A,D)+  |A| \ds(D,C\setminus D).
$$
Plugging in $\ds(A,C\setminus D) = \ds(A,C) - \ds(A,D)$ leads to the statement.
\end{proof}

\begin{claim}\label{lemma:minsumBasicFacts1}
Suppose the clustering instance is $\alpha$-perturbation resilient to the min-sum objective.
For any two different optimal clusters $C_i$ and $C_j$ and any
$A \subseteq C_i$, we have $\alpha \ds(A, C_i \setminus A) < \ds(A, C_j)$.
\end{claim}
\begin{proof}
We consider a specific perturbation and use the fact that the optimum does not change after the perturbation.
The \yingyu{perturbed metric} is defined as:
\begin{eqnarray*}
d'(p,q)  = 
\begin{cases}
\alpha d(p,q) & \textrm{if } p\in A, q\in C_i \setminus A  \textrm{ or } q\in A, p\in C_i \setminus A,\\
d(p,q) & \textrm{otherwise.}
\end{cases}
\end{eqnarray*}
Since $d'$ is a valid $\alpha$-perturbation of $d$, the optimal clustering after perturbation should remain the same.
In particular, its cost should be smaller than that of the clustering obtained
by replacing $C_i$ and $C_j$ with $C_i \setminus A$ and $A\cup C_j$.
After canceling the terms common in the two costs, we have $2\dps(A, C_i \setminus A) < 2\dps(A, C_j)$,
which then leads to $\alpha \ds(A, C_i \setminus A) < \ds(A, C_j)$.
\end{proof}

\begin{claim}\label{lemma:minsumBasicFacts23}
Suppose the clustering instance is $\alpha$-perturbation resilient to min-sum
for $\alpha > 3\frac{\max_{i'} |C_{i'}|}{ \min_{i'}|C_{i'}|}$. Then the following statements are true:
\begin{itemize}
\item[(1)] For any two different optimal clusters $C_i$ and $C_j$ and any \yingyu{nonempty}
$A_i \subseteq C_i, A_j \subseteq C_j$, if $|C_i \setminus A_i|$ and $|C_j \setminus A_j|$ are
larger than $\min_{i'} |C_{i'}|/2$, then $$\da(A_i, A_j) > \min\{\da(A_i, C_i \setminus A_i), \da(A_j, C_j \setminus A_j)\}.$$
\item[(2)] For any point $p$, all its $\min_{i'} |C_{i'}|/2$ nearest neighbors are in the same optimal cluster.
\end{itemize}
\end{claim}

\begin{proof}
(1) Let $\overline{A}_i \defeq C_i \setminus A_i$ and $\overline{A}_j \defeq C_j \setminus A_j$.
By Claim~\ref{lemma:minsumBasicFacts1} and Fact~\ref{fac:minsumBasicFacts},
\begin{eqnarray}
\alpha \ds(A_i, \overline{A}_i) < \ds(A_i,C_j) \leq \frac{1}{|A_j|}\biggl(|C_j|\ds(A_i, A_j) + |A_i|\ds(A_j, \overline{A}_j)\biggr),\label{eqn:minsum-left}\\
\alpha \ds(A_j, \overline{A}_j) < \ds(A_j,C_i) \leq  \frac{1}{|A_i|}\biggl(|C_i|\ds(A_j, A_i) + |A_j|\ds(A_i, \overline{A}_i)\biggr)\label{eqn:minsum-right}.
\end{eqnarray}
Divide Inequality~(\ref{eqn:minsum-left}) by $|A_i|$, divide Inequality~(\ref{eqn:minsum-right}) by $|A_j|$,
add them up, and move the $d(A_j, \overline{A_j})$ and $d(A_i, \overline{A_i})$ terms to the left-hand side:
\begin{eqnarray*}
(\alpha - 1)  |\overline{A}_j| \da(A_j, \overline{A}_j)
+ (\alpha - 1) |\overline{A}_i| \da(A_i, \overline{A}_i) &<&
(|C_i| + |C_j|) \da(A_i, A_j).
\end{eqnarray*}
Since $\alpha$,
$|\overline{A}_j|$ and $|\overline{A}_i|$ are large enough,
$(\alpha - 1) |\overline{A}_i| > |C_i|$ and  $(\alpha - 1) |\overline{A}_j| > |C_j|$.
So, 
$$
|C_j| \da(A_j, \overline{A}_j) + |C_i|\da(A_i, \overline{A}_i) < (|C_i|  + |C_j|)\da(A_i,A_j).
$$  
Dividing by $(|C_i|+|C_j|)$ we see that $\da(A_i,A_j)$ is greater than a weighted average of $\da(A_j, \overline{A}_j)$ and $\da(A_i, \overline{A}_i)$ and so is certainly greater than the minimum as desired.

(2) Suppose $p$ comes from the optimal cluster $C_i$.
Let $q =\arg\min_{p'\not\in C_i}d(p,p')$, and suppose $q\in C_j$.

If $\da(p, C_i) \geq \da(q, C_j)$,
then by Inequality~(\ref{eqn:minsum-left}),
\begin{eqnarray*}
\alpha \ds(p, C_i) \leq   |C_j| d(p, q) + \ds(q, C_j) =   |C_j| (d(p, q) + \da(q, C_j)) \leq |C_j| (d(p, q) + \da(p, C_i))
\end{eqnarray*}
which leads to $\da(p, C_i) < d(p,q)/2$ since $\alpha$ is sufficiently large.

If $\da(p, C_i) < \da(q, C_j)$, then we have $ \da(q,C_j) < d(p,q)/2$
by a similar argument.

In conclusion, we always have $\da(p, C_i)  < d(p,q)/2$. \yingyu{By the Markov inequality, at least half of all points in $C_i$ lie at distance at most $d(p, q) = \min_{p'\not \in C_i} d(p,p')$ from $p$.}
\end{proof}

We are now ready to use the lemmas to prove our theorem.
It is sufficient to show that in Algorithm~\ref{AverageLinkage}:
\begin{itemize}
\item[(1)] Initially each $A \in {\cal C}'$ satisfies $A \subseteq C$ for some $C \in {\cal C}$;
\item[(2)] $\mathcal{C'}$ is always laminar to the optimal clustering $\mathcal{C}$,
that is, for any $A \in \mathcal{C'}$ and $C \in \mathcal{C}$, we have either $A \subseteq C$, or $C\subseteq A$, or $A \cap C = \varnothing$.
\end{itemize}
Then the minimum cost pruning of $\mathcal{T}$ will be the optimal clustering, which can be obtained by dynamic programming.

Claim~\ref{lemma:minsumBasicFacts23}(2) implies that initially each $A \in {\cal C}'$ satisfies $A \subseteq C$ for some $C \in {\cal C}$.  This means that $\mathcal{C'}$ is laminar initially. Then Claim~\ref{lemma:minsumBasicFacts23}(1) can be used to show that the merge steps preserve the laminarity, so $\mathcal{C'}$ is always laminar to the optimal clustering.

More precisely, we prove the laminarity by induction.
By Claim~\ref{lemma:minsumBasicFacts23}(2), $\mathcal{C'}$ is laminar initially.
It is sufficient to prove that if the current clustering is laminar, then the merge step keeps the laminiarity.
Assume that our current clustering $\mathcal{C'}$ is laminar to the optimal clustering.
Consider a merge of two clusters $A$ and $A'$.
There are two cases when laminarity could fail to be satisfied after the merge:
\begin{itemize}
\item[(1)] \yingyu{$A$ and $A'$} are strict subsets from different optimal clusters, that is, $A \subsetneq C_i, A' \subsetneq C_j \neq C_i$;
\item[(2)] $A$ is a strict subset of an optimal cluster $C_i$ and $A'$ is the union of one or several other optimal cluster(s).
\end{itemize}
In the first case, $C_i \setminus A_i$ ($C_j \setminus A'$ respectively) contains at least one cluster in $\mathcal{C}'$.  By construction, each cluster in $\mathcal{C}'$ is of size at least $\min_{i'} |C_{i'}|/2$, and thus $C_i \setminus A_i$ and $C_j \setminus A'$ are of size at least $\min_{i'} |C_{i'}|/2$.
Then by Claim~\ref{lemma:minsumBasicFacts23}(1), $A$ and $A'$ cannot be merged.
In the second case, for any $E$ that is a subset of $C_i \setminus A$ in the current clustering, we have $\da(A, E) \geq \da(A, A')$.
We know that $\da(A, C_i \setminus A)$ is a weighted average of the average distances between $A$ and
the clusters that are subsets of $C_i \setminus A$ in the current clustering, so $\da(A, C_i \setminus A) \geq \da(A, A')$.
Also, $\da(A, A')$ is
a weighted average of the average distances between $A$ and the optimal clusters in $A'$,
so there must exist an optimal cluster $C_j \subseteq A'$ such that $\da(A, C_j) \leq \da(A, A') \leq \da(A, C_i \setminus A)$.
This means
\begin{eqnarray*}
\ds(A, C_j) \leq \frac{|C_j|}{|C_i \setminus A|} \ds(A, C_i \setminus A) \leq \alpha \ds(A, C_i \setminus A)
\end{eqnarray*}
where the last inequality comes from $\alpha \geq 3\frac{\max_{i'} |C_{i'}|}{\min_{i'} |C_{i'}|}$
and $|C_i \setminus A| \geq \min_{i'} |C_{i'}|/2$.
This contradicts Claim~\ref{lemma:minsumBasicFacts1}.
So the merge of the two clusters $A$ and $A'$ will preserve the laminarity.

\paragraph{Running Time} Finding the nearest neighbors for each point takes $O(n \log n)$ time,
so the step of constructing components takes $O(n^2 \log n)$ time.
To compute average distances between clusters, we can record the size of each cluster, and
$\ds(C'_i, C'_j)$ for any $C'_i, C'_j$ in the current clustering, and update
$\ds(C'_i \cup C'_j, C'_l ) = \ds(C'_i, C'_l) + \ds(C'_j, C'_l)$
for any other cluster $C'_l$ when merging $C'_i$ and $C'_j$.
So the merge steps take $O(n^3)$ time.
As dynamic programming takes $O(n^3)$ time, we can find the optimum clustering in $O(n^3)$ time.
\end{proofof}

\subsection{Sublinear Algorithm for the Min-Sum Objective}\label{sec:msinductive}

Here we provide a sublinear algorithm for a clustering instance $(X, d)$
that is $\alpha$-perturbation resilient to the min-sum objective.
For simplicity, suppose the distances are normalized such that $\max_{p,q} d(p,q) = 1$.
Let $N = |X|$. Let
$\rho = {\min_i|C_i|}/{N}$ denote the fraction of the points in the smallest optimal cluster,
and $\eta = \min_{j} \min_{p \not\in C_j} \da(p,C_j)$ denote the minimum average distance between an optimal cluster and a point outside.

\begin{algorithm}[!tbhp]
\caption{Min-sum, $\alpha$ perturbation resilience, sublinear}\label{alg:minsumInductive}
\begin{algorithmic}[1]
\REQUIRE{Data set $X$, distance function $d(\cdot, \cdot)$ on $X$, $\min_i |C_i|$.}
\STATE{Draw a sample $\data$ of size $n = \Theta(\frac{1}{\rho\eta}\ln{\frac{Nk}{\delta}})$ i.i.d.\ from $X$.}
\STATE{Run Algorithm~\ref{AverageLinkage} on $\data$ to obtain ${\cal \tilde{C}}$.}
\ENSURE{The implicit clustering of $X$ obtained by assigning each point $p \in X$ to $\tilde{C}_i \in {\cal \tilde{C}}$ such that $\ds(p, \tilde{C}_i )$ is minimized.}
\end{algorithmic}
\end{algorithm}

Our main result in this subsection is the following.

\begin{theorem} \label{thm:minsum-inductive}
Suppose the clustering instance $(X,d)$ is $\alpha$-perturbation resilient to the min-sum objective where
$\alpha \geq 6\frac{\max_i|C_i|}{\min_i|C_i|}$.
Then with probability at least $ 1-\delta$, Algorithm~\ref{alg:minsumInductive} outputs an implicit optimum clustering
in time $poly(\log{\frac{Nk}{\delta}}, \frac{1}{\rho}, \frac{1}{\eta})$.
\end{theorem}

\begin{proof}
To prove the theorem, we first show the following (Lemma~\ref{lemma:minsumInductiveLaminar}): with high probability,
$\mathcal{C'}$ in Algorithm~\ref{AverageLinkage} is always laminar to $\mathcal{C} \cap S$. The key idea is that
when the sample is sufficiently large, we have that for any $p \in C_i$ and $C_j (j\neq i)$,
$$3\frac{\max_i|C_i \cap S|}{\min_i|C_i \cap S |} \ds(p, C_i \cap S) < \ds(p, C_j \cap S)$$
since $\frac{\ds(p, C_i \cap S)}{n} \approx \frac{\ds(p, C_i)}{N}$,
$\frac{\ds(p, C_j \cap S)}{n} \approx \frac{\ds(p, C_j)}{N}$
and $\frac{\max_i|C_i \cap S|}{\min_i|C_i \cap S|} \approx \frac{\max_i|C_i|}{\min_i|C_i|}$.
Then $\mathcal{C} \cap S$ satisfies the properties for the linkage in Algorithm~\ref{AverageLinkage} to succeed, and
thus $\mathcal{C'}$ in Algorithm~\ref{AverageLinkage} is always laminar to $\mathcal{C} \cap S$. Then $\mathcal{C} \cap S$ is a pruning of the tree.

Then we show that $\mathcal{C} \cap S$ is actually the minimum cost pruning ${\cal \tilde{C}}$ (Lemma~\ref{lemma:minsumInductivePruning}).
\yingyu{
The key idea is that any other pruning of the same size must join some clusters in $\mathcal{C} \cap S$ and at the same time split some other clusters. Since the clusters in $\mathcal{C} \cap S$ are far apart, joining different clusters in it will increase the cost significantly, while splitting clusters will only save a small amount. So any other pruning must have larger cost than $\mathcal{C} \cap S$. }
It immediately follows from the two lemmas that the implicit clustering obtained is the optimum clustering $\mathcal{C}$.
\end{proof}

We now present the proofs of the lemmas \yingyu{that imply} the correctness of the theorem.

\begin{lemma}\label{lemma:minsumInductiveLaminar}
Suppose the clustering instance is $\alpha$-perturbation resilient to the min-sum objective for
$\alpha \geq 6\frac{\max_i|C_i|}{\min_i|C_i|}$.
When $n = O(\frac{1}{\rho \eta}\ln{\frac{Nk}{\delta}})$, with probability at least $1-\delta$,
$\mathcal{C'}$ in Algorithm~\ref{AverageLinkage} is always laminar to $\mathcal{C} \cap S$.
\end{lemma}

\begin{proof}
The intuition is that on $X$, for any $i \neq j$, any $p \in C_i$, we have $\alpha \ds(p, C_i) < \ds(p, C_j)$.
When $n$ is sufficiently large, we can show $\ds(p, C_i \cap S) \approx \frac{n}{N}\ds(p, C_i)$ for any $i$
and $\frac{\max_{i'}|C_{i'} \cap S|}{\min_{i'}|C_{i'} \cap S|} \approx \frac{\max_{i'}|C_{i'}|}{\min_{i'}|C_{i'}|}$, and thus we have a similar claim for $S$.
Then $\mathcal{C'}$ in Algorithm~\ref{AverageLinkage} is always laminar to $\mathcal{C} \cap S$.

First, since $n = O(\frac{1}{\rho \eta}\ln{\frac{Nk}{\delta}})$, by the Chernoff bound we have
\begin{eqnarray*}
\mathrm{Pr}\biggl[\biggl |\frac{|C_i \cap S|}{n} - \frac{|C_i|}{N} \biggr |
 \geq \upsilon\frac{|C_i|}{N} \biggr]
  \leq  2 \exp\left\{-\frac{2\upsilon^2}{2+\upsilon}\frac{|C_i|}{N} n\right\}
  \leq  2 \exp\{-\frac{2\upsilon^2}{2+\upsilon} \rho n\} \leq \frac{\delta}{4k}.
\end{eqnarray*}
Set $\upsilon=1/20$. By the union bound, with probability at least $1-\delta/4$, for any $1\leq i\leq k$,
\begin{eqnarray}
\left(1-\frac{1}{20}\right)  \frac{n}{N} |C_i|\leq |C_i \cap S|\leq \left(1+\frac{1}{20}\right) \frac{n}{N}|C_i|.\label{eqn:size}
\end{eqnarray}

Similarly, with probability at least $1-\delta/4$, for any $i\neq j$ and $p \in C_i$,
\begin{eqnarray}
\ds(p,C_j \cap S) \geq \left(1-\frac{1}{20}\right)  \frac{n}{N}\ds(p,C_j).\label{eqn:distance1}
\end{eqnarray}
Now, fix any $i$ and $p \in C_i$. We have
\begin{eqnarray*}
\mathrm{Pr}\biggl[ \frac{\ds(p,C_i \cap S)}{n} - \frac{\ds(p,C_i)}{N} 
 \geq \upsilon\frac{\ds(p,C_i)}{N} \biggr]
  \leq  \exp\left\{-\frac{2\upsilon^2}{2+\upsilon}\frac{\ds(p,C_i)}{N}  n\right\}.
\end{eqnarray*}
Let $\upsilon = \frac{\min_{j\neq i} \ds(p,C_j) }{10\ds(p,C_i)}$. Since $\upsilon \geq 1/10$, 
\begin{eqnarray*}
&&\mathrm{Pr}\biggl[ \frac{\ds(p,C_i \cap S)}{n} - \frac{\ds(p,C_i)}{N} 
 \geq \frac{\min_{j\neq i} \ds(p,C_j)}{10 N} \biggr] 
  \leq  \exp\left\{-\frac{2\upsilon}{21}\frac{\ds(p,C_i)}{N}  n\right\} \\
	& \leq & \exp\left\{-\frac{2}{210}\frac{\min_{j\neq i}\ds(p,C_j)}{N}  n\right\} \leq \exp\left\{-\frac{2}{210}\min_{j\neq i}\da(p,C_j) \rho n\right\} \leq  \frac{\delta}{4Nk}.
\end{eqnarray*}
By the union bound, with probability at least $1-\delta/4$, for any $j \neq i$ and $p \in C_i$,

\begin{eqnarray}
\ds(p,C_i \cap S) \leq \frac{n}{N}\ds(p,C_i) + \frac{n}{10 N}\ds(p,C_j).\label{eqn:distance2}
\end{eqnarray}

Now, by (\ref{eqn:size}), we have $\max_{i'}|C_{i'} \cap S| \leq (1+\upsilon) \frac{n}{N}\max_{i'}|C_{i'}|$,
$\min_{i'}|C_{i'} \cap S| \geq (1-\upsilon) \frac{n}{N} \min_{i'}|C_{i'}|$.
\yingyu{Combining these bounds with (\ref{eqn:distance1}) and (\ref{eqn:distance2}),} we have
that with probability at least $1-\delta$, for any $ i \neq j $ and any $p \in C_i$,
\begin{eqnarray} \label{eqn:AverageLinkage_sub}
3\frac{\max_{i'}|C_{i'} \cap S|}{\min_{i'}|C_{i'} \cap S|} \ds(p, C_i \cap S) < \ds(p, C_j \cap S),
\end{eqnarray}
which guarantees that Algorithm~\ref{AverageLinkage} will successfully outputs $\mathcal{C} \cap S$.
\yingyu{
More precisely, the proof of Theorem~\ref{thm:mainminsum} only depends on Claim~\ref{lemma:minsumBasicFacts1}:
for any two different optimal clusters $C_i$ and $C_j$ and any
$A \subseteq C_i$, we have $\alpha \ds(A, C_i \setminus A) < \ds(A, C_j)$, where $\alpha \geq 3\frac{\max_{i'}|C_{i'}|}{\min_{i'}|C_{i'}|}$.  (\ref{eqn:AverageLinkage_sub}) implies that for any two different optimal clusters $C_i$ and $C_j$ and any
$A \subseteq C_i$, we have $3\frac{\max_{i'}|C_{i'} \cap S|}{\min_{i'}|C_{i'} \cap S|} \ds(A, (C_i \cap S) \setminus A) < \ds(A, C_j \cap S)$. Therefore, the same argument for Theorem~\ref{thm:mainminsum} can be used to show that Algorithm~\ref{AverageLinkage} successfully outputs $\mathcal{C} \cap S$.
}
\end{proof}

\begin{lemma}\label{lemma:minsumInductivePruning}
Suppose the clustering instance is $\alpha$-perturbation resilient to the min-sum objective for
$\alpha \geq 6\frac{\max_i|C_i|}{\min_i|C_i|}$.
When $n = O(\frac{1}{\rho \eta}\ln{\frac{Nk}{\delta}})$, with probability at least $1-\delta$,
$\mathcal{C} \cap S$ is the unique minimum min-sum cost pruning of the tree in Algorithm~\ref{AverageLinkage}.
\end{lemma}

\begin{proof}
Since the tree is laminar to $\mathcal{C} \cap S$, we know that $\mathcal{C} \cap S$ is a pruning of the tree,
and any other pruning can be obtained by
splitting some clusters in $\mathcal{C} \cap S$ and joining some others into unions.
Intuitively, since the clusters in $\mathcal{C} \cap S$ are far apart, joining different clusters in it will increase the cost significantly, while splitting clusters will only save a small amount. So any other pruning must have larger cost than $\mathcal{C} \cap S$.
This claim then implies $\mathcal{C} \cap S$ is the minimum cost pruning.
We first prove a similar claim for $\mathcal{C}$ by the $\alpha$-perturbation resilience, that is, 
for any three different clusters $C_i, C_j, C_l \in \mathcal{C}$, any $A_X \subseteq C_i$,
$\alpha \ds(A_X, C_i \setminus A_X) < \ds(C_j, C_l)$.
Then we prove the claim for $\mathcal{C} \cap S$:
for any $A \subseteq C_i \cap S, 2\ds(A, C_i \cap S \setminus A) < \ds(C_j \cap S, C_l \cap S)$.
Finally we use it to prove $\mathcal{C} \cap S$ is the minimum cost pruning.

\begin{claim}\label{cla:minsum_three}
For any three different optimal clusters $C_i, C_j, C_l \in \mathcal{C}$, any $A_X \subseteq C_i$,
$\alpha \ds(A_X, C_i \setminus A_X) < \ds(C_j, C_l)$.
\end{claim}
\begin{proof}
For any $A_X \subseteq C_i$, we define a perturbation as follows: blow up the distances between the points in $A_X$ and those in $C_i \setminus A_X$
by a factor of $\alpha$, and keep all the other pairwise distances unchanged.
By the $\alpha$-perturbation resilience, we know that $\mathcal{C}$ is still the optimum clustering after perturbation.
Therefore, it has lower cost than the clustering obtained by replacing $C_i$ with $A_X$ and $C_i \setminus A_X$, and
replacing $C_j$ and $C_l$ with $C_j \cup C_l$. After canceling the common terms in the costs of the two clusterings,
we have $2 \dps(A_X, C_i \setminus A_X) < 2 \dps(C_j, C_l)$, which leads to the claim.
\end{proof}

\begin{claim}\label{cla:minsum_three2}
With probability at least $1-\delta$, for any three different optimal clusters $C_i, C_j, C_l \in \mathcal{C}$, and any $A \subseteq C_i \cap S$, 
$$
2\ds(A, C_i \cap S \setminus A) \leq 2\ds(C_i \cap S , C_i \cap S) < \ds(C_j \cap S, C_l \cap S).
$$
\end{claim}

\begin{proof}
On one hand, for $A_X \subseteq C_i$, $\alpha \ds(A_X, C_i \setminus A_X) < \ds(C_j, C_l)$ by Claim~\ref{cla:minsum_three}. Since 
$
\sum_{A_X \subseteq C_i} \ds(A_X, C_i \setminus A_X) = 2^{|C_i|} d(C_i, C_i)/2,
$ 
we have $\frac{\alpha}{2} \ds(C_i, C_i) < \ds(C_j, C_l).$
On the other hand, a similar argument as that of Lemma~\ref{lemma:minsumInductiveLaminar} shows with probability at least $1 - \delta/2$, for any $C_i$ and any $p \in C_i, \ds(p,C_i \cap S) \leq (1+\upsilon) \frac{n}{N}\ds(p,C_i)$ for $\upsilon=1/20$.
So 
\begin{eqnarray}
\ds(C_i \cap S, C_i \cap S)
& = & \sum_{p \in C_i \cap S} \ds(p, C_i \cap S)
\leq  (1+\upsilon) \frac{n}{N}\sum_{p \in C_i \cap S} \ds(p, C_i) \nonumber\\
& = & (1+\upsilon) \frac{n}{N} \sum_{q \in C_i}\ds(C_i \cap S, q) \leq (1+\upsilon)^2 \frac{n^2}{N^2} \ds(C_i, C_i). \nonumber \label{eqn:minsumInductiveCostSameCluster}
\end{eqnarray}
Similarly, with probability at least $1 - \delta/2$, for any $C_j$ and $C_l$, $\ds(C_j \cap S, C_l \cap S)\geq (1-\upsilon)^2 \frac{n^2}{N^2} \ds(C_j, C_l)$.
The claim then follows by combining the three inequalities and noting $\upsilon=1/20$.
\end{proof}

Now, we use Claim~\ref{cla:minsum_three2} to prove the optimality of $\mathcal{C}\cap S$.
Suppose a pruning $\mathcal{P^*}$ is obtained by splitting $h$ clusters in $\mathcal{C} \cap S$ and at the same time
joining some other clusters into $g$ unions.
Specifically, for $1\leq i \leq h$, split $C_i \cap S$ into $m_i \geq 2$ clusters $S_{i,1},\dots, S_{i, m_i}$;
after that, merge $C_{h + 1} \cap S, \dots, C_{h + l_g} \cap S$ into $g$ unions,
that is, for $1\leq j \leq g$ and $l_0 = 0$, merge $l_j - l_{j-1} \geq 2$ clusters $C_{h+l_{j-1}+1} \cap S,\dots,C_{h+l_{j}} \cap S$
into a union $U_j$; the other clusters in $\mathcal{C} \cap S$ remain the same in $\mathcal{P^*}$.
Since the number of clusters is still $k$, we have $\sum_{i} m_i - h = l_g - g$.
The cost saved by splitting clusters is
\begin{eqnarray}
&&\sum_{1\leq i \leq h} \ds(C_i \cap S, C_i \cap S) - \sum_{1\leq i \leq h} \ \ \sum_{1\leq p \leq m_i} \ds(S_{i,p}, S_{i,p}) \nonumber \\ 
& = &\sum_{1\leq i \leq h} \ \ \sum_{1\leq p \leq m_i} \ds(S_{i,p}, C_i \cap S \setminus S_{i,p}). \label{eqn:minsumInductiveCostSaved}
\end{eqnarray}
The cost increased by joining clusters is
\begin{eqnarray}
&& \sum_{1\leq j \leq g} \ds(U_j, U_j) - \sum_{1\leq j \leq g} \ \ \sum_{h+l_{j-1} < p \leq h+l_{j}} \ds(C_{p} \cap S, C_{p} \cap S) \nonumber\\
&=& \sum_{1\leq j \leq g} \ \ \sum_{h+l_{j-1} < p \neq q \leq h+l_{j}} \ds(C_{p} \cap S, C_{q} \cap S). \label{eqn:minsumInductiveCostIncreased}
\end{eqnarray}
To prove that $\mathcal{C} \cap S$ is the unique minimum cost pruning, we need to show that (\ref{eqn:minsumInductiveCostSaved}) is less than (\ref{eqn:minsumInductiveCostIncreased}).
Since each term in (\ref{eqn:minsumInductiveCostIncreased}) is twice larger than any term in (\ref{eqn:minsumInductiveCostSaved}),
it suffices to show that the number of the terms in (\ref{eqn:minsumInductiveCostIncreased}) is at least half the number of the terms in (\ref{eqn:minsumInductiveCostSaved}). Formally, we need to show $$2\sum_{1\leq j \leq g} {l_j - l_{j-1} \choose 2} \geq \sum_{1\leq i \leq h} m_i.$$
We have $2\sum_{j} {l_j - l_{j-1} \choose 2} = \sum_j (l_j - l_{j-1})(l_j - l_{j-1} - 1) \geq 2\sum_j (l_j - l_{j-1} - 1) = 2(l_g - g)$,
where the inequality comes from $l_j - l_{j-1} \geq 2$.
Since $l_g - g = \sum_{i} m_i - h$, it is sufficient to show $l_g - g \geq h$.
This comes from $l_g - g = \sum_{i} m_i - h = \sum_{i} (m_i - 1) \geq \sum_{i} 1 = h$ since $ m_i \geq 2$.
\end{proof}


\section{$(\alpha,\epsilon)$-Perturbation Resilience for the Min-Sum Objective}\label{sec:alpha_espilon_minsum}

For $(\alpha,\epsilon)$-perturbation resilient min-sum instances, we will show that if $\alpha \geq \frac{8\max_i |C_i|}{ \min_i |C_i|}, \epsilon = \tilde{O}\left(\rho\right)$ where $\rho=\frac{\min_i |C_i|}{ n}$,
there exists a polynomial time algorithm that outputs a clustering that is both a  $(1+\tilde{O}(\epsilon/\rho) )$-approximation
and also $\tilde{O}(\epsilon)$-close to the optimal clustering.
Formally,

\begin{theorem}\label{thm:alpha_espilon_minsum}
Suppose the instance is $(\alpha, \epsilon)$-perturbation resilient to the min-sum objective
for $\alpha > \frac{8\max_i |C_i|}{\min_i|C_i|}$ and $\epsilon < \frac{\min_i |C_i|}{600 n \log n}$.
There exists an algorithm that outputs a clustering which is a $\left(1+\frac{40 \epsilon n \log n}{\min_i |C_i|}\right)$-approximation to the optimal clustering in polynomial time.
Furthermore, the output clustering is also $(6\epsilon \log n)$-close to the optimal clustering.
\end{theorem}

Since $\epsilon = O\left(\frac{\min_i |C_i|}{n\log n}\right)$, the approximation factor is always $O(1)$ and gets better if $\epsilon$ gets smaller.
To prove the theorem, we first derive new useful structural properties implied by $(\alpha,\epsilon)$-perturbation resilience for min-sum, and then use them to design our algorithm achieving the guarantees in the theorem.
Throughout this section, we assume $\alpha > \frac{8\max_i |C_i|}{\min_i|C_i|}$ and $\epsilon < \frac{\min_i |C_i|}{600 n \log n}$, except for where their values are explicitly specified. 
Also, since $\max_i|C_i|/\min_i|C_i| < n$ we may assume without loss of generality that $\alpha \leq 8n$.

\yingyu{The rest of the section is organized as follows. In Section~\ref{sec:properties_aeminsum}, we prove useful properties of the $(\alpha,\epsilon)$-perturbation resilient min-sum instances. We first show that in the optimal clustering, except for a few bad points, all the other points are good in the sense that they are much closer to their own  clusters than to any other clusters.  Furthermore, we show that there exist a subset of points we call potentially good points which can act as a proxy for the good points in the clustering tasks.  Given these properties, we design an algorithm in Section~\ref{sec:algo_aeminsum}. We first construct a tree with a pruning close to the optimal clustering, find that pruning, and finally adjust the points so that the pruning becomes the desired approximation.}

%

\subsection{Structural Properties of $(\alpha,\epsilon)$-Perturbation Resilience}\label{sec:properties_aeminsum}
We describe the high level ideas for the structural properties, and then present the formal proofs in Section~\ref{sec:bound} and Section~\ref{sec:properties_gp_aeminsum}.

\mycomment{From this point to Section 6.1.1 are new intuitions added.}
The good points and bad points for min-sum are defined as follows. Here $\beta$ is chosen to be $\frac{4}{5}\alpha$ since we are not able to prove the bound for $\beta = \alpha$ but will be able to when $\beta$ is slightly smaller than $\alpha$. Some other constant can be used instead of $\frac{4}{5}$.

\begin{definition}
Define bad points for min-sum to be those that are not $\beta$ times closer to their own clusters than to other clusters, where $\beta=\frac{4}{5}\alpha$. That is, 
\begin{eqnarray*}
B \defeq \cup_i B_i, \textrm{where}~B_i \defeq \{p \in C_i: \exists j \neq i, d(p, C_j) \leq \beta d(p, C_i)\}, \beta \defeq \frac{4}{5}\alpha.
\end{eqnarray*}
The other points $G=\data \setminus B$ are called good points.
\end{definition}

We will show that when $\alpha$ is sufficiently large and $\epsilon$ is sufficiently small,
the number of bad points are bounded by $\tilde{O}(\epsilon n)$ (Theorem~\ref{thm:bpbound_aeminsum} in Section~\ref{sec:bound}).
Intuitively, if there are more bad points, then we can construct a perturbation, so that the optimal clustering after perturbation must move large fraction of the bad points to new clusters, which will then be contradictory to the $(\alpha,\epsilon)$-perturbation resilience.
Consider a bad point $p \in C_i$ and let $C_j$ be its second nearest cluster. We try to make sure that $p$ will move to $C_j$ while good points in $C_i$ stay. That is, we try to make sure that the optimal clustering after perturbation is $\{\tilde{C}_i\}$ where $\tilde{C}_i$ is the union of the good points in $C_i$ and the bad points that have $C_i$ as their second nearest cluster. So the perturbation should blow up all distances except those within $\tilde{C}_i (1\leq i\leq k)$. The proof then proceeds by considering clustering $\{C'_i\}$ that are $\epsilon$-close to $\{C_i\}$, and showing that $\{\tilde{C}_i\}$ has smaller cost than $\{C'_i\}$. 

There are some technical difficulties to be addressed. One is that we can only show $\{\tilde{C}_i\}$ has smaller cost than $\{C'_i\}$ when the costs of bad points are within a constant factor from each other. Therefore, we partition the bad points according to their costs so that there are $O(\log n)$ groups, and the bad points in the same group have costs within a factor of 2 from each other. 
Another technical difficulty is that besides bad points, there might be other points moved to new clusters in $\{C'_i\}$, so comparing the costs of $\{\tilde{C}_i\}$ and $\{C'_i\}$ requires bounding the costs of these points too. We address this by partitioning the points into several types so that the costs of points of the same type can be bounded in the same way.

Once we bound the number of bad points, it is possible to design approximation algorithms if the influence of the few bad points can be eliminated, since the good points in different optimal clusters are far from each other by definition, and thus can be handled by simple algorithms (such as the variant of average linkage algorithm used for $\alpha$-perturbation resilient min-sum instances). 
However, it is unclear how to compute the bad points or the good points.
The key is to introduce a proxy called potentially bad points (potentially good points respectively), which can be easily computed. These notions are formalized as follows.

\begin{definition}
\begin{itemize}
\item[(1)] Define $m_B \defeq 6 \epsilon \log n$.
\item[(2)] For a set $A$ with $|A| > 2 m_B$, define the potentially bad points $F(A)$ to be the $2 m_B$ points in $A$ that are farthest from $A$.
That is, $F(A) \subseteq A, |F(A)| = 2 m_B$, and for any $p \in F(A)$, $q \in S \setminus F(A)$, $d(p, A) \geq d(q, A)$. \yingyu{The potentially good points of $A$ are defined to be $P(A) \defeq A \setminus F(A)$.}
\item[(3)]  The robust average distance $\dra(A_1, A_2)$ between two sets $A_1, A_2$ is defined as the average distance between their potentially good points. Formally,
$$
\dra(A_1, A_2) \defeq \frac{\ds(P(A_1), P(A_2))}{|P(A_1)||P(A_2)|}.
$$
\item[(4)] For a cluster $A$, define its robust min-sum cost as $\drs(A) \defeq \ds(P(A), P(A))$.
For a clustering $\mathcal{C}$, define its robust min-sum cost as $\sum_{C\in \mathcal{C}} \drs(C)$.
\end{itemize}
\end{definition}

To show that the potentially good points can be regarded roughly as a proxy for the good points, we show that
the cost between sufficiently many good points in two clusters accounts for most of the cost between the two clusters (Lemma~\ref{lem:boundMergeCost}),
and that the cost between any points and the potentially good points in a cluster is roughly bounded by the cost between these points and any sufficiently large subset (in particular, the good points) of the cluster (Lemma~\ref{lem:FapproxG}). 
The first statement is due to that there are just a few bad points and the good points in different clusters are far apart. The second is due to that by definition, the potentially bad points are furthest away from other points in the cluster, so removing other points will not change the cost as much as removing the potentially bad points, even in the worst case.

Since the potentially good points  can act as a proxy,  the robust average distance approximates the average distance between good points, and the robust min-sum cost computed after removing the potentially bad points approximates the min-sum cost computed after removing the actual bad points. Thus, they can be used to design approximation algorithms using the potentially bad points as if we knew the actual bad points, as done in Section~\ref{sec:algo_aeminsum}. 

\subsubsection{Bounding the Number of Bad Points}\label{sec:bound}

Here we bound the number of bad points by $\tilde{O}(\epsilon n)$ when $\alpha$ is sufficiently large and $\epsilon$ is sufficiently small.

\begin{theorem}\label{thm:bpbound_aeminsum}
Suppose the clustering instance is $(\alpha, \epsilon)$-perturbation resilient to the min-sum objective
for $\alpha > 4$ and $\epsilon < \frac{\min_i |C_i| }{200  n}$.
Then $|B| \leq m_B = 6 \epsilon n \log n$.
\end{theorem}

\begin{proof}
We will first show that $|B| \leq 2 \eta \epsilon n$ where $\eta = \left\lceil \log \frac{\max_i \max_{p\in B_i} d(p,C_i)}{\min_i \min_{p\in B_i} d(p,C_i)}\right\rceil$, and then show that $ \eta \leq 3 \log n$, completing the proof.

As the first step, assume for contradiction $|B| > 2 \eta \epsilon n$.
We will construct a perturbation which will eventually lead to a contradiction.

We begin by introducing some notations.
\yingyu{Consider the following $\eta$ intervals:} $[2^{t-1} v, 2^t v]$ where $v=\min_i \min_{p\in B_i} d(p,C_i), 1\leq t \leq \eta$.
At least one of the intervals, say $[r,2r]$, will contain the costs of more than $2\epsilon n$ bad points.
Let $\hat{B}$ denote an arbitrary subset of $2\epsilon n$ bad points in this interval.
Let $\hat{B}_i = \hat{B} \cap C_i$ denote the selected bad points in the optimal cluster $C_i$. Let $K_i = C_i \setminus \hat{B}_i$ denote the other points in $C_i$, and set $K = \cup_i K_i$.
Denote as $D_i$ all those selected bad points whose second nearest cluster is $C_i$, that is, $D_i = \{p: \exists j~\textrm{such that}~p \in \hat{B}_j~\textrm{and}~i = \arg\min_{\ell\neq j}d(p, C_{\ell})\}$. Note that by definition we have $\cup_i D_i = \hat{B}$. Finally, let $\tilde{C}_i = K_i \cup D_i$. See Figure~\ref{fig:BCD} for an illustration. 

Now we are ready to construct the perturbation, which tries to make the selected bad points move to their second nearest clusters and keep the other points in their original clusters. That is, the perturbation favors the clustering $\tilde{C}_i$.
More precisely, the perturbation is constructed as follows: blow up all distances by a factor of $\alpha$ except those within $\tilde{C}_i, 1\leq i\leq k$. That is,
\begin{eqnarray*}
d'(p,q) = \left\{ \begin{array}{ll}
d(p,q) & \textrm{if $p \in \tilde{C}_i$, and $q \in \tilde{C}_i$ for some $i$,}\\
\alpha d(p,q) & \textrm{otherwise}.
\end{array} \right.
\end{eqnarray*}

To derive a contradiction, consider the optimal clustering after perturbation, denoted as $\{C'_i\}$.
Since there are more than $\epsilon n$ bad points in $\hat{B}$, by $(\alpha, \epsilon)$-perturbation resilience, not all of them move to new clusters in $\{C'_i\}$, and thus $\{C'_i\}$ is different from $\{\tilde{C}_i \}$.
In fact, we will show that the clustering $\{\tilde{C}_i \}$ has a lower cost than $\{C'_i\}$, which is a contradiction.
To do so, we consider the following thought experiment: remove all points in $C'_i \setminus \tilde{C}_i$ so that $\{C'_i\}$ becomes $\{C'_i \cap \tilde{C}_i\}$, and then add all points in $\tilde{C}_i\setminus C'_i$ so that  $\{C'_i \cap \tilde{C}_i\}$ becomes $\{\tilde{C}_i\}$.  The cost saved in the first step is 
\begin{eqnarray} \label{eqn:costsaved}
\sum_i \dps(C'_i,C'_i) - \sum_i \dps(C'_i \cap \tilde{C}_i, C'_i \cap \tilde{C}_i)
\end{eqnarray}
and the cost added in the second step is 
\begin{eqnarray} \label{eqn:costadded}
\sum_i \dps(\tilde{C}_i, \tilde{C}_i) - \sum_i \dps(C'_i \cap \tilde{C}_i, C'_i \cap \tilde{C}_i).
\end{eqnarray}
We will show that (\ref{eqn:costsaved}) is larger than (\ref{eqn:costadded}), leading to the contradiction.

\begin{figure*}[!thbp]
\begin{minipage}[t]{0.49\textwidth}
\centering

\begin{tikzpicture}
\usetikzlibrary{shapes,backgrounds}

\def\largeN{2.2}
\def\bbox{(-\largeN,-\largeN) rectangle (\largeN,\largeN)}
\clip \bbox;

\coordinate (centerC) at (-0.6,0);
\def\C{(centerC) circle (1.2)}

\coordinate (centerCt) at (0.6,0);
\def\Ct{(centerCt) circle (1.2)}

\fill[green!25] \C;
\fill[blue!25] \Ct;

\begin{scope}
	\clip \C;
	\fill[green!25!blue] \Ct;
\end{scope}

\draw \C;
\draw \Ct;

\node at (-2,-1.1) {$C_i$};
\node at (2,-1.1) {$\tilde{C}_i$};
\node at (-1.3,0) {$\hat{B}_i$};
\node at (0,0) {$K_i$};
\node at (1.3,0) {$D_i$};
\end{tikzpicture}
\caption{Perturbation construction.}\label{fig:BCD}
\end{minipage}
\begin{minipage}[t]{0.49\textwidth}
\centering

\begin{tikzpicture}
\usetikzlibrary{shapes,backgrounds}

\def\largeN{2.2}
\def\largeNu{2.2*1.5}
\def\bbox{(-\largeN,-\largeN) rectangle (\largeN,\largeNu)}
\clip \bbox;

\coordinate (centerC) at (-0.6,0);
\def\C{(centerC) circle (1.2)}

\coordinate (centerCt) at (0.6,0);
\def\Ct{(centerCt) circle (1.2)}

\coordinate (centerCp) at (0,1.039);
\def\Cp{(centerCp) circle (1.2)}

\fill[green!25] \C;
\fill[blue!25] \Ct;
\fill[red!25] \Cp;

\coordinate (F1) at (0,2.239);
\def\FP{(centerCp) -- (F1)}
\coordinate (F1int) at (0,\largeNu);
\coordinate (F2int) at (0,-\largeN);
\def\FPup{(F1int) -- (F2int) -- (\largeN, -\largeN) -- (\largeN, \largeNu) -- (F1int) --cycle}
\begin{scope}
	\clip \FPup;
	\fill[green!25!blue!25] \Cp;
\end{scope}
\begin{scope}[even odd rule]
	\clip \FPup \bbox;
	\fill[red!25] \Cp;
\end{scope}

\coordinate (Ft1) at (1.63,-0.61) {};
\def\FtP{(centerCt) -- (Ft1)}
\def\FtPup{(centerCt) -- (Ft1) -- (\largeN, -\largeN) -- (\largeN, \largeNu) -- (centerCt) --cycle}
\begin{scope}
	\clip \FtPup;
	\fill[green!35!red!35] \Ct;
\end{scope}
\begin{scope}[even odd rule]
	\clip \FtPup \bbox;
	\fill[red!25] \Ct;
\end{scope}

\begin{scope}
	\clip \C;
	\fill[green!25!blue!25] \Ct;
\end{scope}

\begin{scope}
	\clip \C;
	\fill[green!35!red!35] \Cp;
\end{scope}

\begin{scope}
	\clip \Ct;
	\fill[blue!25!red!25] \Cp;
\end{scope}

\begin{scope}
	\clip \Ct;
	\clip \C;
	\fill[blue!33!red!33!green] \Cp;
\end{scope}

\draw \C;
\draw \Ct;
\draw \Cp;
\draw \FP;
\draw \FtP;

\node at (-2,-1.1) {$C_i$};
\node at (2,-1.1) {$\tilde{C}_i$};
\node at (0,2.6) {$C'_i$};
\node at (-0.8,0.8) {$U_i$};
\node at (-0.5,1.6) {$V_i$};
\node at (0.4,1.6) {$W_i$};
\node at (1.3,0) {$\tilde U_i$};
\node at (0.77,-0.7) {$\tilde V_i$};
\node at (0,-0.5) {$\tilde W_i$};
\end{tikzpicture}
\caption{Different types of points.}\label{fig:UVW}
\end{minipage}
\end{figure*}

To bound the costs, we first divide the points into different types. See Figure~\ref{fig:UVW} for an illustration.
First, we need to move out $C'_i \setminus \tilde{C}_i$  from each $C'_i$. These points can be divided into three types:
\begin{itemize}
\item[(1)] $U_i = C'_i \cap \hat{B}_i$ are the selected bad points in $C_i$ that need to be moved out.
\item[(2)] $V_i= (C'_i \setminus \tilde{C_i}) \cap (\cup_{j\neq i} \hat{B}_j) = (C'_i \setminus D_i) \cap (\cup_{j\neq i} \hat{B}_j) = (\hat{B} \setminus\hat{B}_i) \cap (C'_i \setminus D_i) $ are the selected bad points that are from other optimal clusters. They are in $C'_i$ but their second nearest cluster is not $C_i$.
\item[(3)] $W_i= (C'_i \setminus \tilde{C_i}) \cap (\cup_{j\neq i} K_j) = \cup_{j\neq i} (K_j \cap C'_i)$ are points that are from $K_j$ for some $j\neq i$ and are in $C'_i$. But they are not from $K_i$ and thus are not in $\tilde{C_i}$.
\end{itemize}
Second, we need to move in $\tilde{C}_i \setminus C'_i$ for each $\tilde{C}_i$. Similarly, these points can also be divided into three types:
\begin{itemize}
\item[(1)] $\tilde{W}_i=K_i \setminus C'_i = \cup_{j \neq i} (K_i \cap C'_j)$ are those points in $K_i$ and in $C'_j$ for some $j\neq i$. This means that they are points in $\cup_j W_j$. More specifically, we have $\cup_i \tilde{W}_i = \cup_j W_j$.
\item[(2)] $\tilde{V}_i = (D_i \setminus C'_i)  \cap \left(\cup_{\ell \neq j} (\hat{B}_\ell \cap C'_j) \right)$ are part of the selected bad points whose second nearest cluster is $C_i$ but not in $C'_i$. They are originally in $\hat{B}_\ell$ for some $\ell$ but are in $C'_j$ for some $j\neq \ell$. In other words, they are points from $V_j$ for some $j$, and we have $\cup_i \tilde{V}_i = \cup_j V_j$. Formally,
\begin{eqnarray*}
  \tilde{V}_i =  (D_i \setminus C'_i)  \cap \left(\cup_j ((\hat{B} \setminus \hat{B}_j ) \cap C'_j) \right)
	 =  \cup_{j\neq i} \left( (\hat{B} \setminus \hat{B}_j ) \cap C'_j \cap D_i\right),\\
\cup_i \tilde{V}_i  = \cup_{j} \left( (\hat{B} \setminus \hat{B}_j ) \cap C'_j \cap (\cup_{i\neq j} D_i)\right)  = \cup_j \left( (\hat{B} \setminus \hat{B}_j) \cap C'_j \setminus D_j\right) = \cup_j V_j.
\end{eqnarray*}
\item[(3)] $\tilde{U}_i = (D_i \setminus C'_i)  \cap \left( \cup_{j} (\hat{B}_j \cap C'_j) \right)$ are also part of the selected bad points whose second nearest cluster is $C_i$. They are originally in $\hat{B}_j$ for some $j$ and are also in $C'_j$. In other words, they are points from $U_j$ for some $j$, and we have $\cup_i \tilde{U}_i = \cup_j U_j$.
\end{itemize}
In summary,  $C'_i$ and $\tilde{C}_i$ are partitioned into four parts respectively:
$$C'_i = (C'_i \cap \tilde{C}_i) \dcup U_i \dcup V_i \dcup W_i,~~\tilde{C}_i = (C'_i \cap \tilde{C}_i) \dcup \tilde{U}_i \dcup \tilde{V}_i \dcup \tilde{W}_i.$$
These different types of points have the following relations:
$$\cup_i \tilde{U}_i = \cup_j U_j, ~~\cup_i \tilde{V}_i = \cup_j V_j,~~\cup_i \tilde{W}_i = \cup_j W_j.$$

We now consider the costs saved and added when moving these points for all clusters $i$.
Suppose we first move out $\{W_i\}_{i=1}^k$, then $\{V_i\}_{i=1}^k$, and finally $\{U_i\}_{i=1}^k$. The cost saved by moving out $\{W_i\}_{i=1}^k$ is defined as
\begin{eqnarray}  \label{eqn:Wsave}
\Delta_W :=  \sum_i \dps(C'_i, C'_i)  - \sum_i \dps(C'_i \setminus W_i, C'_i \setminus W_i) \geq 2\sum_i \dps(W_i, C'_i \cap C_i).
\end{eqnarray} 
The cost saved by moving out $\{V_i\}_{i=1}^k$ is 
\begin{eqnarray} \label{eqn:Vsave}
\Delta_V & := & \sum_i \dps(C'_i \setminus W_i, C'_i \setminus W_i)  - \sum_i \dps(C'_i \setminus W_i \setminus V_i, C'_i \setminus W_i \setminus V_i)  \nonumber \\
& \geq & 2\sum_i \dps(V_i, C'_i \cap C_i).
\end{eqnarray} 
The cost saved by moving out $\{U_i\}_{i=1}^k$ is 
\begin{eqnarray} \label{eqn:Usave}
\Delta_U & := &  \sum_i \dps(C'_i \setminus W_i \setminus V_i, C'_i \setminus W_i \setminus V_i) - \sum_i \dps(C'_i \cap C_i, C'_i \cap C_i)  \nonumber\\
& \geq & 2\sum_i \dps(U_i, C'_i \cap K_i).
\end{eqnarray} 
Next, we move in $\{\tilde{W}_i\}_{i=1}^k$, then $\{\tilde{V}_i\}_{i=1}^k$, and finally $\{\tilde{U}_i\}_{i=1}^k$. The cost added by moving in $\{\tilde{W}_i\}_{i=1}^k$ is
\begin{eqnarray} \label{eqn:Wadd}
\Delta_{\tilde{W}} & := & \sum_i \dps(\tilde{W}_i \dcup (C'_i \cap \tilde{C}_i), \tilde{W}_i \dcup (C'_i \cap \tilde{C}_i)) - \sum_i \dps(C'_i \cap \tilde{C}_i, C'_i \cap \tilde{C}_i)  \nonumber\\
& \leq & 2 \sum_i \dps(\tilde{W}_i, \tilde{W}_i \dcup (C'_i \cap \tilde{C}_i)).
\end{eqnarray} 
The cost added by moving in $\{\tilde{V}_i\}_{i=1}^k$ is 
\begin{eqnarray} \label{eqn:Vadd}
\Delta_{\tilde{V}} & := & \sum_i \dps(\tilde{C}_i \setminus \tilde{U}_i, \tilde{C}_i \setminus \tilde{U}_i) - \sum_i \dps(\tilde{W}_i \dcup (C'_i \cap \tilde{C}_i), \tilde{W}_i \dcup (C'_i \cap \tilde{C}_i))   \nonumber\\
& \leq & 2 \sum_i \dps(\tilde{V}_i, \tilde{C}_i).
\end{eqnarray} 
The cost added by moving in $\{\tilde{U}_i\}_{i=1}^k$ is 
\begin{eqnarray} \label{eqn:Uadd}
\Delta_{\tilde{U}} := \sum_i \dps( \tilde{C}_i, \tilde{C}_i) - \sum_i \dps(\tilde{C}_i \setminus \tilde{U}_i, \tilde{C}_i \setminus \tilde{U}_i) \leq 2 \sum_i \dps(\tilde{U}_i, \tilde{C}_i).
\end{eqnarray} 

Clearly, 
$$ (\Delta_W + \Delta_V + \Delta_U ) - (\Delta_{\tilde{W}} + \Delta_{\tilde{V}} + \Delta_{\tilde{U}} )  = \sum_i \dps( C'_i, C'_i) - \sum_i \dps( \tilde{C}_i, \tilde{C}_i). $$ 
We are now ready to show that the cost saved $(\Delta_W + \Delta_V + \Delta_U ) $ is greater than the cost added $(\Delta_{\tilde{W}} + \Delta_{\tilde{V}} + \Delta_{\tilde{U}} )$, which leads to the contradiction $\sum_i \dps( C'_i, C'_i) > \sum_i \dps( \tilde{C}_i, \tilde{C}_i)$.
The high level idea is that a significant amount of cost is saved by moving $U_i$ to the correct clusters, while the costs added are generally small since the number of points moved is bounded by $3\epsilon n$ and the cost of the selected bad points moved is bounded by $2r$. 
Formally, we have the following claim, whose proof is presented in Appendix~\ref{app:bound}.

\begin{claim}\label{cla:U}\label{cla:V}\label{cla:W}
The costs saved and added by moving out $\{ U_i\}_{i=1}^k,\{V_i\}_{i=1}^k$ and $\{W_i\}_{i=1}^k$,  and then moving in $\{\tilde{W}_i\}_{i=1}^k$, $\{\tilde{V}_i\}_{i=1}^k$ and $\{\tilde{U}_i\}_{i=1}^k$ satisfy:
\begin{eqnarray*}
(1) \ \Delta_U - \Delta_{\tilde{U}}  & \geq & 2\sum_i \dps(U_i, C'_i \mycap K_i) - 2 \sum_i \dps(\tilde{U}_i, \tilde{C}_i) 
 \\  &\geq & \frac{3}{10}\alpha \sum_i \ds(U_i, C_i) - \frac{2\alpha}{100} \sum _i \ds(W_i, C_i)  - \frac{8\alpha + 16}{100}r \epsilon n,\\
(2) \ \ \Delta_V - \Delta_{\tilde{V}}  & \geq &  2\sum_i \dps(V_i, C'_i \mycap C_i) - 2\sum_i \dps(\tilde{V}_i, \tilde{C}_i)
 \\ &\geq &  \frac{99}{50}(\alpha  -2)\sum_i \ds(V_i, C_i) - \frac{2\alpha}{100}\sum_i \ds(W_i, C_i)  - \frac{4\alpha + 8\beta}{100}r\epsilon n,\\
(3) \Delta_W - \Delta_{\tilde{W}}  & \geq &  2\sum _i \dps(W_i, C'_i \mycap C_i) - 2\sum_i \dps(\tilde{W}_i, \tilde{W}_i \dcup (C'_i \mycap \tilde{C}_i))
  \\  &\geq & \frac{98}{50}(\alpha  -2) \sum_i \ds(W_i, C_i)   - \frac{4\alpha+4\beta }{100}r\epsilon n.
\end{eqnarray*}
\end{claim}

After adding up all the inequalities in the claim, the right-hand side is a lower bound on $\dps( C'_i, C'_i) > \dps( \tilde{C}_i, \tilde{C}_i)$,
which we now show must be positive when $\alpha > 4$ and $\beta \leq \frac{4}{5}\alpha$.
The terms about $\ds(W_i,C_i)$ and $\ds(V_i, C_i)$ are non-negative, so it suffices to show that 
$$
 \frac{3}{10}\alpha \sum_i \ds(U_i, C_i) - \frac{8\alpha + 16}{100}r \epsilon n -  \frac{4\alpha + 8\beta}{100}r\epsilon n - \frac{4\alpha+4\beta }{100}r\epsilon n > 0.
$$
Since $30\alpha > 16\alpha + 12 \beta + 16$, what remains is to show $\sum_i \ds(U_i, C_i) \geq r\epsilon n$.
First, $\ds(p,C_i)\geq r$ for any $p \in U_i$. Second, $|\cup_i U_i| \geq \epsilon n$ since there are $2\epsilon n$ selected bad points
but no more than $\epsilon n$ of them move from their original clusters in $\{C_i\}$ to a different cluster in $\{C'_i\}$.
Then we have $\sum_i \ds(U_i,C_i)\geq \sum_i r |U_i| = r \sum_i |U_i| \geq r \epsilon n$.
Hence, the difference between the cost saved and the cost added is positive. This means the cost of $\{\tilde{C}_i\}$ is smaller than the cost of $\{C'_i\}$, which contradicts the assumption that $\{C'_i\}$ is the optimal clustering under $d'$.
Therefore, there can be at most $2\eta \epsilon n$ bad points.

Finally, it suffices to show that $\eta \leq 3\log n$.
Suppose $p_1$ is the point that achieves $\max_i \max_{ p \in B_i} \ds(p,C_i)$ and $p_2$ is the point that achieves $\min_i \min_{p \in B_i} \ds(p,C_i)$.
Without loss of generality, suppose $p_1 \in C_1$ and $p_2 \in C_2$.
By the definition of bad points, there exists $C_i \neq C_2$ such that $\ds(p_2, C_i) \leq \beta \ds(p_2,C_2)$.
If $C_i\neq C_1$, we have $\ds(p_1,C_1) \leq \ds(C_2, C_i)$, since otherwise we can get lower cost by splitting $C_1$ into $p_1$ and $C_1 \setminus \{p_1\}$ while merging $C_2$ and $C_i$.
If $C_i = C_1$, we also have $\ds(p_1,C_1) \leq \ds(C_2, C_i)$, since otherwise we can get lower cost by splitting $C_1$ into $p_1$ and $C_1 \setminus \{p_1\}$ and then merging $C_2$ and $C_1 \setminus \{p_1\}$.
In both cases, we have
\begin{eqnarray*}
\ds(p_1,C_1) \leq \ds(C_2, C_i)  & \leq &  |C_i| \ds(p_2, C_2) + |C_2| \ds(p_2, C_i) \\
 &\leq & |C_i| \ds(p_2, C_2) + \beta |C_2| \ds(p_2, C_2) \\
 & \leq & 8 n^2 \ds(p_2, C_2)
\end{eqnarray*}
where the last inequality follows from $\beta \leq 8n$. Then we have $\eta \leq 3\log n$.
\end{proof}

\subsubsection{Properties of  Good Points and Potentially Good Points}
\label{sec:properties_gp_aeminsum} \label{sec:properties_pgp_aeminsum}

Since there are just a few bad points and the good points in different clusters are far apart, the cost between sufficiently large subsets of their good points accounts for most of the cost between the two clusters. This means that \yingyu{we would be able to approximate the min-sum cost of all points by the min-sum cost only on the good points, if we knew the good points (Lemma~\ref{lem:boundMergeCost}).}
\yingyu{To prove Lemma~\ref{lem:boundMergeCost}, we will need the triangle inequality for the average distance, and also a technical lemma which shows that good points are much closer to its own cluster than to good points in any other cluster.}

\begin{fact}\label{fac:tri}
For any nonempty sets $A,B$ and $C$, we have $\da(A,B) \leq  \da(A,C) + \da(C,B)$, and thus
$\ds(A,B) \leq \frac{|B|}{|C|}\ds(A,C) + \frac{|A|}{|C|}\ds(C,B).$
\end{fact}
\begin{proof}
It follows from
\begin{eqnarray*} 
|C|\sum_{a \in A} \sum_{b\in B} d(a,b)  & = & \sum_{a \in A} \sum_{b\in B} \sum_{c\in C} d(a,b) \leq \sum_{a \in A} \sum_{b\in B} \sum_{c\in C} (d(a,c) + d(c,b)) \\ 
 &= & |B| \sum_{a\in A} \sum_{c\in C} d(a,c) + |A| \sum_{b\in B} \sum_{c\in C} d(c,b).
\end{eqnarray*}
\end{proof}

\begin{lemma}\label{lem:goodSubsets}
For any nonempty $A \subseteq G_i, B \subseteq G_j, j\neq i$,
we have
$$\da(A, C_i) \leq \gamma_{ji} \ \da(A, B), \mathrm{where} \ \ \gamma_{ji} = \frac{|C_j|}{(\beta-1/\beta)|C_i|} + \frac{1}{\beta^2-1}.$$
Consequently,
if $\alpha>\frac{8\max_i|C_i|}{\min_i |C_i|}$, we have $\da(A, C_i) \leq \frac{11}{50} \ \da(A, B)$.
\end{lemma}

\begin{proof}
For any $p \in A$, we have $\beta \ds(p, C_i) < \ds(p, C_j)$.
By Fact~\ref{fac:tri},
\begin{eqnarray*}
\beta \ds(A, C_i) < \ds(A, C_j) \leq \frac{|C_j|}{|B|} \ds(A, B) + \frac{|A|}{|B|} \ds(C_j, B)\\
\beta \ds(B, C_j) < \ds(B, C_i) \leq \frac{|C_i|}{|A|} \ds(B, A) + \frac{|B|}{|A|} \ds(C_i, A).
\end{eqnarray*}
Plug the second inequality into the first inequality, then the lemma follows.
\end{proof}

We are now ready to prove Lemma~\ref{lem:boundMergeCost}. 
\begin{lemma} \label{lem:boundMergeCost}
Suppose $\alpha > \frac{8\max_i |C_i|}{\min_i|C_i|}$
and $W_i \subseteq G_i, W_j \subseteq G_j$.
When $|C_i| \geq 50 |C_i \setminus W_i|$ and $|C_j| \geq 50 |C_j \setminus W_j|$, we have $\ds(C_i, C_j) \leq \frac{3}{2} \ds(W_i, W_j).$
\end{lemma}

\begin{proof}
By Fact~\ref{fac:tri} and Lemma~\ref{lem:goodSubsets}, we have
$$\da(C_i,C_j) \leq \da(C_i,W_i) + \da(W_i,W_j) + \da(W_j,C_j) \leq (\frac{11}{50} + 1 + \frac{11}{50}) \da(W_i,W_j)$$
which leads to $\ds(C_i, C_j) \leq \frac{36}{25}\frac{|C_i||C_j|}{|W_i||W_j|} \ds(W_i, W_j) \leq \frac{3}{2} \ds(W_i, W_j).$
\end{proof}

%
%

\begin{figure}
\begin{center}

\begin{tikzpicture}
\usetikzlibrary{shapes,backgrounds}

\def\largeN{3}
\def\bbox{(-\largeN,-\largeN) rectangle (\largeN,\largeN)}

\coordinate (center) at (0,0);
\def\A{(center) circle (1.2)}

\coordinate (F1) at (-1.2,1.6);
\coordinate (F2) at (2,-0.8);
\def\FP{(F1) -- (F2)}
\coordinate (F1int) at (F1 |- 0,\largeN);
\coordinate (F2int) at (F2 -| \largeN,0);
\def\FPup{\FP -- (F2int) -- (\largeN, \largeN) -- (F1int) --cycle}

\coordinate (B1) at (1.2,1.6);
\coordinate (B2) at (-2,-0.8);
\def\BG{(B1) -- (B2)}
\coordinate (B1int) at (B1 |- 0,\largeN);
\coordinate (B2int) at (B2 -| -\largeN,0);
\def\BGup{\BG -- (B2int) -- (-\largeN, \largeN) -- (B1int) --cycle}

\fill[blue!25] \A;
\node at (0,-0.5) {$W$};

\begin{scope}
	\clip \FPup;
	\fill[green!25] \A;
\end{scope}
\node at (0.6,0.6) {$X$};

\begin{scope}
	\clip \BGup;
	\fill[red!25] \A;
\end{scope}
\node at (-0.64,0.6) {$Y$};

\begin{scope}
	\clip \FPup;
	\clip \BGup;
	\fill[gray!25] \A;
\end{scope}
\node at (0,0.98) {$V$};

\node at (0,1.8) {$A$};

\draw \A;
\draw \FP;
\draw \BG;

\node at (1.6,-1.2) {$P$};
\node at (2,-0.4) {$F$};

\node at (-1.6,-1.2) {$H$};
\node at (-2,-0.4) {$\bar{H}$};
\end{tikzpicture}
\caption{Notations in Lemma~\ref{lem:FapproxG}.}\label{fig:WVXY1}
\vspace{-.2in}
\end{center}
\end{figure}
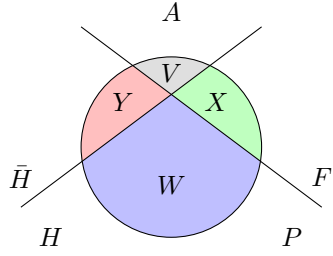

Now we turn to analyze the potentially good points.
A key property of the potentially good points is the following: for any point $p$ and any sufficiently large set $A$,
the cost between $p$ and the potentially good points in $A$
is roughly bounded by the cost between $p$ and any sufficiently large subset $H$ of $A$.
See Lemma~\ref{lem:FapproxG} for details.
A specific case is when $H$ is the actual good points in $A$.
In this case, the property says that the cost between $p$ and the potentially good points
is roughly bounded by the cost between $p$ and the actual good points.
This means that in suitable situations, we can regard potentially good points as actual good points.

\begin{lemma}\label{lem:FapproxG}
Suppose $H \subseteq A$ such that $|A\setminus H| \leq m_B$.
Let $F=F(A), P=P(A), \bar{H}=A\setminus H$.
Let $W = H \cap P, V = \bar{H} \cap F, X = F \cap H, Y=\bar{H}\cap P$.
See Figure~\ref{fig:WVXY1} for an illustration.
If $|A| \geq 20m_B$, then for any $p$,
$\ds(p, P) \leq \frac{|W|+|Y|}{|W|-|X|}\ds(p, H).$
\end{lemma}

\begin{proof}
Since $\ds(p,H) = \ds(p,X) + \ds(p,W)$ and $\ds(p, P) = \ds(p, Y) + \ds(p,W)$,
the lemma is true if $Y=\emptyset$.
Otherwise, we need to compare $\ds(p,X)$ and $\ds(p,Y)$.
By the triangle inequality, we have 
\begin{eqnarray*}
\da(W,X) \leq \da(W,p) + \da(p,X),&&~
\da(p,Y) \leq  \da(p,W) + \da(W,Y)
\end{eqnarray*}
which then lead to
\begin{eqnarray*}
\ds(p,X) \geq  \frac{\ds(W,X)}{|W|} - \frac{|X|}{|W|} \ds(p,W), &&~
\ds(p,Y) \leq \frac{\ds(W,Y)}{|W|} + \frac{|Y|}{|W|} \ds(p,W).
\end{eqnarray*}
From these bounds on $\ds(p,X)$ and $\ds(p,Y)$, we have
\begin{eqnarray*}
\ds(p,H) \geq \frac{\ds(W,X)}{|W|} + \frac{|W|-|X|}{|W|} \ds(p,W), &&~
\ds(p,P) \leq \frac{\ds(W,Y)}{|W|} + \frac{|W|+|Y|}{|W|} \ds(p,W).
\end{eqnarray*}
The lemma then follows from these two inequalities and the following claim.

\begin{claim}\label{lem:farPointCost}
$ \ds(X,W) \geq \ds(Y,W \dcup Y)$.
\end{claim}

\begin{proof}
The claim is true if $Y=\emptyset$. Otherwise, by the definition of the potentially bad points $F=F(A)$, we have $\da(X,A) \geq \da(Y,A)$.
By definition,
\begin{eqnarray}
|A| \da(X,A) & = &|W|\da(X,W) + |V| \da(X,V)  \label{eqn:farPointCostX}\\
& & + |Y|\da(X,Y) + |X| \da(X,X), \nonumber\\
|A| \da(Y,A) & = & |W \dcup Y|\da(Y,W \dcup Y) + |V| \da(Y,V) + |X| \da(Y,X). \label{eqn:farPointCostY}
\end{eqnarray}
To compare $\ds(X,W)$ and $\ds(Y, W  \dcup  Y)$, we need to bound the other terms in (\ref{eqn:farPointCostX}) and (\ref{eqn:farPointCostY}).
By Fact~\ref{fac:tri},
$$\da(X,V) \leq \da(X,W) + \da(W,Y) + \da(Y,V),$$ $$\da(X,Y) \leq \da(X,W) + \da(Y,W), \ \ \da(X,X) \leq 2\da(X,W).$$
Now we plug these into (\ref{eqn:farPointCostX}),
and then plug (\ref{eqn:farPointCostX}) and (\ref{eqn:farPointCostY}) into $\da(X,A) \geq \da(Y,A)$.
Since $\ds(W,Y) \leq \ds(Y,W \dcup Y)$ and $\da(Y,X) \geq 0$, we have
\begin{eqnarray*}
\bigl( |W| - |Y| - |V| \bigr) \ds(Y,W \dcup Y) \leq \bigl( |W| + 2|X| + |Y| + |V|\bigr) \frac{|Y|}{|X|} \ds(X,W).
\end{eqnarray*}
Since $|X  \dcup V| = |F| = 2 m_B$ and $|Y \dcup V| = |A\setminus H| \leq m_B$, we have $\frac{|Y|}{|X|}\leq 1/2$.
Then the lemma follows from the fact that $|A| \geq 20m_B, |F| = 2m_B$ and $|A\setminus H| \leq m_B$.
\end{proof}


This then completes the proof of Lemma~\ref{lem:FapproxG}.
\end{proof}

\subsection{Approximation Bound}\label{sec:algo_aeminsum}

In this subsection, we design an approximation algorithm and prove our final result Theorem~\ref{thm:alpha_espilon_minsum} by utilizing the properties of the $(\alpha,\epsilon)$-perturbation resilience.

\mycomment{Below to before Section 6.2.1 are new intuitions added}
First, note that we can generate a list of sufficiently large almost ``pure'' blobs using Algorithm~\ref{alg:aekmedian_blobs}.
However, unlike for $(\alpha, \epsilon)$-perturbation resilient $k$-median instances, it is not guaranteed that the robust linkage procedure in~\cite{nina_colt_2010} can link these blobs into a tree so that a pruning of the tree assigns all but bad points correctly.
Fortunately, since the potentially good points can act as a proxy for the good points, we can pretend there are only good points. Since the average linkage succeeds in this case (as shown for the $\alpha$-perturbation resilient instances), one would expect that the same idea can be applied. Indeed, we apply the idea but using the robust average distance instead of the average distance.
As described in Algorithm~\ref{alg:raLinkage}, we first use Algorithm~\ref{alg:aekmedian_blobs} to generate a list of blobs, and then use a robust version of average linkage to link them into a tree: repeatedly merge the two blobs with the minimum robust average distance.

After building the tree, one would like to find the pruning that assigns all but bad points correctly.
Suppose we can remove the actual bad points and compute the cost between the good points.
Since the good points from different clusters are far apart,
the good point cost increased by joining different clusters in $\mathcal{C}'$ is larger than that saved
by splitting clusters in $\mathcal{C}'$ (Lemma~\ref{lem:splitMerge}). Then any other pruning has larger cost than $\mathcal{C}'$.
Unfortunately, we do not know the actual good points.
Therefore, we consider the potentially good points and compute the robust min-sum cost.
 We show that the pruning is in fact the pruning with the minimum robust min-sum cost, so that it can be computed in polynomial time by dynamic programming.

\newcommand{\minSizeFrac}{\frac{1}{2}}
\newcommand{\clist}{\mathcal{L}}

\begin{algorithm}[!t]
\caption{Robust Average Linkage}
\label{alg:raLinkage}
\begin{algorithmic}[1]
\REQUIRE{Data set $\data$, distance function $d(\cdot, \cdot)$ on $\data$, $\min_i |C_i|$, $\epsilon>0$.}
\STATE{Use Algorithm~\ref{alg:aekmedian_blobs} with $u_B=6\epsilon n \log n$ and $\gamma=\frac{4}{5}$ to get a list $\blobs_0$ of blobs.}
\STATE{Initialize the clustering $\clist$ with each blob being a cluster.}
\STATE{Repeat till only one cluster remains:\\ merge clusters $C,C'$ which minimize $\dra(C,C')$. }
\STATE{Let $\mathcal{T}$ be the tree with blobs as leaves and internal nodes corresponding to the merges performed.}
\ENSURE{The tree $\mathcal{T}$.}
\end{algorithmic}
\end{algorithm}

\begin{algorithm}[!t]
\caption{Getting a good approximation}\label{alg:minsum_app}
\begin{algorithmic}[1]
\REQUIRE A clustering $\mathcal{C'}=\{C'_1,\dots, C'_k\}$, where $G_i \subseteq C'_i \subseteq C_i \mycup B$.
\FOR{each point $p$}
    \STATE Associate $p$ to the index $i$ such that $\ds(p, P(C'_i))$ is minimized. \label{step:assign}
\ENDFOR
\STATE Let $C''_i$ be the set of points associated to the index $i$.
\ENSURE The clustering $\mathcal{C''}=\{C''_1,\dots, C''_k\}$.
\end{algorithmic}
\end{algorithm}

\begin{algorithm}[!t]
\caption{Min-sum, $(\alpha,\epsilon)$ perturbation resilience}\label{alg:minsum_ae}
\begin{algorithmic}[1]
\REQUIRE{Data set $\data$, distance function $d(\cdot, \cdot)$ on $\data$, $\min_i |C_i|$, $\epsilon>0$.}
\STATE Run Algorithm~\ref{alg:raLinkage}  to get a tree $\mathcal{T}$.
\STATE Find the pruning $\mathcal{C'}$ with the minimum robust min-sum cost in the tree $\mathcal{T}$ by dynamic programming. 
\STATE Run Algorithm~\ref{alg:minsum_app}  to get the final clustering $\mathcal{C''}$.
\ENSURE The clustering $\mathcal{C''}=\{C''_1,\dots, C''_k\}$.
\end{algorithmic}
\end{algorithm}

However, this pruning may not be a good approximation. For example, consider an instance consisting of two unbalanced clusters.
Assume that there is only one bad point, belonging to the small cluster. Further assume the distances between the good points in each cluster are negligible, then assigning the bad point incorrectly to the large cluster will lead to an $\Omega\left(\frac{\max_i|C_i|}{\min_i |C_i|} \right)$-approximation. So the pruning $\mathcal{C}'$ may not be a constant approximation. Notice that the bad point causing trouble in this example can actually be identified: it is closer to its own optimal cluster than to its cluster in $\mathcal{C}'$.
Then by reassigning the points in $\mathcal{C}'$,  a better approximation can be computed.
It turns out that the reassignment is useful beyond this particular example, and can be used to compute a good approximation for general perturbation resilient instances.
The details are described in Algorithm~\ref{alg:minsum_app}.

All these combined together lead to our final algorithm for $(\alpha,\epsilon)$-perturbation resilient min-sum instances, summarized in Algorithm~\ref{alg:minsum_ae}.

The rest of the subsection presents the formal proofs.
In Section~\ref{sec:link}, we show that Algorithm~\ref{alg:raLinkage} outputs a tree with a pruning that assigns all but bad points correctly.
In Section~\ref{sec:pruning}, we show that this pruning can be found in polynomial time by dynamic programming.
In Section~\ref{sec:cost}, we show Algorithm~\ref{alg:minsum_app} computes a good approximation,
completing the proof of Theorem~\ref{thm:alpha_espilon_minsum}.

\subsubsection{Constructing A Tree with A Pruning Close to the Optimum}\label{sec:blob}\label{sec:link}

We now present our guarantee of Algorithm~\ref{alg:raLinkage}.

\begin{lemma}\label{thm:pruning}
The tree output in Algorithm~\ref{alg:raLinkage} has a pruning $\mathcal{C'}$ that assigns all good points correctly.
\end{lemma}

\begin{proof}
To analyze the algorithm, we begin with the following property of good points.
When combined with the property of  Algorithm~\ref{alg:aekmedian_blobs} (Lemma~\ref{lem:bloblist}),
it immediately shows that each blob in the list $\blobs_0$ has size at least $\minSizeFrac\min_i |C_i|$,
and contains good points from only one optimal cluster.

\begin{claim}\label{cla:nn}
For any $p \in G_i$,
all its $\frac{4|C_i|}{5}$ nearest neighbors belong to $C_i \cup B$.
\end{claim}

\begin{proof}
We need to show that for any $j \neq i$ and any good point $q \in G_j$, $d(p,q)$ is sufficiently large compared to $\da(p,C_i)$.
Intuitively, $p$ is much farther away from $C_j$ than from $C_i$, that is, $\beta \ds(p,C_i) \leq \ds(p,C_j)$.
It suffices to bound $\ds(p,C_j)$ by $d(p,q)$ and $\ds(p,C_i)$.
By the triangle inequality, 
\begin{eqnarray*}
\ds(p,C_j) & \leq & |C_j|d(p,q) + \ds(q,C_j) \\
\ds(q,C_j) & \leq & \frac{1}{\beta}\ds(q,C_i) \leq \frac{|C_i|}{\beta} d(p,q) + \frac{1}{\beta}\ds(p,C_i).
\end{eqnarray*}
Combining these inequalities, we have
$(\beta - \frac{1}{\beta})\ds(p,C_i) \leq   (|C_j|+ \frac{|C_i|}{\beta}) d(p,q).$
When $\alpha > 8 \frac{\max_i |C_i|}{\min_i |C_i|}$, $5\da(p,C_i) < d(p,q)$,
which then leads to the conclusion.
\end{proof}

It now suffices to prove by induction that the clustering $\clist \cap G$ is always laminar to $\mathcal{C} \cap G$.
It is true at the beginning by the property of Algorithm~\ref{alg:aekmedian_blobs}.
Assume for contradiction that the laminarity is first violated after merging $A$ and $D$.
There are two cases: 
\begin{itemize}
\item[(1)] $A$ and $D$ are strict subsets of different optimal clusters;
\item[(2)] $A$ is a strict subset of $G_i$ while $D$ is the union of the good points in several optimal clusters.
\end{itemize}

We have the following statements for the two cases respectively. By these two statements, we should first merge $A$ with $A'$ rather than with $D$, which is contradictory and completes the proof.

\begin{claim}\label{cla:case1}\label{cla:case2}
\begin{itemize}
\item[(1)] Suppose $A \in \clist, A \cap G \subsetneq G_i$, and $D \in \clist, D \cap G \subsetneq G_j (j\neq i)$.
Then there exists $A'\neq A$ in $\clist$ such that $A' \cap G \subsetneq G_i$ and $\dra(A, A') < \dra(A, D)$.
\item[(2)] Suppose $A \in \clist, A \cap G \subsetneq G_i$, and $D \in \clist$, $D \cap G$ is the union of good points in several optimal clusters. Then there exists $A'\neq A$ in $\clist$ such that $A' \cap G \subsetneq G_i$ and $\dra(A, A') < \dra(A, D)$.
\end{itemize}
\end{claim}
\mycomment{The proof of this claim is moved from the appendix to here.}
\begin{proof}
(1) It follows from the following three statements:
\begin{itemize}
\item[(a)] $\da(A \cap G, A' \cap G) < \frac{1}{2} \da(A \cap G, D \cap G)$;
\item[(b)] $\dra(A,A') \leq  \frac{7}{5}\da(A \cap G, A' \cap G);$
\item[(c)] $\frac{9}{10} \da(A \cap G, D \cap G) \leq  \dra(A,D).$
\end{itemize}

We now prove the statements respectively.

(a) For simplicity, let $G_A = A \cap G, G_D = D \cap G$. From Lemma~\ref{lem:goodSubsets}, we have
\begin{eqnarray*}
\da(G_A, C_i) \leq \gamma_{ji} \ \da(G_A,G_D), \mathrm{where} \ \ \gamma_{ji} = \frac{|C_j|}{(\beta-1/\beta)|C_i|} + \frac{1}{\beta^2-1}.
\end{eqnarray*}
Since $\ds(G_A, G_i \setminus G_A) \leq \ds(G_A, C_i)$, we have
\begin{eqnarray*}
\da(G_A, G_i \setminus G_A) \leq \frac{|C_i|}{|G_i \setminus G_A|}  \da(G_A, C_i) \leq \gamma_{ji} \frac{|C_i|}{|G_i \setminus G_A|} \da(G_A,G_D)
\leq \frac{1}{2} \da(G_A,G_D)
\end{eqnarray*}
where the last step follows from $\alpha \geq 6\frac{\max_i |C_i|}{\min_i |C_i|} + 2$, $|G_i\setminus A|$ is at least $\minSizeFrac\min_i |C_i| - m_B$.

(b) By Lemma~\ref{lem:FapproxG} and the fact that $|A|\geq \minSizeFrac\min_i |C_i|, |A'| \geq \minSizeFrac\min_i |C_i|$ and $\min_i |C_i| > 100 m_B$, we have
$$\ds(P(A), P(A')) \leq \frac{10}{9} \ds(P(A), A' \cap G) \leq \frac{100}{81} \ds(A \cap G, A' \cap G).$$
Then the claim follows from the fact that $|P(A)| \geq \frac{48}{50} |A|, | P(A')| \geq \frac{48}{50} |A'| $.

(c) For simplicity, let $G_A = A \cap G, G_D = D \cap G$. Divide $G_A$ into two parts: $W_A = G_A \cap P(A)$ and $X_A = G_A \cap F(A)$.
Define $W_D$ and $X_D$ similarly. See Figure~\ref{fig:AHF} for an illustration.

\begin{figure*}[h]
\centering

\begin{tikzpicture}
\usetikzlibrary{shapes,backgrounds}

\def\largeN{3}
\def\bbox{(-\largeN,-\largeN) rectangle (\largeN,\largeN)}

\begin{scope}[shift={(3cm,0cm)}]
\coordinate (center) at (0,0);
\def\A{(center) circle (1.2)}

\coordinate (F1) at (-1.2,1.6);
\coordinate (F2) at (2,-0.8);
\def\FP{(F1) -- (F2)}
\coordinate (F1int) at (F1 |- 0,\largeN);
\coordinate (F2int) at (F2 -| \largeN,0);
\def\FPup{\FP -- (F2int) -- (\largeN, \largeN) -- (F1int) --cycle}

\coordinate (B1) at (1.2,1.6);
\coordinate (B2) at (-2,-0.8);
\def\BG{(B1) -- (B2)}
\coordinate (B1int) at (B1 |- 0,\largeN);
\coordinate (B2int) at (B2 -| -\largeN,0);
\def\BGup{\BG -- (B2int) -- (-\largeN, \largeN) -- (B1int) --cycle}

\fill[blue!25] \A;
\node at (0,-0.5) {$W_D$};

\begin{scope}
	\clip \FPup;
	\fill[green!25] \A;
\end{scope}
\node at (0.6,0.6) {$X_D$};

\begin{scope}
	\clip \BGup;
	\fill[red!25] \A;
\end{scope}

\begin{scope}
	\clip \FPup;
	\clip \BGup;
	\fill[gray!25] \A;
\end{scope}

\node at (0,1.8) {$D$};

\draw \A;
\draw \FP;
\draw \BG;

\node at (1.6,-1.2) {$P(D)$};
\node at (2,-0.3) {$F(D)$};

\node at (-1.6,-1.2) {$G_D$};
\end{scope}

\begin{scope}[shift={(-3cm,0cm)}]
\coordinate (center) at (0,0);
\def\A{(center) circle (1.2)}

\coordinate (F1) at (-1.2,1.6);
\coordinate (F2) at (2,-0.8);
\def\FP{(F1) -- (F2)}
\coordinate (F1int) at (F1 |- 0,\largeN);
\coordinate (F2int) at (F2 -| \largeN,0);
\def\FPup{\FP -- (F2int) -- (\largeN, \largeN) -- (F1int) --cycle}

\coordinate (B1) at (1.2,1.6);
\coordinate (B2) at (-2,-0.8);
\def\BG{(B1) -- (B2)}
\coordinate (B1int) at (B1 |- 0,\largeN);
\coordinate (B2int) at (B2 -| -\largeN,0);
\def\BGup{\BG -- (B2int) -- (-\largeN, \largeN) -- (B1int) --cycle}

\fill[blue!25] \A;
\node at (0,-0.5) {$W_A$};

\begin{scope}
	\clip \FPup;
	\fill[green!25] \A;
\end{scope}
\node at (0.6,0.6) {$X_A$};

\begin{scope}
	\clip \BGup;
	\fill[red!25] \A;
\end{scope}

\begin{scope}
	\clip \FPup;
	\clip \BGup;
	\fill[gray!25] \A;
\end{scope}

\node at (0,1.8) {$A$};

\draw \A;
\draw \FP;
\draw \BG;

\node at (1.6,-1.2) {$P(A)$};
\node at (2,-0.3) {$F(A)$};
\node at (-1.6,-1.2) {$G_A$};
\end{scope}
\end{tikzpicture}
\caption{Illustration of the notations in Claim~\ref{cla:case1}.}\label{fig:AHF}
\end{figure*}
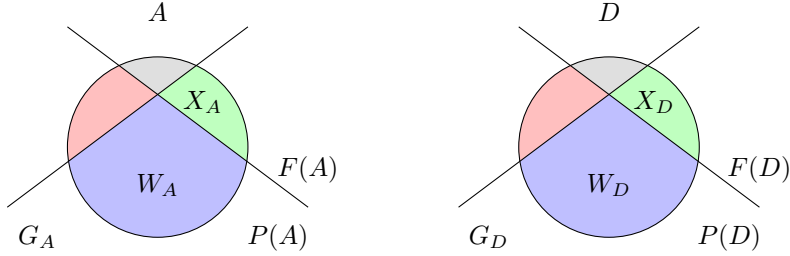

To show $\ds(G_A, G_D) \leq O(1) \ds(P(A), P(D))$,
it suffices to show $\ds(G_A, G_D) \leq O(1)\ds(W_A, W_D)$.
Since 
\begin{eqnarray*}
\ds(G_A,G_D)  &= & \ds(W_A, W_D) + \ds(G_A, X_D) + \ds(G_D, X_A) - d(X_A, X_D) \\
& \leq & \ds(W_A, W_D) + \ds(G_A, X_D) + \ds(G_D, X_A),
\end{eqnarray*}
we only need to bound $\ds(G_A,X_D)$ and $\ds(G_D,X_A)$.
By Fact~\ref{fac:tri} and Lemma~\ref{lem:goodSubsets} we have
$$\da(G_A,X_D) \leq \da(G_A, G_D) + \da(G_D, X_D) \leq (1+\frac{11}{50}) \da(G_A,G_D).$$
Since $|X_D| \leq 2m_B$ and $|D|\geq \frac{1}{2}\min_i |C_i|\geq 50 m_B$, we have
$$\ds(G_A,X_D) \leq \frac{61}{50}\frac{|X_D|}{|G_D|} \ds(G_A,G_D) \leq \frac{1}{20} \ds(G_A,G_D).$$
Similarly, $\ds(G_D,X_A) \leq \frac{1}{20}\ds(G_A,G_D)$. Therefore,
\begin{eqnarray*}
\ds(G_A,G_D) & = & \ds(W_A,W_D) + \ds(G_A, X_D) + \ds(G_D, X_A) - \ds(X_D, X_A) \\
& \leq & \ds(W_A,W_D) + \ds(G_A, X_D) + \ds(G_D, X_A) \\
& \leq  & \ds(W_A,W_D) + \frac{\ds(G_A,G_D)}{10} 
\end{eqnarray*}
which leads to $\frac{9}{10} \ds(G_A,G_D) \leq \ds(W_A,W_D) \leq \ds(P(A), P(D))$.
Then the claim follows from the fact that $|G_A|\geq |P(A)|, |G_D| \geq |P(D)| $.

(2) The proof idea is similar to that for Claim~\ref{cla:case1}.(1).
The only difference is the proof for 
$$\da(A \cap G, A' \cap G) < \frac{1}{2} \da(A \cap G, D \cap G).$$
Since $D \cap G = \cup_{j \in I_D} G_j$, it suffices to show that 
$$\da(A \cap G, A' \cap G) < \frac{1}{2} \da(A \cap G, G_j)$$ 
for any $j \in I_D$,
which can be proved by the same argument as in Claim~\ref{cla:case1}.(1).
\end{proof}

Applying the claim completes the proof of Lemma~\ref{thm:pruning}.
\end{proof}

\subsubsection{Getting A Pruning Close to the Optimal Clustering}\label{sec:pruning}
We now show that the pruning $\mathcal{C}'$ that assigns all good points correctly is the pruning with the minimum robust min-sum cost.

\begin{lemma}\label{thm:robustPruning}
Suppose the pruning $\mathcal{C}'=\{C'_1, \dots, C'_k\}$ in tree $\mathcal{T}$ assigns all good points correctly.
Then $\mathcal{C}'$ is the minimum robust min-sum cost pruning in the tree.
\end{lemma}

\begin{proof}
Computing the robust min-sum cost will eliminate the effect of the bad points and work as if we knew the actual good points:
the robust min-sum cost saved by splitting a node is at most the good point cost saved (Claim~\ref{lem:robustCostSaved}),
and the robust min-sum cost increased by merging two nodes is of the same order as the good point cost increased (Claim~\ref{lem:robustMergeCost} and Corollary~\ref{cor:multiRobustMergeCost}).


\begin{claim}\label{lem:robustCostSaved}
If $|C'_i| \geq 20m_B$, then $\drs(C'_i) \leq \ds(G_i, G_i)$.
\end{claim}

\begin{proof}
The claim follows from Claim~\ref{lem:farPointCost} (See Figure~\ref{fig:WVXY1} for an illustration of the notations) by setting $A=C'_i$ and $H = G_i$. In particular, we have $\drs(A) = \ds(P,P) \leq \ds(W,W)+ 2 \ds(Y,W \dcup Y)$
and $\ds(H,H) \geq \ds(W,W) + 2 \ds(X,W)$.
By Claim~\ref{lem:farPointCost}, $\ds(Y,W \dcup Y) \leq \ds(X,W)$, which completes the proof.
\end{proof}


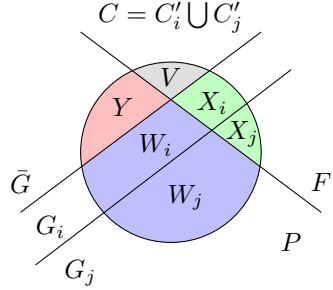
\begin{figure}[!ht]
\centering

\begin{tikzpicture}
\usetikzlibrary{shapes,backgrounds}

\def\largeN{2.2}
\def\bbox{(-\largeN,-\largeN) rectangle (\largeN,\largeN)}
\clip \bbox;

\coordinate (center) at (0,0);
\def\A{(center) circle (1.2)}

\coordinate (F1) at (-1.2,1.6);
\coordinate (F2) at (2,-0.8);
\def\FP{(F1) -- (F2)}
\coordinate (F1int) at (F1 |- 0,\largeN);
\coordinate (F2int) at (F2 -| \largeN,0);
\def\FPup{\FP -- (F2int) -- (\largeN, \largeN) -- (F1int) --cycle}

\coordinate (B1) at (1.2,1.6);
\coordinate (B2) at (-2,-0.8);
\def\BG{(B1) -- (B2)}
\coordinate (B1int) at (B1 |- 0,\largeN);
\coordinate (B2int) at (B2 -| -\largeN,0);
\def\BGup{\BG -- (B2int) -- (-\largeN, \largeN) -- (B1int) --cycle}

\coordinate (G1) at (1.6,1.1) {};
\coordinate (G2) at (-1.8,-1.5) {};
\def\GG{(G1) -- (G2)}
\coordinate (G1int) at (G1 |- 0,\largeN);
\coordinate (G2int) at (G2 -| -\largeN,0);
\def\GGup{\GG -- (G2int) -- (-\largeN, \largeN) -- (G1int) --cycle}

\fill[blue!25] \A;
\node at (0.2,-0.6) {$W_j$};
\node at (-0.2,0.1) {$W_i$};

\begin{scope}
	\clip \FPup;
	\fill[green!25] \A;
\end{scope}
\node at (0.55,0.65) {$X_i$};
\node at (0.95,0.25) {$X_j$};

\begin{scope}
	\clip \BGup;
	\fill[red!25] \A;
\end{scope}
\node at (-0.64,0.6) {$Y$};

\begin{scope}
	\clip \FPup;
	\clip \BGup;
	\fill[gray!25] \A;
\end{scope}
\node at (0,1) {$V$};

\node at (0,1.8) {$C=C'_i \bigcup C'_j$};

\draw \A;
\draw \FP;
\draw \BG;
\draw \GG;

\node at (1.6,-1.2) {$P$};
\node at (2,-0.4) {$F$};

\node at (-1.2,-1.6) {$G_j$};
\node at (-1.6,-1) {$G_i$};
\node at (-2,-0.4) {$\bar{G}$};
\end{tikzpicture}
\vspace{-.2in}
\caption{Notations in Claim~\ref{lem:robustMergeCost}.}\label{fig:WVXY2}
\end{figure}

\begin{claim}\label{lem:robustMergeCost}
For $t \in \{i,j\}$, $|C_t|\geq 100 m_B$,
and $C'_t $ contains all good points in $C_t$ but no good points in other optimal clusters.
Then
$
\drs(C'_i \cup C'_j) - \ds(G_i, G_i) - \ds(G_j, G_j) \geq (\frac{4}{3}-\frac{4}{\beta}) \ds(G_i, G_j).
$
\end{claim}

\begin{proof}
Let $C= C'_i \cup C'_j, F=F(C), P=P(C)$, and $G=G_i  \dcup  G_j, \bar{G}=C\setminus G$. Define $W_i = G_i \cap P, X_i = G_i \cap F$; define $W_j$, $X_j$ similarly.
Also, define $Y = P \cap \bar{G}, V=F \cap \bar{G}$. See Figure~\ref{fig:WVXY2}.
Then the left-hand side of the statement is
\begin{eqnarray}
&& \ds(P, P) - \ds(G_i, G_i) - \ds(G_j, G_j) \nonumber\\
& \geq & 2 (\ds(W_i, W_j) - \ds(X_i, G_i) - \ds(X_j, G_j)).\label{eqn:robustMergeCost}
\end{eqnarray}
By Lemma~\ref{lem:boundMergeCost}, we have $\ds(C_i,C_j) \leq \frac{3}{2}\ds(W_i,W_j).$
By the definition of good points,
$$
\ds(X_i, C_i) + \ds(X_j, C_j) \leq \frac{\ds(X_i, C_j) + \ds(X_j, C_i)}{\beta} \leq \frac{2}{\beta}\ds(C_i,C_j).
$$ 
Plugging these into (\ref{eqn:robustMergeCost}), we have that the left-hand side of the statement is at least $(\frac{4}{3}-\frac{4}{\beta}) \ds(C_i,C_j) \geq (\frac{4}{3}-\frac{4}{\beta}) \ds(G_i,G_j)$.
\end{proof}

The same argument as that for Claim~\ref{lem:robustMergeCost} leads to a corollary for the general case when multiple clusters are merged.
\begin{corollary}\label{cor:multiRobustMergeCost}
Let $I \subseteq [k]$.
Suppose for any $t \in I$, $|C_t|\geq 100 m_B$,
and $C'_t $ contains all good points in $C_t$ but no good points in other optimal clusters.
Then
$$\drs\left(\cup_{t\in I} C'_t \right) - \sum_{t\in I} \ds(G_t, G_t) \geq (\frac{4}{3}-\frac{4}{\beta}) \sum_{s \neq t \in I} \ds(G_t, G_s).$$
\end{corollary}

Besides these claims, another key property we need is that the good points in different optimal clusters are far apart in the sense that the good points from two different clusters have cost much larger than those in a third cluster have, as formalized in Lemma~\ref{lem:splitMerge}.  The proof of this lemma is technical  and not related to the other parts of the proof, so we defer it to Appendix~\ref{app:goodpointprp}.

\begin{lemma}\label{lem:splitMerge}
For any three different optimal clusters $C_i, C_j$, and $C_l$,
and any $A \subset G_i$, $\frac{18}{5}\ds(A, G_i\setminus A) < \ds(G_j, G_l).$
Consequently, $\frac{9}{5}\ds(G_i, G_i) < \ds(G_j, G_l).$
\end{lemma} 

Given the claims and Lemma~\ref{lem:splitMerge}, We are now ready to prove Lemma~\ref{thm:robustPruning}.

\mycomment{The proof of Lemma~\ref{thm:robustPruning} is moved from the appenix to here.}

First, by Lemma~\ref{lem:splitMerge}, good points from different clusters are far apart while good points in the same cluster are close.
Second, by Claim~\ref{lem:robustCostSaved} and Corollary~\ref{cor:multiRobustMergeCost},
the cost of good points can be approximated by the cost of the potentially good points (the robust min-sum cost).
We now use the above lemmas to show that $\mathcal{C}'$ has minimum robust min-sum cost,
so that we can use dynamic programming on the tree to get the pruning.

Suppose a pruning $\mathcal{P}$ is obtained by splitting $h$ clusters in $\mathcal{C}'$ and at the same time
joining some other clusters into $g$ unions.
Specifically, for $1\leq i \leq h$, split $C'_i$ into $m_i \geq 2$ clusters $S_{i,1},\dots, S_{i, m_i}$;
after that, merge $C'_{h + 1}, \dots, C'_{h + l_g}$ into $g$ unions,
that is, for $1\leq j \leq g$, $l_0 = 0$, merge $l_j - l_{j-1} \geq 2$ clusters $C'_{h+l_{j-1}+1},\dots,C'_{h+l_{j}}$
into a union $U_j$; the other clusters in $\mathcal{C}'$ remain the same in $\mathcal{P}$.
Since the number of clusters is still $k$, we have $\sum_{i} m_i - h = l_g - g$.

By Claim~\ref{lem:robustCostSaved}, the cost saved by splitting the $h$ clusters is
\begin{eqnarray}
\sum_{1\leq i \leq h} \drs(C'_i) - \sum_{1\leq i \leq h} \sum_{1\leq p \leq m_i} \drs(S_{i,p}) \leq \sum_{1\leq i \leq h} \drs(C'_i) \leq \sum_{1\leq i \leq h} \ds(G_i, G_i).\label{eqn:saveCost}
\end{eqnarray}
The cost increased by joining clusters is
\begin{eqnarray}
&&\sum_{1\leq j \leq g} \left( \drs(U_j) - \sum_{h+l_{j-1} < t  \leq h+l_{j}} \drs(C'_t) \right)\nonumber\\
& \geq & \sum_{1\leq j \leq g} \left(\drs(U_j) - \sum_{h+l_{j-1} < t  \leq h+l_{j}} \ds(G_t, G_t) \right) \nonumber\\
& \geq & \sum_{1\leq j \leq g} \left(\sum_{h+l_{j-1} < t  \neq s \leq h+l_{j}}(\frac{4}{3}-\frac{4}{\beta})\ds(G_t, G_s) \right)\label{eqn:increaseCost}
\end{eqnarray}
where the first inequality follows from Claim~\ref{lem:robustCostSaved}, and the second inequality follows from Corollary~\ref{cor:multiRobustMergeCost}.
To prove $\mathcal{C}'$ is the minimum cost pruning, we need to show that the saved cost (\ref{eqn:saveCost}) is less than the increased cost (\ref{eqn:increaseCost}).
Since by Lemma~\ref{lem:splitMerge}, each term in (\ref{eqn:increaseCost}) is larger than any term in (\ref{eqn:saveCost}),
it is sufficient to show that the number of the terms in (\ref{eqn:increaseCost}) is no less than the number of the terms in (\ref{eqn:saveCost}), that is $\sum_{1\leq j \leq g} {l_j - l_{j-1} \choose 2} \geq h.$
We have $\sum_{j} {l_j - l_{j-1} \choose 2} = \frac{1}{2}\sum_j (l_j - l_{j-1})(l_j - l_{j-1} - 1) \geq \sum_j (l_j - l_{j-1} - 1) = l_g - g$,
where the inequality is from $l_j - l_{j-1} \geq 2$.
Since $ m_i \geq 2$,  $l_g - g = \sum_{i-1}^h m_i - h \geq h$, which completes the proof.
\end{proof}

\subsubsection{Getting a Good Approximation}\label{sec:cost}
We now show that Algorithm~\ref{alg:minsum_app} outputs a good approximation.
We first prove that after reassignment all good points are still assigned correctly (Lemma~\ref{lem:goodPointsAssign}),
and then bound the cost.

\begin{figure*}[!h]
\centering

\begin{tikzpicture}
\usetikzlibrary{shapes,backgrounds}

\def\maxx{3}
\def\maxy{2}
\def\bbox{(-\maxx,-\maxy) rectangle (\maxx,\maxy)}
\clip \bbox;

\def\largeN{3}
\coordinate (center) at (0,0);
\def\A{(center) circle (1.2)}

\coordinate (F1) at (-1.2,1.6);
\coordinate (F2) at (2,-0.8);
\def\FP{(F1) -- (F2)}
\coordinate (F1int) at (F1 |- 0,\largeN);
\coordinate (F2int) at (F2 -| \largeN,0);
\def\FPup{\FP -- (F2int) -- (\largeN, \largeN) -- (F1int) --cycle}

\coordinate (B1) at (1.2,1.6);
\coordinate (B2) at (-2,-0.8);
\def\BG{(B1) -- (B2)}
\coordinate (B1int) at (B1 |- 0,\largeN);
\coordinate (B2int) at (B2 -| -\largeN,0);
\def\BGup{\BG -- (B2int) -- (-\largeN, \largeN) -- (B1int) --cycle}

\fill[blue!25] \A;
\node at (0,-0.5) {$W$};

\begin{scope}
	\clip \FPup;
	\fill[green!25] \A;
\end{scope}
\node at (0.6,0.6) {$X$};

\begin{scope}
	\clip \BGup;
	\fill[red!25] \A;
\end{scope}
\node at (-0.64,0.6) {$Y$};

\begin{scope}
	\clip \FPup;
	\clip \BGup;
	\fill[gray!25] \A;
\end{scope}
\node at (0,0.96) {$V$};

\node at (0,1.8) {$C'_i$};

\draw \A;
\draw \FP;
\draw \BG;

\node at (1.6,-1.2) {$P(C'_i)$};
\node at (2,-0.3) {$F(C'_i)$};

\node at (-1.6,-1.2) {$G_i$};

\end{tikzpicture}
        \caption{Notations in Lemma~\ref{lem:goodPointsAssign} and Claim~\ref{cla:costAD}.}\label{fig:WVXYi_app}
\end{figure*}
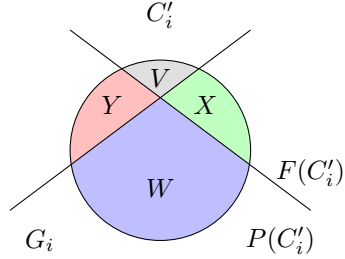

\begin{lemma}\label{lem:goodPointsAssign}
For any $p \in G_i$, any $j\neq i$, $\ds(p, P(C'_j)) > \ds(p, P(C'_i))$.
\end{lemma}

\mycomment{The proof of this lemma is moved from appendix to here}
\begin{proof}
Let $W_i = G_i \cap P(C'_i)$ denote the good points that are also potentially good points, and
let $Z_i = C_i \setminus W_i$ denote all other points in $C_i$.
See Figure~\ref{fig:WVXYi_app} for an illustration.
By Lemma~\ref{lem:FapproxG}, $\ds(p,P(C'_i)) \approx \ds(p,C_i)$.
By the definition of good points, $\beta \ds(p,C_i) \leq \ds(p,C_j)$.
So it suffices to show that $\ds(p, P(C'_j))$ is not so small compared to $\ds(p,C_j)$.
Since $W_j \subseteq P(C'_j)$, it suffices to prove that $\ds(p,W_j)$ is large compared to $\ds(p,Z_j)$.

First, by the triangle inequality, $\ds(p,Z_j) \leq \frac{|Z_j|}{|W_j|}\ds(p,W_j) + \frac{1}{|W_j|} \ds(Z_j,W_j)$.
Also, $\ds(Z_j,W_j) \leq \ds(C_j, W_j) \leq \frac{1}{\beta} \ds(C_i, W_j)$ by the definition of good points.
Furthermore, $\ds(C_i, W_j) \leq |W_j| \ds(p, C_i) + |C_i| \ds(p,W_j)$. So
\begin{eqnarray*}
\ds(p,Z_j) &\leq & \left(\frac{|Z_j|}{|W_j|} + \frac{|C_i|}{\beta|W_j|} \right) \ds(p,W_j) + \frac{1}{\beta} \ds(p,C_i) \\
&\leq & \left(\frac{|Z_j|}{|W_j|} + \frac{|C_i|}{\beta|W_j|} \right) \ds(p,W_j) + \frac{1}{\beta^2} (\ds(p,Z_j) +  \ds(p,W_j)).
\end{eqnarray*}
Therefore, we have $\ds(p,Z_j) \leq \frac{1}{3}\ds(p,W_j)$, since $|Z_j| \leq 4m_B, |W_j| \geq \frac{95}{100}|C_j| \geq 95 m_B$.
This leads to $\ds(p,W_j) \geq \frac{3}{4} \ds(p,C_j)$.

Then the lemma follows from $\ds(p, P(C'_j)) \geq \ds(p, W_j)$ and
$$\ds(p, W_j) \geq \frac{3\ds(p,C_j) }{4} \geq \frac{3\beta \ds(p,C_i)}{4} \geq \frac{3\beta \ds(p,G_i)}{4} \geq \frac{30\beta \ds(p,P(C'_i)) }{44}$$
where the last step follows from Lemma~\ref{lem:FapproxG}.
\end{proof}

We are now ready to prove our final result.
\begin{proofof}{Theorem~\ref{thm:alpha_espilon_minsum}}
By Lemma~\ref{lem:goodPointsAssign}, all the good points in $C_i$ are assigned correctly to $C''_i$.
Let $A_i = C''_i \setminus G_i$ denote all the bad points assigned to $C''_i$.
The cost of the output clustering $\mathcal{C''}$ can be written as follows.
\begin{eqnarray}
\sum_i \ds(C''_i, C''_i) & = & \sum_i \ds(G_i  \dcup  A_i, G_i  \dcup  A_i) \label{eqn:cost_decomp}\\
& = &\sum_i \ds(G_i,G_i) + 2 \sum_i \ds(G_i, A_i) + \sum_i \ds(A_i,A_i).\nonumber
\end{eqnarray}
We need to bound the last two terms.

Let $r = \frac{\min_i |C_i|}{m_B}$.
By the triangle inequality, we have $\da(A_i,A_i) \leq 2\da(A_i, G_i) $, leading to
\begin{eqnarray}
\ds(A_i,A_i) \leq  \frac{2 |A_i|}{|G_i|} \ds(A_i, G_i)  \leq \frac{2 m_B}{|C_i| - m_B} \ds(A_i, G_i)  \leq \frac{2}{r-5} \ds(A_i, G_i). \label{eqn:cost_aa}
\end{eqnarray}
So it suffices to bound $\ds(A_i, G_i)$. We have the following claim for this.

\begin{claim}\label{cla:costAD}
$\sum_i \ds(A_i,G_i) \leq \frac{r^2}{(r-5)^2} \sum_i \ds(C_i,C_i) - \frac{r^2-1}{(r-5)^2} \sum_i \ds(G_i,G_i)$.
\end{claim}
\mycomment{The proof of this claim is moved from appendix to here}
\begin{proof}
Let $W_i = P(C'_i) \cap G_i$. See Figure~\ref{fig:WVXYi_app} for an illustration.
By Fact~\ref{fac:tri}, 
\begin{eqnarray}
 \ds(A_i, G_i) &\leq & \frac{|G_i|}{|W_i|}  \ds(A_i, W_i) + \frac{|A_i|}{|W_i|}\ds(G_i, W_i)  \label{eqn:costAD1}\\
& \leq  & \frac{|G_i|}{|W_i|}  \ds(A_i, P(C'_i)) +  \frac{|A_i|}{|W_i|}\ds(G_i, G_i). \nonumber
\end{eqnarray}
So it suffices to bound $\ds(A_i, P(C'_i))$.
Fix $p \in A_i$, and suppose $p\in C_j$.
We have 

\begin{eqnarray*}
\ds(p, P(C'_i)) & \leq & \ds(p,P(C'_j)) \leq \frac{|W_j| + |Y_j|}{|W_j| - |X_j|} \ds(p, G_j) \\
& \leq & \frac{|C_j|}{|C_j| - 2|X_j| - |Y_j|} \ds(p, G_j) = \frac{r}{r-5} \ds(p, G_j)
\end{eqnarray*}
where the second step follows from Lemma~\ref{lem:FapproxG} and the last from $|X_j| \leq 2 m_B$ and $ |Y_j| \leq m_B$.
Then
\begin{eqnarray}
\sum_{i=1}^k  \ds(A_i, P(C'_i)) & \leq  & \frac{r}{r-5} \sum_ j \sum_{p \in (\cup_i A_i) \cap C_j} \ds(p, G_j) =  \frac{r}{r-5} \sum_{j=1}^k  \ds(B_j, G_j) \nonumber\\
& \leq & \frac{r}{r-5} \sum_{j=1}^k  (\ds(C_j, C_j) - \ds(G_j,G_j)).\label{eqn:costAD2}
\end{eqnarray}
The claim follows from the inequalities (\ref{eqn:costAD1}), (\ref{eqn:costAD2}) and $|X_i| \leq 2 m_B, |A_i| \leq m_B$.
\end{proof}

The proof of correctness is completed by combining Claim~\ref{cla:costAD},  (\ref{eqn:cost_decomp}),  and (\ref{eqn:cost_aa}).

\paragraph{Running Time}
Algorithm~\ref{alg:aekmedian_blobs} takes time $O(n^{\omega+1})$ (as shown in the proof of Theorem~\ref{thm:aekmedian}),
and the rest steps of Algorithm~\ref{alg:raLinkage} take time $O(n^3)$.
Finding the minimum robust min-sum cost pruning in the tree output by Algorithm~\ref{alg:raLinkage} takes time $O(n^3)$,
and Algorithm~\ref{alg:minsum_app} takes time $O(n^3)$.
So the total running time is $O(n^{\omega+1})$.
\end{proofof}

\section{Discussion and Open Questions}

We advance the line of research on clustering under
perturbation resilience in multiple ways. For $\alpha$-perturbation resilient instances,
we improve on the known guarantees for center-based objectives and
give the first analysis for min-sum.
Furthermore, for $k$-median and min-sum, we analyze and give the first algorithmic guarantees
known for a relaxed but more challenging condition of $(\alpha,
\epsilon)$-perturbation resilience, where an $\epsilon$ fraction of
points are allowed to move after perturbation.
We also give sublinear-time algorithms for $k$-median and min-sum under perturbation resilience.

A natural direction for future investigation is to explore whether one can take advantage of smaller perturbation factors for
perturbation resilient instances in Euclidian spaces\footnote{That is, where $d$ is a
Euclidean metric, though as in Definitions~\ref{def:alphaPR} and~\ref{def:alphaPRr}, $d'$ need not be.
Alternatively, one could also consider a natural version of Definitions~\ref{def:alphaPR} and~\ref{def:alphaPRr} in which $d'$ must be Euclidean as well, and in fact implemented via
a perturbation of coordinate values.}. More broadly, it would be interesting to
explore other ways in which perturbation resilient instances behave better than worst case instances (e.g., natural algorithms converge faster).

Another interesting direction is to design clustering algorithms under perturbation resilience \yingyu{whose output satisfies certain privacy requirements.}
For example, some stability notions can be useful for differential private analysis~\cite{nissim2007smooth,dwork2009differential}.
It would be interesting to explore the perturbation resilience property and design efficient clustering algorithms that preserve differential privacy.


\subsection*{Acknowledgments}
We  thank  Avrim  Blum  for  numerous  useful  discussions.
This  work  was  supported  in  part  by  NSF  grants  CCF-0953192,
CCF-1101283, CCF-1451177, ONR  N00014-09-1-0751,  AFOSR  grant
FA9550-09-1-0538,  a  Microsoft  Faculty  Fellowship, a Google Research
Award, and a Sloan Fellowship.

\bibliographystyle{abbrv}
\bibliography{paper}

\appendix

\section{Finding the Minimum Cost $k$-Cluster Pruning}\label{subsec:dp}

The idea of using dynamic programming to find the optimal $k$-clustering in a tree of clusters is proposed
in~\cite{ABS10}. We can find the optimal clustering by examining the entire tree of clusters produced.

\yingyu{First recall our setting. Suppose we have a tree whose leaves are the data points. Each internal node of the tree represents a cluster that contains all points in the clusters represented by its children. Also suppose that the clustering objective is separable:  (1) the objective function value of a given clustering is
either a (weighted) sum or the maximum of the individual cluster scores;
(2) given a proposed single cluster, its score can be computed
in polynomial time.  Our goal is to find a pruning of the tree that has $k$ clusters and has minimum cost.}

\yingyu{We first consider the case when each node of the tree has at most $2$ children.}
Denote the cost of the optimal $m$-clustering of a tree node $p$ as $\cost(p,m)$.
The optimal $m$-clustering of a tree node $p$ is either the entire subtree as one cluster ($m = 1$),
or the minimum over all choices of $m_1$-clustering over its left subtree and $m_2$-clustering over
its right subtree ($1< m \leq k$), where $m_1, m_2$ are positive integers such that
$m_1 + m_2 = m$.
Therefore, we can traverse the tree bottom up, recursively solving the $m$-clustering problem
for $1\leq m\leq k$ for each tree node. The algorithm is presented in
Algorithm~\ref{DynamicProgramming}.
Suppose that computing the cost of a cluster takes time $O(t)$ ($O(n^2)$ for $k$-median, $k$-means and min-sum).
Since there are $O(n)$ nodes, and on each node $p$, computing $\cost(p, 1)$ takes $O(t)$ time,
computing $\cost(p, m)(1 < m \leq k)$ takes $O(k^2)$, in total the algorithm takes time $O(nt+nk^2)$.

Note that when $\mathcal{T}$ is a multi-branch tree and not suitable for dynamic programming, we need to turn it into a 2-branch tree $\mathcal{T}'$ as follows.
For each node with more than 2 children, for example, the node $R$ with children $R_1, R_2, \dots, R_t(t>2)$, we first
merge $R_1$ and $R_2$ into one node, then merge this node with $R_3$; repeat until we merge all nodes $R_1, R_2, \cdots, R_{t}$ into $R$.
In this way, we get a 2-branch tree $\mathcal{T}'$ and can run dynamic programming on it.
Note that each pruning in $\mathcal{T}$ has a corresponding pruning in $\mathcal{T}'$,
so the minimum cost pruning of $\mathcal{T}'$ has no greater cost than the minimum cost pruning of $\mathcal{T}$.
Also note that when the cost function is center-based, such as $k$-median,
the algorithm essentially computes a center for
the node $p$ when computing $\cost(p, 1)$.
So it can output the centers together with the pruning.

\begin{algorithm}[tbhp]
\caption{Dynamic Programming in Tree of Clusters}
\label{DynamicProgramming}
\begin{algorithmic}[1]
\REQUIRE{A tree of clusters $\mathcal{T}$ on a data set $\data$, distance function $d(\cdot, \cdot)$ on $\data$, $k$.}
\STATE{Traverse $\mathcal{T}$ bottom up.}
\FOR{each node $R \in \mathcal{T}$}
\STATE{Calculate $\cost(R, 1)$. For $1 < m \leq k$, calculate $\cost(R,m)$ as follows.}
\IF{$R$ is a leaf}
\STATE{$\cost(R,m) = \cost(R,1)$.}
\ELSE
\STATE{$\cost(R,m)= \min\{\cost(R_1,m_1) + \cost(R_2, m_2)\}$, where $R_1, R_2$ are $R$'s children, $m_1 + m_2 = m$, and the minimum is taken over all possible $R_1, R_2, m_1$, and $m_2$.}
\ENDIF
\ENDFOR
\STATE{Traverse backwards to get the $k$-clustering $\mathcal{C}$ that achieves $\cost(r, k)$ where $r$ is the root.}
\ENSURE{The $k$-clustering $\mathcal{C}$.}
\end{algorithmic}
\end{algorithm}

\section{An Efficient Implementation of Algorithm~\ref{ClosureLinkage}}\label{subsec:running_time}

Here we show an efficient implementation of Algorithm~\ref{ClosureLinkage}, namely Algorithm~\ref{alg:EffImpClosureLink}.
This implementation takes time only $O(n^3)$.

\begin{algorithm}[!t]
\caption{Efficient Implementation of Algorithm~\ref{ClosureLinkage}}
\label{alg:EffImpClosureLink}
\begin{algorithmic}[1]
\REQUIRE{Data set $\data$, distance function $d(\cdot, \cdot)$ on $\data$.}
\STATE{Sort all the pairwise distances in ascending order.}
\FOR{each $p\in \data$ and $1\leq i\leq n$}
\STATE{Compute $L^p$, $\chi(p,i)$ according to Definition~\ref{def:list}. Then compute $\chi^*(p,i)$ by Equation~(\ref{eqn:update_checking}).}
\ENDFOR
\STATE{Let the current clustering be $n$ singleton clusters.}
\FOR{$d(p,q)$ in ascending order}
\STATE{Suppose $q = L^p_i$. Check if $d(p, q)$ satisfies the three claims in Fact~\ref{lemma:efficientImpl}, where the third claim can be checked by verifying if $\chi^*(p,i)=-1$.}
\STATE{If so, merge all the clusters covered by $\mathbb{B}(p, d(p,q))$.}
\ENDFOR
\STATE{Construct the tree $\mathcal{T}$ with points as leaves and internal nodes corresponding to the merges performed.}
\STATE{Run dynamic programming on $\mathcal{T}$ to get the minimum cost pruning ${\cal \tilde{C}}$.}
\ENSURE{The clustering ${\cal \tilde{C}}$.}
\end{algorithmic}
\end{algorithm}

Note that at each merge step in Algorithm~\ref{ClosureLinkage}, we only need to find the two clusters with the minimum closure distance.
So we hope to compute the minimum closure distance without computing all the distances
between any two current clusters.
First we notice the following facts.
\begin{fact}\label{lemma:minClosureDistanceProperty}
In the execution of Algorithm~\ref{ClosureLinkage}, if $d$ is the minimum closure distance for the current
clustering, then 
\begin{itemize}
\item[(1)] there exist $c, p \in S$ such that $d = d(c, p)$;
\item[(2)] $d$ is no less than the minimum closure distances in previous clusterings.
\end{itemize}
\end{fact}
\begin{proof}
For the first claim, let $c$ be the center of the ball in the definition of closure distance,
and $p$ be the farthest point from the center in the ball, then $d = d(c,p)$.
The second claim comes from the fact that the clusters in the current clustering are supersets of
those in previous clusterings.
\end{proof}

Fact~\ref{lemma:minClosureDistanceProperty} implies that we can check in ascending order the pairwise distances no less
than the minimum closure distance in the last clustering, and determine if the checked pairwise distance is
the minimum closure distance in the current clustering.
More specifically, suppose we have some black-box method for checking if a pairwise distance
is the minimum closure distance in the current clustering,
we can perform the closure linkage as follows:
sort the pairwise distances in a list in ascending order;
start from the first distance in the list;
check if the current distance is the minimum closure distance in the current clustering;
if it is, merge clusters covered by the ball defined by the checked distance;
continue to check the next distance in the list.
So it is sufficient to design a method to determine if a pairwise distance
is the minimum closure distance in the current clustering.
Our method is based on the following facts.

\begin{fact}\label{lemma:efficientImpl}
In Algorithm~\ref{ClosureLinkage}, if $d(c,p)$ is the minimum closure distance for the current
clustering, then
\begin{itemize}
\item[(1)] at least 2 clusters intersect $\mathbb{B}(c,d(c,p))$;
\item[(2)] all the clusters intersecting $\mathbb{B}(c,d(c,p))$ are covered by $\mathbb{B}(c,d(c,p))$;
\item[(3)] for any $p' \in \mathbb{B}(c,d(c,p)),q\not\in \mathbb{B}(c,d(c,p)),d(c,p')<d(p',q)$.
\end{itemize}
\end{fact}

\begin{proof}
The first claim and the third claim follow from the definition.
We can prove the second claim by induction.
This is trivial at the beginning.
Suppose it is true up to any previous
clustering, we prove it for the current clustering $\mathcal{C'}$.
We need to show that for any $C' \in \mathcal{C'}$ such that $C' \cap \mathbb{B}(c,d(c,p)) \neq \varnothing$,
$C' \subseteq \mathbb{B}(c,d(c,p))$.
If $c \in C'$, then by definition, $C' \subseteq \mathbb{B}(c,d(c,p))$.
If $C'$ is a single point set $\{c_1\}$, then trivially $C' \subseteq \mathbb{B}(c,d(c,p))$.
What is left is the case when $c \not\in C'$ and $C'$ is generated by merging clusters in a previous step.
Suppose when $C'$ is formed, the closure distance between
those clusters is defined by $c_1\in C'$ and $p_1$.
By induction, if $c \in \mathbb{B}(c_1,d(c_1,p_1))$, $c$ would have been
merged into $C'$ when $C'$ is merged, which is contradictory to $c \not\in C'$.
So we have $c \not\in \mathbb{B}(c_1,d(c_1,p_1))$, that is, $d(c,c_1) > d(c_1,p_1)$.
Then by the margin requirement of $\mathbb{B}(c_1, d(c_1,p_1))$, $d(c, q) > d(c_1, q)$ for any $q \in \mathbb{B}(c, d(c,p)) \cap C'$.
This further leads to $c_1 \in \mathbb{B}(c, d(c,p))$, since otherwise
by the margin requirement of $\mathbb{B}(c, d(c,p))$ and $q \in \mathbb{B}(c, d(c,p))$, we would have $d(c, q) < d(c_1, q)$.
So for any point $q'  \in C'$, since $d(c_1, q') \leq d(c_1, p_1) < d(c, c_1)$,
we have $q' \in \mathbb{B}(c, d(c,p))$ from the margin requirement, so $C' \subseteq \mathbb{B}(c,d(c,q))$.
\end{proof}

Notice if a pairwise distance satisfies the three claims, then it defines a closure distance for the clusters covered.
So if we check the pairwise distances in ascending order, then the first one that satisfies
the three claims must be the minimum closure distance in the current clustering.
So we have a method to determine if a pairwise distance is  the minimum closure distance.

However, naively checking the third claim in Fact~\ref{lemma:efficientImpl} takes $O(n^2)$,
which is still not good enough. We can refine this step since intuitively, for every $c$,
if $d(c,q)$ comes after $d(c,p)$ in the distance list, then when checking $d(c,q)$,
we can utilize the information obtained from checking $d(c,p)$. To do so, we introduce some notations.

\begin{definition} \label{def:list}
\begin{itemize}
\item[(1)] For every $p\in S$, define $L^p=(L^p_1,\dots,L^p_n)$ to be a sorted list of points in $S$,
according to their distances to $p$ in ascending order. 
\item[(2)] Define $\chi^*(p,i)$ to be the maximum $j>i$ such that
there exits $s\leq i$ satisfying $d(p,L^p_s) \geq d(L^p_s, L^p_j)$; if no such point $L^p_j$ exists, let $\chi^*(p,i) = -1$.
\item[(3)] Define $\chi(p,i)$ to be the maximum $j>i$ such that
$d(p,L^p_i) \geq d(L^p_i,L^p_j)$; if no such $j$ exists, let $\chi(p,i) = -1$.
\end{itemize}
\end{definition}

Intuitively, $\chi^*(p,i)$ is the index of the
farthest point in $L^p$, which makes $d(p,L^p_i)$ fail the third claim in Fact~\ref{lemma:efficientImpl}.
Then $d(p,L^p_i)$ satisfies the third claim if and only if $\chi^*(p,i) = -1$, thus we turn the task of
checking the claim into computing $\chi^*(p,i)$.
In order to use the information obtained when previously checking $d(p, L^p_{i-1})$,
we compute $\chi^*(p,i)$ from $\chi^*(p,i-1)$.
By the definition of $\chi^*$, 
$\chi^*(p,i)$ is either the maximum $j > i$ such that there exits $s \leq i - 1$ satisfying $d(p,L^p_s) \geq d(L^p_s, L^p_j)$,
or the maximum $j > i$ there exits $s = i$ satisfying $d(p,L^p_s) \geq d(L^p_s, L^p_j)$.
Then it is easy to verify that
\begin{eqnarray}\label{eqn:update_checking}
\chi^*(p,i) = 
\begin{cases}
\chi(p,i) & \text{if } \chi^*(p,i-1) = i, \\
\max\{\chi^*(p,i-1),\chi(p,i)\} & \text{otherwise.}
\end{cases}
\end{eqnarray}
It takes $O(n)$ time to compute $\chi(p,i)$, thus we can compute $\chi^*(p,i)$ for all
$p\in S,1\leq i\leq n$ in $O(n^3)$ time.
The implementation is finally summarized in Algorithm~\ref{alg:EffImpClosureLink}.

%

\section{$(\alpha,\epsilon)$-Perturbation Resilient Min-Sum Instances}\label{app:gpProperties}\label{app:pbpProperties}

\subsection{Proofs for Bounding the Number of Bad Points for Min-Sum}\label{app:bound}


First, recall the definitions of the bad points and the perturbation constructed to bound the number of bad points in Section~\ref{sec:bound}.
Assume for contradiction that $|B| > 2\eta\epsilon n$.
\yingyu{Consider the following $\eta$ intervals:} $[2^{t-1} v, 2^t v]$ where $v=\min_i \min_{p\in B_i} d(p,C_i)$ where $1\leq t \leq \eta$.
At least one of the intervals, say $[r,2r]$, will contain the costs of more than $2\epsilon n$ bad points.
Let $\hat{B}$ denote an arbitrary subset of $2\epsilon n$ bad points in this interval.
Let $\hat{B}_i = \hat{B} \cap C_i$ denote the selected bad points in the optimal cluster $C_i$. Let $K_i = C_i \setminus \hat{B}_i$ denote the other points in $C_i$, and set $K = \cup_i K_i$.
Denote as $D_i$ all those selected bad points whose second nearest cluster is $C_i$, that is, $D_i = \{p: \exists j~\textrm{such that}~p \in \hat{B}_j~\textrm{and}~i = \arg\min_{\ell\neq j}d(p, C_{\ell})\}$. Note that by definition we have $\cup_i D_i = \hat{B}$. Finally, let $\tilde{C}_i = K_i \cup D_i$. See Figure~\ref{fig:BCD} for an illustration. 


The perturbation is constructed as follows: blow up all distances by a factor of $\alpha$
except those within $\tilde{C}_i, 1\leq i\leq k$.
That is,
\begin{eqnarray*}
d'(p,q) = \left\{ \begin{array}{ll}
d(p,q) & \textrm{if $p \in \tilde{C}_i$, and $q \in \tilde{C}_i$ for some $i$,}\\
\alpha d(p,q) & \textrm{otherwise}.
\end{array} \right.
\end{eqnarray*}

Let $\{C'_i\}$ denote the optimal clustering after perturbation.
Recall the definitions of $U_i, V_i, W_i$ and $\tilde{U}_i, \tilde{V}_i, \tilde{W}_i$, and see Figure~\ref{fig:UVW} for an illustration.
The following facts come from their definitions.

\begin{fact}\label{fac:UVW}
We have $\cup_i U_i = \cup_i \tilde{U}_i$, $\cup_i V_i = \cup_i \tilde{V}_i$ and $\cup_i W_i = \cup_i \tilde{W}_i$. Furthermore,
\begin{eqnarray*}
\sum_i \ds(\tilde{U}_i, C_i) &\leq & \beta \sum_i  \ds(U_i,C_i)\\
\sum_i \ds(\tilde{V}_i, C_i) &\leq & \sum_i \ds(V_i, C_i),\\
\sum_i \ds(\tilde{W}_i, C_i) & \leq & \sum_i \ds(W_i, C_i).
\end{eqnarray*}
\end{fact}

We are ready to prove the claim needed for bounding the number of bad points.

\textsc{Claim~\ref{cla:U}.(1).~}{\it
The costs saved and added by moving $\{ U_i, 1\leq i\leq k\}$ and $\{ \tilde{U}_i, 1\leq i\leq k\}$ satisfy
\begin{eqnarray*}
\Delta_U - \Delta_{\tilde{U}}  & \geq & 2\sum_i \dps(U_i, C'_i \cap K_i) - 2 \sum_i \dps(\tilde{U}_i, \tilde{C}_i) \\&\geq &\frac{3}{10}\alpha\sum_i \ds(U_i, C_i) - \frac{2\alpha}{100} \sum _i \ds(W_i,C_i) - \frac{8\alpha + 16}{100}r \epsilon n.
\end{eqnarray*}
}

\begin{proof}
Intuitively, we have that $\sum_i \dps(U_i, C'_i \cap K_i) \approx \alpha \sum_i \ds(U_i, C_i)$.
Similarly, $\sum_i \dps(\tilde{U}_i, \tilde{C}_i) \approx \sum_i \ds(\tilde{U}_i, C_i)$.
Their difference is then
roughly $(\alpha-\beta)  \sum_i \ds(U_i, C_i)$, since $\sum_i \ds(\tilde{U}_i, C_i) \leq \beta \sum_i  \ds(U_i,C_i)$.

Formally, we have
\begin{eqnarray}
\dps(U_i, C'_i \cap K_i) = \alpha \ds(U_i, C'_i \cap K_i) = \alpha \ds(U_i, C_i) -\alpha \ds(U_i, \tilde{W}_i + \hat{B}_i),\label{eqn:claimU1}\\
\dps(\tilde{U}_i, \tilde{C}_i) = \ds(\tilde{U}_i, \tilde{C}_i)  = \ds(\tilde{U}_i, K_i) + \ds(\tilde{U}_i, D_i) \leq \ds(\tilde{U}_i, C_i) + \ds(\tilde{U}_i, D_i).\label{eqn:claimU2}
\end{eqnarray}

Then it suffices to bound the approximation error $\ds(U_i, \tilde{W}_i  \dcup  \hat{B}_i)$ and $\ds(\tilde{U}_i, D_i)$.
First, for $\ds(U_i, \tilde{W}_i  \dcup  \hat{B}_i)$ we have
\begin{eqnarray*}
\ds(U_i, \tilde{W}_i + \hat{B}_i)
 & \leq  &\frac{|\tilde{W}_i  \dcup  \hat{B}_i|}{|C_i|} \ds(U_i, C_i) + \frac{|U_i|}{|C_i|} \ds(C_i, \tilde{W}_i \dcup \hat{B}_i)\\
 &\leq & \frac{3}{100} \ds(U_i, C_i) + \frac{1}{100} \ds(C_i, \tilde{W}_i \dcup \hat{B}_i)
\end{eqnarray*}
where the first inequality is by Fact~\ref{fac:tri}, and the second is from the fact that $|\tilde{W}_i| \leq \epsilon n, |\hat{B}_i| \leq 2 \epsilon n, |U_i| \leq \epsilon n$ and $|C_i| \geq 100 \epsilon n$.
For the second term on the right-hand side, we have $\sum_i \ds(C_i, \tilde{W}_i)\leq \sum_i \ds(W_i,C_i)$,
Furthermore, the points in $\hat{B}_i$ has cost at most $2r$ and $\sum_i |\hat{B}_i| \leq 2\epsilon n$.
So \begin{eqnarray*}
\sum_i \ds(U_i, \tilde{W}_i \dcup \hat{B}_i) \leq \frac{3}{100} \sum_i  \ds(U_i, C_i) + \frac{1}{100} \sum_i  \ds(C_i, W_i) + \frac{4 r \epsilon n}{100}.
\end{eqnarray*}
Similarly, for $\ds(\tilde{U}_i, D_i)$ we have
\begin{eqnarray*}
\sum_i \ds(\tilde{U}_i, D_i)  \leq  \sum_i  \left(\frac{|D_i|}{|C_i|} \ds(\tilde{U}_i, C_i) +  \frac{|\tilde{U}_i|}{|C_i|}\ds(C_i, D_i)\right) \leq \frac{2\beta}{100}\sum_i \ds(U_i, C_i) +  \frac{8r \epsilon n}{100}.
\end{eqnarray*}

The claim follows by summing (\ref{eqn:claimU1}) and (\ref{eqn:claimU2}) over $1\leq i\leq k$ and plugging in the last two inequalities.
\end{proof}

\textsc{Claim~\ref{cla:V}.(2).~}{\it
The costs saved and added by moving $\{ V_i, 1\leq i\leq k\}$ and $\{ \tilde{V}_i, 1\leq i\leq k\}$ satisfy
\begin{eqnarray*}
\Delta_V - \Delta_{\tilde{V}}  & \geq & 2\sum_i \dps(V_i, C'_i \cap C_i) - 2\sum_i \dps(\tilde{V}_i, \tilde{C}_i)\\
& \geq & \frac{99}{50}(\alpha  -2) \sum_i \ds(V_i, C_i) - \frac{2\alpha}{100}\sum_i \ds(W_i, C_i) - \frac{4\alpha + 8\beta}{100}r\epsilon n.
\end{eqnarray*}
}

\begin{proof}
The intuition is similar to that of Claim~\ref{cla:U}.(a): $\sum_i \dps(V_i, C'_i \cap C_i) \approx \alpha\sum_i \ds(V_i, C_i),\sum_i \dps(\tilde{V}_i, \tilde{C}_i) \approx \sum_i \ds(\tilde{V}_i, C_i)$.
Since $\sum_i \ds(V_i, C_i) \geq \sum_i \ds(\tilde{V}_i, C_i)$, their difference is roughly $(\alpha-1)\sum_i \ds(V_i, C_i)$.

Formally, we have
\begin{eqnarray}
 \dps(V_i, C'_i \cap C_i) & = & \alpha \ds(V_i, C'_i \cap C_i) = \alpha \ds(V_i, C_i) - \alpha  \ds(V_i, C_i \setminus C'_i),\label{eqn:claimV1}\\
 \dps(\tilde{V}_i, \tilde{C}_i) & = &  \ds(\tilde{V}_i, \tilde{C}_i) \leq  \ds(\tilde{V}_i, C_i) +  \ds(\tilde{V}_i, D_i).\label{eqn:claimV2}
\end{eqnarray}
Then it suffices to bound the approximation error $\ds(V_i, C_i \setminus C'_i)$ and $\ds(\tilde{V}_i, D_i)$.
First,
\begin{eqnarray*}
\ds(V_i, C_i \setminus C'_i)
& \leq &\frac{|C_i \setminus C'_i|}{|C_i|} \ds(V_i, C_i)+  \frac{|V_i|}{|C_i|} \ds(C_i \setminus C'_i, C_i )\\
&\leq &\frac{1}{100} \ds(V_i, C_i)+ \frac{1}{100} \ds(C_i \setminus C'_i, C_i )
\end{eqnarray*}
where the first inequality is by Fact~\ref{fac:tri} and the second is from the fact that $|C_i \setminus C'_i| \leq \epsilon n, |V_i| \leq \epsilon n$
and $|C_i| \geq 100 \epsilon n$.
For the second term on the right-hand side, we have $\ds(C_i \setminus C'_i, C_i ) \leq \ds(\tilde{W}_i, C_i ) + \ds(\hat{B}_i \setminus C'_i, C_i )$.
Note that $\sum_i \ds(\tilde{W}_i, C_i ) \leq \sum_i \ds(W_i, C_i)$. Furthermore, the points in $\hat{B}_i$ have cost at most $2r$ and $\sum_i |\hat{B}_i \setminus C'_i| \leq \epsilon n$ by perturbation resilience.
So
$$\sum_i \ds(V_i, C_i \setminus C'_i) \leq \frac{1}{100}\sum_i \ds(V_i, C_i) + \frac{1}{100}\sum_i \ds(W_i, C_i) + \frac{2r\epsilon n}{100}.$$
Similarly, for $\ds(\tilde{V}_i, D_i)$ we have
\begin{eqnarray*}
\sum_i \ds(\tilde{V}_i, D_i) & \leq & \sum_i \left( \frac{|D_i|}{|C_i|} \ds(\tilde{V}_i, C_i) + \frac{|\tilde{V}_i|}{|C_i|} \ds(C_i, D_i)\right) \leq \frac{2}{100}\sum_i  \ds(V_i, C_i) + \frac{4\beta r\epsilon n}{100}.
\end{eqnarray*}

The claim follows by summing (\ref{eqn:claimV1}) and (\ref{eqn:claimV2}) over $1\leq i\leq k$ and plugging the last two inequalities.
\end{proof}

\textsc{Claim~\ref{cla:V}.(3).~}{\it
The costs saved and added by moving $\{ W_i, 1\leq i\leq k\}$ and $\{ \tilde{W}_i, 1\leq i\leq k\}$ satisfy
\begin{eqnarray*}
\Delta_W - \Delta_{\tilde{W}}  & \geq & 2\sum _i \dps(W_i, C'_i \cap C_i) - 2\sum_i \dps(\tilde{W}_i, \tilde{W}_i \dcup (C'_i \cap \tilde{C}_i))\\
&\geq & \frac{98}{50}(\alpha  -2) \sum_i \ds(W_i, C_i)  - \frac{4\alpha+4\beta }{100}r\epsilon n.
\end{eqnarray*}
}

\begin{proof}
The intuition is similar to that of Claim~\ref{cla:U}.(a):
$\sum_i \dps(W_i, C'_i \cap C_i) \approx \alpha  \sum_i \ds(W_i, C_i)$ and $\sum_i \dps(\tilde{W}_i, \tilde{W}_i \dcup (C'_i \cap \tilde{C}_i) ) \approx \sum_i \ds(\tilde{W}_i, C_i)$.
Since $\sum_i \ds(W_i, C_i) \geq \sum_i \ds(\tilde{W}_i, C_i)$, their difference is roughly $(\alpha-1) \sum_i \ds(W_i, C_i)$.

Formally, we have
\begin{eqnarray}
\dps(W_i, C'_i \cap C_i)  &=&  \alpha  \ds(W_i, C'_i \cap C_i) = \alpha  \ds(W_i, C_i) - \alpha  \ds(W_i,  C_i \setminus C'_i), \label{eqn:claimW1}
\end{eqnarray}
\begin{eqnarray}
\dps(\tilde{W}_i, \tilde{W}_i \dcup (C'_i \cap \tilde{C}_i))  & =&  \ds(\tilde{W}_i, (C'_i \cap D_i) \dcup (C_i \cap \tilde{C}_i)) \label{eqn:claimW2} \\
& \leq & \ds(\tilde{W}_i, C'_i \cap D_i) + \ds(\tilde{W}_i, C_i ). \nonumber
\end{eqnarray}
Then it suffices to bound the approximation error $\ds(W_i,  C_i \setminus C'_i)$ and $\ds(\tilde{W}_i, C'_i \cap D_i)$.
First, for $\ds(W_i,  C_i \setminus C'_i)$ we have
\begin{eqnarray*}
\ds(W_i,  C_i \setminus C'_i)
& \leq   & \frac{| C_i \setminus C'_i|}{|C_i|} \ds(W_i, C_i) + \frac{|W_i|}{|C_i|} \ds( C_i \setminus C'_i, C_i)\\
& \leq   & \frac{1}{100} \ds(W_i, C_i) + \frac{1}{100} \ds( C_i \setminus C'_i, C_i)
\end{eqnarray*}
where the first inequality is by Fact~\ref{fac:tri} and the second from the fact that $| C_i \setminus C'_i| \leq \epsilon n, |W_i| \leq \epsilon n$
and $|C_i| \geq 100 \epsilon n$. For the second term on the right-hand side,
we have $\ds( C_i \setminus C'_i, C_i) = \ds(\tilde{W_i}, C_i)  + \ds( \hat{B}_i \setminus C'_i, C_i)$.
Note that
$\sum_i  \ds(\tilde{W_i}, C_i) \leq \sum_i  \ds(W_i, C_i)$. Furthermore, the points in $\hat{B}_i$ have cost at most $2r$
and $\sum_i |\hat{B}_i \setminus C'_i| \leq \epsilon n$ by perturbation resilience.
So
\begin{eqnarray*}
\sum_i  \ds(W_i,  C_i \setminus C'_i) \leq \frac{2}{100}\sum_i \ds(W_i, C_i) + \frac{2r \epsilon n}{100}.
\end{eqnarray*}
Similarly, for $ \ds(\tilde{W}_i, C'_i \cap D_i)$ we have
\begin{eqnarray*}
\sum_i \ds(\tilde{W}_i, C'_i \cap D_i)
& \leq  & \sum_i \left( \frac{|C'_i \cap D_i|}{|C_i|} \ds(\tilde{W}_i, C_i)  + \frac{|\tilde{W}_i|}{|C_i|}\ds(C'_i \cap D_i, C_i)  \right) \\
& \leq  & \frac{1}{100}\sum_i \ds(W_i, C_i) + \frac{2\beta r\epsilon n}{100}.
\end{eqnarray*}

The claim follows by summing (\ref{eqn:claimW1}) and (\ref{eqn:claimW2}) over $1\leq i\leq k$ and plugging the last two inequalities.
\end{proof}

\subsection{Properties of Good Points in Min-Sum} \label{app:goodpointprp}

\newcommand{\favorClustering}{\mathcal{\tilde{C}}}
\newcommand{\newClustering}{\mathcal{{C'}}}
\newcommand{\newCluster}[1]{{C'_{#1}}}

A useful property of good points is that the good points from two different clusters have cost much larger than those in a third cluster have (Lemma~\ref{lem:splitMerge}).
To prove this, we need to prove Lemma~\ref{lem:boundNewCost}, which bounds
the cost of the optimal clustering $\newClustering=\{\newCluster{t}\}$ under the perturbed distance function.
Recall the definitions of the perturbation and $\newClustering$ in Section~\ref{sec:properties_gp_aeminsum}.
The perturbation blows up all pairwise distances by a factor of $\alpha$ except the intra-cluster distances in $\favorClustering$,
where $\favorClustering$ is the clustering obtained from the optimal clustering by splitting $C_i$ into $A$ and $C_i \setminus A$ and
merging $C_j$ and $C_l$.
Let $\newClustering=\{\newCluster{i}\}$ denote the optimal clustering under the perturbed distance function $d'$,
where the clusters are indexed so that $\newCluster{i}$ corresponds to $C_i$ and the distance between the two clustering is $\sum_{i}|C_i \setminus \newCluster{i}|$.

To bound the cost of $\newClustering=\{\newCluster{t}\}$, we compare it to the cost of the optimal clustering $\mathcal{C}=\{C_t\}$ before perturbation.
If $\newClustering = \mathcal{C}$, then the cost is only increased by blowing up the distances between $A$ and $C_i \setminus A$ (Claim~\ref{cla:increase}).
However, the optimal clustering may change after the perturbation, so we need to consider how much cost is saved by the change (Claim~\ref{cla:save}).

Intuitively, the cost saved should be small. To see this, consider a point $p$ moved from $C_s$ to $\newCluster{t}$.
Then we need to pay $\dps(p,\newCluster{t})$ instead of $\ds(p,C_s)$.
Note that $p$ is in $C_s$ but not $C_t$, so $\ds(p,C_s) \leq \ds(p,C_t)$.
Also, $\newCluster{t}$ and $C_t$ differ only on at most $\epsilon n$ points, then $\dps(p,\newCluster{t})$
is larger or comparable to $\ds(p,C_t)$ and thus $\ds(p,C_s)$.

There are two technical details in the above description.
The first is to translate $\dps(p,\newCluster{t})$ to $\ds(p,\newCluster{t})$.
We consider two cases (as in the proof of Claim~\ref{cla:save}).
If $p$ is moved between $C_j$ and $C_l$, then $\dps(p,\newCluster{t})$ is roughly $\ds(p,\newCluster{t})$
since the distances between $C_j,C_l$ are not blown up.
Otherwise, $\dps(p,\newCluster{t})$ is roughly $\alpha \ds(p,\newCluster{t})$.
Another technical detail is to show that $\ds(p,\newCluster{t})$ roughly equals $\ds(p,C_t)$.
Since $\ds(p,\newCluster{t}) \geq \ds(p, \newCluster{t} \cap C_t)$, it suffices to show that $\ds(p, \newCluster{t} \cap C_t)$ is comparable to $\ds(p,C_t)$,
where Fact~\ref{fac:minsumBasicFacts} turns out to be useful.

\begin{figure*}[!t]
\centering

\begin{tikzpicture}
\usetikzlibrary{shapes,backgrounds}

\def\maxx{2.5}
\def\maxy{1.7}
\def\bbox{(-\maxx,-\maxy) rectangle (\maxx,\maxy)}
\clip \bbox;

\coordinate (centerC) at (-0.4,0);
\def\C{(centerC) circle (1.2)}

\coordinate (centerCt) at (0.4,0);
\def\Ct{(centerCt) circle (1.2)}

\fill[green!25] \C;
\fill[blue!20] \Ct;

\begin{scope}
	\clip \C;
	\fill[red!35!blue!35] \Ct;
\end{scope}

\draw \C;
\draw \Ct;

\node at (-2,-1.1) {$C_i$};
\node at (2,-1.1) {$\hat{C}_i$};
\node at (-1.3,0) {$M_i$};
\node at (0,0) {$K_i$};
\node at (1.3,0) {$A_i$};
\end{tikzpicture}
\caption{Illustration of the notations in Lemma~\ref{lem:boundNewCost}.}\label{fig:AM}
\end{figure*}
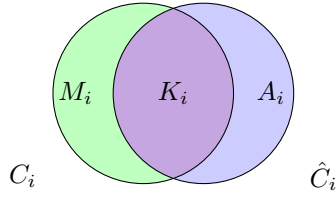

\begin{lemma}\label{lem:boundNewCost}
Suppose $\alpha > \frac{6\max_i |C_i|}{\min_i|C_i|}$ and $\min_i |C_i| \geq 100 m_B$. We have
$$\sum_{t=1}^k \dps(\newCluster{t},\newCluster{t}) - \sum_{t=1}^k \ds(C_t,C_t) \geq 2(\alpha-1)\ds(A, G_i \setminus A) - \frac{4\alpha+8}{100}\ds(C_j, C_l).$$
\end{lemma}

\newcommand{\ocapn}[1]{K_{#1}}
\newcommand{\oci}[1]{l(#1)}
\newcommand{\nci}[1]{l'(#1)}
\newcommand{\oc}[1]{C_{\oci{#1}}}
\newcommand{\nc}[1]{C_{\nci{#1}}}
\newcommand{\nk}[1]{K_{\nci{#1}}}

\begin{proof}
Let $\ocapn{t} = C_t \cap \newCluster{t}, A_t = \newCluster{t} \setminus C_t, M_t = C_t \setminus \newCluster{t}$. See Figure~\ref{fig:AM} for an illustration. We have
\begin{eqnarray*}
&& \sum_{t=1}^k \dps(\newCluster{t},\newCluster{t}) - \sum_{t=1}^k \ds(C_t,C_t) \\
& \geq & \sum_{t=1}^k \left( \dps(\ocapn{t},\ocapn{t}) + 2 \dps(A_t,\ocapn{t})\right) - \sum_{t=1}^k \left(\ds(\ocapn{t},\ocapn{t}) + 2 \ds(M_t,C_t)\right)\\
& = & \left(\sum_{t=1}^k \dps(\ocapn{t},\ocapn{t})  - \sum_{t=1}^k \ds(\ocapn{t},\ocapn{t}) \right) + 2 \left(\sum_{t=1}^k \dps(A_t,\ocapn{t}) -  \sum_{t=1}^k  \ds(M_t,C_t)\right).
\end{eqnarray*}
The first term on the right-hand side corresponds to the cost increased by blowing up the distances within the clusters,
the second term corresponds to the cost increased by moving points away.
We will bound the two terms respectively in the following two claims,
which then lead to the lemma.

Let $\oci{p}$ denote the index of the optimal cluster in $\mathcal{C}$ that $p$ falls in: if $p \in C_t$, then $\oci{p}=t$.
Similarly, let $\nci{p}$ denote the optimal cluster in $\newClustering$ that $p$ falls in after perturbation: if $p \in \newCluster{t}$, then $\nci{p}=t$.

The first term is roughly the cost increased by blowing the distances between $A_i$ and $C_i \setminus A_i$,
which is about $2(\alpha - 1) \ds(A, C_i \setminus A)$.
However, some points in $C_i$ may move away, so we need to exclude the cost of these points.
More precisely, we only consider good points, and also exclude the cost of the good points moved away ($G_i \cap M_i$).
\begin{claim}\label{cla:increase}
\begin{eqnarray*}
\sum_{t=1}^k \dps(\ocapn{t},\ocapn{t})  - \sum_{t=1}^k \ds(\ocapn{t},\ocapn{t}) 
 \geq  2(\alpha - 1) \left(\ds(A, G_i \setminus A) -  \sum_{p \in G_i \cap M_i} \ds(p, \nc{p})/\beta \right). 
\end{eqnarray*}
\end{claim}

\begin{proof}
By the definition of the perturbation, we have
\begin{eqnarray*}
&&\sum_{t=1}^k \dps(\ocapn{t},\ocapn{t})  - \sum_{t=1}^k \ds(\ocapn{t},\ocapn{t})\\
& \geq & 2 \dps(A \cap K_i, (G_i \setminus A) \cap K_i) - 2 \ds(A \cap K_i, (G_i \setminus A) \cap K_i)\\
& \geq & 2(\alpha - 1) \ds(A \cap K_i, (G_i \setminus A) \cap K_i) \\
& \geq & 2(\alpha - 1) \left( \ds(A, G_i \setminus A) - \ds(A \cap M_i, G_i \setminus A)  - \ds((G_i \setminus A) \cap M_i, A)\right) \\
& \geq & 2(\alpha - 1) \left( \ds(A, G_i \setminus A) - \ds(G_i \cap M_i, C_i)\right).
\end{eqnarray*}
The claim then follows from that for any $p \in G_i \cap M_i$, $\ds(p,C_i) \leq \frac{1}{\beta} \ds(p, \nc{p})$.
\end{proof}


The second term is roughly the cost increased by moving points away.
Consider a point $p \in C_1$ that moves to $\newCluster{2}$.
The new cost is $\dps(p, \newCluster{2}) \approx \dps(p,C_2) = \alpha \ds(p,C_2)$,
and the old cost is $\ds(p,C_1) \leq \ds(p,C_2)$, so the cost increased is roughly $(\alpha-1)\ds(p,C_2)$.
Note that $\newCluster{2}$ only approximately equals $C_2$. Also, the above intuition does not hold for points that move between $C_j$ and $C_l$
since the distances between them are not blown up. These facts only decrease the bound slightly, as shown in the following claim.
\begin{claim}\label{cla:save}
Let $X=(\cup_t A_t) \setminus (A_l \cap C_j) \setminus (A_j \cap C_l)$.
\begin{eqnarray*}
\sum_{t=1}^k \dps(A_t,\ocapn{t}) -  \sum_{t=1}^k  \ds(M_t,C_t) \geq \left(\frac{98\alpha}{100} -1 \right) \sum_{p \in X} \ds(p, \nc{p}) - \frac{2\alpha+4}{100} \ds(C_j, C_l).
\end{eqnarray*}
\end{claim}

\begin{proof}
We have
\begin{eqnarray*}
\sum_{t=1}^k \dps(A_t,\ocapn{t}) -  \sum_{t=1}^k  \ds(M_t,C_t)
\geq \sum_{t=1}^k \sum_{p \in A_t} \left( \dps(p,\ocapn{t}) - \ds(p, \oc{p})\right).
\end{eqnarray*}

Intuitively, $\dps(p,\nk{p})$ should be larger or comparable to $\ds(p, \oc{p})$.
On one hand, $\ds(p,\oc{p}) \leq \ds(p,\nc{p})$ since $p$ is assigned to $\oc{p}$ instead of $\nc{p}$ in the optimal clustering under $d$.
On the other hand, we also know that $\ds(p, \nk{p})$ is comparable to $\ds(p, \nc{p})$ by Fact~\ref{fac:minsumBasicFacts}.

Before using this intuition, we first need to translate $\dps(p, \nk{p})$ to $\ds(p, \nk{p})$.
Since the distances between $C_j$ and $C_l$ is not blown up,
we need to consider separately the case when $p$ is moved between $C_j$ and $C_l$.
Equivalently, we divide $\cup_t A_t $ into two parts: $V=(A_j \cap C_l) \cup (A_l \cap C_j)$ and  $X=(\cup_t A_t) \setminus (A_l \cap C_j) \setminus (A_j \cap C_l)$.
Now we consider the two parts respectively.

\paragraph{Case 1}  Suppose $p \in A_j \cap C_l$. By Fact~\ref{fac:minsumBasicFacts}, we have
$\dps(p,\nk{p}) = \ds(p, \ocapn{j})  \geq \frac{|\ocapn{j}|}{|C_j|} \ds(p, C_j) - \frac{1}{|C_j|} \ds(M_j, C_j).$
Since we have $\ds(p,\oc{p}) = \ds(p, C_l) \leq \ds(p, C_j)$, and $\ds(M_j, C_j) \leq \ds(M_j, C_l) \leq \ds(C_j, C_l)$,
\begin{eqnarray*}
\dps(p,\nk{p}) - \ds(p, \oc{p}) &\geq& -\frac{|M_j|}{|C_j|} \ds(p, C_j) - \frac{1}{|C_j|} \ds(C_j, C_l),\\
\sum_{p \in A_j \cap C_l}  \left( \dps(p,\nk{p}) - \ds(p, \oc{p}) \right)
& \geq & -\left( \frac{|M_j|}{|C_j|} + \frac{|A_j \cap C_l|}{|C_j|}\right) \ds(C_j, C_l).
\end{eqnarray*}
Since $|M_j|\leq \epsilon n$, $|A_j| \leq \epsilon n$, this is bounded by $-\frac{2}{100}  \ds(C_j, C_l)$.
A similar argument holds for $A_j \cap C_l$. So
\begin{eqnarray}
\sum_{p \in V}  [\dps(p,\nk{p}) - \ds(p, \oc{p}) ] & \geq & -\frac{4}{100}  \ds(C_j, C_l).\label{eqn:term2_1}
\end{eqnarray}

\paragraph{Case 2}
For $p\in X$, we have by Fact~\ref{fac:minsumBasicFacts}
\begin{eqnarray*}
\dps(p,\nk{p}) = \alpha \ds(p, \nk{p}) \geq  \alpha\left( \frac{|\nk{p}|}{|\nc{p}|} \ds(p, \nc{p}) - \frac{1}{|\nc{p}|} \ds(M_{\nci{p}}, \nc{p})\right).
\end{eqnarray*}
Then for $X$, since $\ds(p,\oc{p}) \leq \ds(p,\nc{p})$ and $ \frac{|\nk{p}|}{|\oc{p}|} \geq \frac{99}{100}$, we have
\begin{eqnarray}
&&\sum_{p \in X}  \left( \dps(p,\nk{p}) - \ds(p, \oc{p}) \right) \nonumber\\
& \geq & \left(\frac{99\alpha}{100} -1 \right) \sum_{p \in X} \ds(p, \nc{p}) - \sum_{p \in X} \frac{\alpha}{|\nc{p}|} \ds(M_{\nci{p}}, \nc{p}).\label{eqn:term2_31}
\end{eqnarray}

Since $X \subseteq \cup_t A_t$,
and $|C_t| \geq 100|A_t|$, the second term on the right-hand side is bounded by
\begin{eqnarray}
\sum_t\frac{\alpha |A_t|}{|C_t|} \ds(M_t, C_t)
& \leq & \frac{\alpha}{100} \sum_t \ds(M_t, C_t) =  \frac{\alpha}{100} \sum_{p \in \cup_t A_t} \ds(p, \oc{p}) \nonumber\\
& = & \frac{\alpha}{100} \left( \sum_{p \in V} \ds(p, \oc{p}) + \sum_{p \in V} \ds(p, \oc{p})\right)  \nonumber\\
& \leq & \frac{\alpha}{100} \left( \sum_{p \in V} \ds(p, \nc{p}) + 2 \ds(C_i,C_j)\right ).\label{eqn:term2_32}
\end{eqnarray}
The claim follows from the inequalities (\ref{eqn:term2_1}), (\ref{eqn:term2_31}), and (\ref{eqn:term2_32}).
\end{proof}

The proof is completed by combining the two claims.
\end{proof}

\textsc{Lemma~\ref{lem:splitMerge}.~}{\it
Suppose $\alpha > \frac{8\max_i |C_i|}{\min_i|C_i|}$ and $\epsilon < \frac{\min_i |C_i|}{600 n}$.
For any three different optimal clusters $C_i, C_j$, and $C_l$,
and any $A \subset G_i$, $\frac{18}{5}\ds(A, G_i\setminus A) < \ds(G_j, G_l).$
Consequently, $\frac{9}{5}\ds(G_i, G_i) < \ds(G_j, G_l).$
}

\begin{proof}
The key idea is as follows.
Let $\favorClustering$ denote the clustering obtained from the optimal clustering by splitting $C_i$ into $A$ and $C_i \setminus A$ and
merging $C_j$ and $C_l$, that is, $\favorClustering  = \{A, C_i \setminus A, C_j \cup C_l\} \cup \{C_t, t\neq i,j,l\}$.
Suppose we construct a perturbation that favors the clustering $\favorClustering$: blow up all pairwise distances by a factor of $\alpha$ except the intra-cluster distances in $\favorClustering$.
Let $\newClustering=\{\newCluster{i}\}$ denote the optimal clustering under the perturbed distance function $d'$,
where the clusters are indexed so that $\newCluster{i}$ corresponds to $C_i$ and the distance between the two clustering is $\sum_{i}|C_i \setminus \newCluster{i}|$.
By $(\alpha, \epsilon)$-perturbation resilience, we know that $\newClustering$ is different from $\favorClustering$
and has no greater cost than  $\favorClustering$.
We then show that compared to the optimal cost under the original distances,
the cost of $\favorClustering$ under perturbed distances $d'$ is larger by at most $O(\ds(C_j, C_l))=O(\ds(G_j, G_l))$,
while the cost of $\newClustering$ under perturbed distances $d'$ is larger by roughly $O(\alpha)\ds(A, G_i \setminus A)$.
These then lead to the first statement.

More precisely, the cost of $\favorClustering$ under $d'$ is larger than that of $\mathcal{C}$ under $d$ by at most
$2 \ds(C_j, C_l)$.
For $\newClustering$, we have
$\sum_{t=1}^k \dps(\newCluster{t},\newCluster{t}) - \sum_{t=1}^k \ds(C_t,C_t) \geq 2(\alpha-1)\ds(A, G_i \setminus A) - \frac{4\alpha+8}{100}\ds(C_j, C_l)$
by Lemma~\ref{lem:boundNewCost}.
Since $\newClustering$ has smaller cost than $\favorClustering$, we have $$
	2(\alpha-1)\ds(A, G_i \setminus A) - \frac{4\alpha+8}{100}\ds(C_j, C_l) \leq 2 \ds(C_j, C_l).
$$
When $\alpha > \frac{8\max_i |C_i|}{\min_i|C_i|}$, we have $\frac{27}{5} \ds(A, G_i\setminus A) \leq \ds(C_j,C_l)$.
By Lemma~\ref{lem:boundMergeCost}, we have $\ds(C_j, C_l) \leq \frac{3}{2}\ds(G_j,G_l)$, which then leads to the first part of the lemma.

The second part of the lemma follows from the fact that $\sum_{A \subseteq G_i}\ds(A, G_i\setminus A) = \frac{2^{|G_i|}}{2} \ds(G_i,G_i)$.
\end{proof}

\end{document}